\documentclass{article}

\usepackage[margin=1in]{geometry}
\usepackage[utf8]{inputenc} 
\usepackage[T1]{fontenc}    
\usepackage{hyperref}       
\usepackage{url}            
\usepackage{booktabs}
\usepackage{amsfonts}
\usepackage{nicefrac}
\usepackage{microtype}
\usepackage[utf8]{inputenc}
\usepackage{amsmath,amsthm,amssymb,bbm,bm,dsfont}
\usepackage[capitalize]{cleveref}
\usepackage{tikz}
\usepackage{xspace}
\usetikzlibrary{arrows}
\usepackage{xargs}
\usepackage{setspace}
\usepackage{natbib}
\usepackage{algorithm,algorithmicx}
\usepackage[noend]{algpseudocode}

\renewcommand{\algorithmicrequire}{\textbf{Input:}}
\renewcommand{\algorithmicensure}{\textbf{Output:}}

\newcommand{\argmin}{\operatornamewithlimits{argmin}}

\newcommand{\Ob}{\mathbf{O}}

\newcommand{\Sb}{\mathbf{S}}

\newcommand{\Acal}{\mathcal{A}}

\newcommand{\Fcal}{\mathcal{F}}
\newcommand{\Gcal}{\mathcal{G}}
\newcommand{\Hcal}{\mathcal{H}}

\newcommand{\Ocal}{\mathcal{O}}
\newcommand{\Pcal}{\mathcal{P}}
\newcommand{\Qcal}{\mathcal{Q}}
\newcommand{\Rcal}{\mathcal{R}}
\newcommand{\Scal}{{\mathcal{S}}}

\newcommand{\EE}{\mathbb{E}}

\newcommand{\bigmid}{\,\Big|\,\xspace}
\newcommand{\wrt}{\text{w.r.t.}\xspace}

\newcommand{\ie}{\emph{i.e.}\xspace}
\newcommand{\eg}{\emph{e.g.}\xspace}
\newcommand{\st}{\emph{s.t.}\xspace}
\newcommand{\indep}{\perp \!\!\! \perp\xspace}
\newcommand{\indicator}[1]{\mathds{1}\{#1\}\xspace}

\newcommand{\e}{\mathbb{E}\xspace}
\newcommand{\pe}{\mathcal{P}_e\xspace}
\newcommand{\epe}{\mathbb{E}_{\pe}\xspace}
\newcommand{\pb}{\mathcal{P}_b\xspace}
\newcommand{\pind}{\mathcal{P}_\text{ind}\xspace}
\newcommand{\eb}{\mathbb{E}_{\pb}\xspace}
\newcommand{\ps}[1]{\mathcal{P}^*_{#1} \xspace}
\newcommand{\pd}[1]{P^*_{#1} \xspace}
\newcommand{\es}[1]{\mathbb{E}^*_{#1} \xspace}
\newcommand{\psidr}{\psi_{\text{DR}}\xspace}
\newcommand{\psiis}{\psi_{\text{IS}}\xspace}
\newcommand{\psireg}{\psi_{\text{Reg}}\xspace}
\newcommand{\psieff}{\psi_{\text{eff}}\xspace}
\newcommand{\psind}[1]{\mathcal{P}^*_{\text{ind},#1}\xspace}

\newcommand{\mpci}{\mathcal M_{\textup{PCI}}\xspace}
\newcommand{\mt}[1]{\mathcal M_e^{(#1)}\xspace}
\newcommand{\fbounded}{\mathcal F_{\textup{bounded}\xspace}}

\newcommand{\noisyobs}{\textsc{NoisyObs}\xspace}
\newcommand{\epsnoise}{\epsilon_{\textup{noise}}\xspace}
\newcommand{\pinoisyobs}{\pi_b^{\noisyobs}}
\newcommand{\pieasy}{\pi_e^{\textup{easy}}}
\newcommand{\pihard}{\pi_e^{\textup{hard}}}
\newcommand{\pioptim}{\pi_e^{\textup{optim}}}

\newcommand{\sset}{\mathcal S\xspace}
\newcommand{\oset}{\mathcal O\xspace}
\newcommand{\aset}{\mathcal A\xspace}

\Crefname{assumption}{Assumption}{Assumptions}
\crefname{algocf}{algorithm}{algorithms}
\Crefname{algocf}{Algorithm}{Algorithms}

\newtheorem{theorem}{Theorem}
\newtheorem{corollary}[theorem]{Corollary}
\newtheorem{lemma}{Lemma}
\newtheorem{assumption}{Assumption}
\newtheorem{definition}{Definition}

\newtheorem{proposition}{Proposition}

\begin{document}

\title{Proximal Reinforcement Learning: Efficient Off-Policy Evaluation in Partially Observed Markov Decision Processes}

\author{Andrew Bennett and Nathan Kallus}
\date{}

\maketitle

\begin{abstract}
In applications of offline reinforcement learning to observational data, such as in healthcare or education, a general concern is that observed actions might be affected by unobserved factors, inducing confounding and biasing estimates derived under the assumption of a perfect Markov decision process (MDP) model. Here we tackle this by considering off-policy evaluation in an partially observed MDP (POMDP). Specifically, we consider estimating the value of a given target policy in an unknown POMDP given observations of trajectories with only partial state observations and generated by a different and unknown policy that may depend on the unobserved state. We tackle two questions: what conditions allow us to identify the target policy value from the observed data and, given identification, how to best estimate it. To answer these, we extend the framework of proximal causal inference to our POMDP setting, providing a variety of settings where identification is made possible by the existence of so-called bridge functions. We term the resulting framework proximal reinforcement learning (PRL). We then show how to construct estimators in these settings and prove they are semiparametrically efficient. We demonstrate the benefits of PRL in an extensive simulation study and on the problem of sepsis management.
\end{abstract}

\section{Introduction}
\label{sec:intro}

An important problem in reinforcement learning (RL) is off policy evaluation (OPE), which is defined as estimating the average reward generated by a target \emph{evaluation} policy, given observations of data generated by running some different \emph{behavior} policy. This problem is particularly important in many application areas such as healthcare, education, or robotics, where experimenting with new policies may be expensive, impractical, or unethical. In such applications OPE may be used in order to estimate the benefit of proposed policy changes by decision makers, or as a building block for the related problem of policy optimization. At the same time, in the same applications, unobservables can make this task difficult due to the lack of experimentation.

As an example, consider the problem of evaluating a newly proposed policy for assigning personalized curricula to students semester by semester, where the curriculum assignment each semester is decided based on observed student covariates, such as course outcomes and aptitude tests, with the goal of maximizing student outcomes as measured, \emph{e.g.}, by standardized test scores. Since it may be unethical to experiment with potentially detrimental curriculum plans, we may wish to evaluate such policies based on passively collected data where the targeted curriculum was decided by teachers.
However, there may be factors unobserved in the data that jointly influence the observed student covariates, curriculum assignments, and student outcomes; this may arise for example because the teacher can perceive subjective aspects of the students' personalities or aptitudes and take these into account in their decisions. 
While such confounding breaks the usual Markovian assumptions that underlie standard approaches to OPE,
the process may well be modeled by a partially observed Markov decision process (POMDP). Two key questions for OPE in POMDPs are: when is policy value still identifiable despite confounding due to partial observation and, when it is, how can we estimate it most efficiently.

In this paper we tackle these two questions, expanding the range of settings that enable identification and providing efficient estimators in these settings.
First, we extend an existing identification result for OPE in tabular POMDPs \citep{tennenholtz20off} to the continuous setting, which provides some novel insight on this existing approach but also highlights its limitations. To break these limitations, motivated by these insights, we provide a new general identification result based on extending the proximal causal inference framework \citep{miao2018identifying,cui2020semiparametric,kallus2021causal} to the dynamic, longitudinal setting.
This permits identification in more general settings.
And, unlike the previous results, this one expresses the value of the evaluation policy as the mean of some score function under the distribution over trajectories induced by the logging policy, which allows for natural estimators with good qualities.
In particular, we prove appropriate conditions under which the estimators arising from this result are consistent, asymptotically normal, and semiparametrically efficient. In addition, we provide a tractable algorithm for computing the nuisance functions that allow such estimators to be computed, based on recent state-of-the-art methods for solving conditional moment problems.
We term this framework proximal reinforcement learning (PRL), highlighting the connection to proximal causal inference.
We finally provide a series of experiments, on both a synthetic toy scenario and a complex scenario based on a sepsis simulator, which empirically validate our theoretical results and demonstrate the benefits of PRL.

\section{Related Work}

First, there is an extensive line of recent work on OPE under unmeasured confounding. This line of work considers many different forms of confounding, including confounding that is i.i.d. at each time step \citep{wang2021provably,bennett2021off,liao2021instrumental}, occurs only at a single time step \citep{namkoong2020off}, satisfies a ``memorylessness'' property \citep{kallus2020confounding}, follows a POMDP structure \citep{tennenholtz20off,nair2021spectral,oberst2019counterfactual,killian2022counterfactually}, may take an arbitrary form \citep{chen2021estimating,chandak2021universal}, or is in fact not a confounder \citep{hu2021off}. These works have varying foci: \citet{namkoong2020off,kallus2020confounding,chen2021estimating} focus on computing intervals comprising the partial identification set of all hypothetical policy values consistent with the data and their assumptions; \citet{oberst2019counterfactual,killian2022counterfactually} focus on sampling counterfactual trajectories under the evaluation policy given that the POMDP follows a particular Gumbel-softmax structure; \citet{wang2021provably,gasse2021causal} focus on using the offline data to warm start online reinforcement learning; \citet{liao2021instrumental} study OPE using instrumental variables; 
\citet{chandak2021universal} show that OPE can be performed under very general confounding if the behavior policy probabilities of the logged actions are known; 
\citet{hu2021off} consider hidden states that do not affect the behavior policy and are therefore not confounders but do make OPE harder by breaking Markovianity thereby inducing a curse of horizon;
and \citet{tennenholtz20off,nair2021spectral} study conditions under which the policy value under the POMDP model is identified. 

Of the past work on OPE under unmeasured confounding, \citet{tennenholtz20off,nair2021spectral} are closest to ours, since they too consider a general POMDP model of confounding, namely without restrictions that preserve Markovianity via i.i.d. confounders, knowing the confounder-dependent propensities, having unconfounded logged actions, or using a specific Gumbel-softmax form. \citet{tennenholtz20off} consider a particular class of tabular POMDPs satisfying some rank constraints, and \citet{nair2021spectral} extend these results and slightly relax its assumptions. However, both do not consider how to actually construct OPE estimators based on their identification results that satisfy desirable properties such as consistency or asymptotic normality, and they can only be applied to tabular POMDPs. This work presents a novel and general identification result and proposes a class of resulting OPE estimators that possesses such desirable properties.

Another area of relevant literature is on proximal causal inference (PCI). PCI was first proposed by \citet{miao2018identifying}, showing that using two conditionally independent proxies of the confounder (known as a negative control outcome and a negative control action) we can learn an outcome bridge function that generalizes the standard mean-outcome function and controls for the confounding effects. Since then this work has been expanded, including by alternatively using an action bridge function which instead generalizes the inverse propensity score \citep{miao2018confounding}, allowing for multiple fixed treatments \citep{tchetgen2020introduction}, performing multiply-robust treatment effect estimation \citep{shi2020multiply}, combining outcome and action bridge functions for semiparametrically efficient estimation \citep{cui2020semiparametric}, using PCI to estimate the value of contextual-bandit policies \citep{xu2021deep} or generalized treatment effects \citep{kallus2021causal}, or estimating bridge functions using adversarial machine learning \citep{kallus2021causal,ghassami2022minimax}. In addition, the OPE for POMDP methodologies of \citet{tennenholtz20off,nair2021spectral} discussed above were said to be motivated by PCI.
This work relates to this body of work as it proposes a new way of performing OPE for POMDPs using PCI, and it also proposes a new adversarial machine learning-based approach for estimating the bridge functions.

At the intersection of work of OPE and PCI is the concurrent work of \citet{ying2021proximal}, which considers PCI in multi time step scenarios, given two proxies at each time step similar to what we consider in \cref{sec:ident-pci}. Unlike us they only consider the problem of estimating treatment effects for fixed vectors of treatment at each time step, optionally conditional on observable context at $t=1$, as opposed to evaluating policies that can adaptively treat based on the context available so far.

Finally, there is an extensive body of work on learning policies for POMDPs using online learning. For example, see \citet{azizzadenesheli2016reinforcement}, \citet{katt2017learning}, \citet{bhattacharya2020reinforcement},\citet{yang2021causal}, \citet{singh2021structured}, and references therein. 
This work is distinct in that we consider an offline setting where identification is an issue.
At the same time, this work is related to the online setting in that it could potentially be used to augment and warm start such approaches if there is also offline observed data available.

\section{Problem Setting}

A POMDP
is formally defined by a tuple $(\mathcal S, \mathcal A, \mathcal O, H, P_O, P_R, P_T)$, where $\mathcal S$ denotes a state space, $\mathcal A$ denotes a finite action space, $\mathcal O$ denotes an observation space, $H \in \mathbb N$ denotes a time horizon, $P_O$ is an observation kernel, with $P_O^{(t)}(\cdot \mid s)$ denoting the density of the observation $O_t$ given the state $S_t=s$ at time $t$, $P_R$ is a reward kernel, with $P_R^{(t)}(\cdot \mid s,a)$ denoting the density of the (bounded) reward $R_t\in[-R_{\max},R_{\max}]$ given the state $S_t=s$ and action $A_t=a$ at time $t$, and $P_T$ is a transition kernel, with $P_T^{(t)}(\cdot \mid s, a)$ denoting the density of the next $S_{t+1}$ given the state $S_t=s$ and action $A_t=a$ at time $t$. Note that we allow for the POMDP to be time inhomogeneous; that is, we allow the outcome, reward, and transition kernels to potentially depend on the time index. Finally, we let $O_0$ denote some prior observation of the state before $t=1$ (which may be empty), and we let $\tau^{\textup{full}}_t$ and $\tau_t$ denote the true and observed trajectories up to time $t$ respectively, which we define as
\begin{align*}
    \tau_0 &= \tau^{\textup{full}}_0 = O_0 \\
    \tau_t &= (O_0, (O_1,A_1,R_1), (O_2,A_2,R_t), \ldots, (O_t,A_t,R_t))  \\
    \tau^{\textup{full}}_t &= (O_0, (S_1,O_1,A_1,R_1), (S_2,O_2,A_2,R_t), \ldots, (S_t,O_t,A_t,R_t))  \,.
\end{align*}

Let $\pi_b$ be some given randomized \emph{logging policy}, which is characterized by a sequence of functions $\pi_b^{(1)},\ldots,\pi_b^{(H)}$, where $\pi_b^{(t)}(a \mid S_t)$ denotes the probability that the logging policy takes action $a \in \mathcal A$ at time $t$ given state $S_t$. The logging policy together with the POMDP define a joint distribution over the (true) trajectory $\tau_H^{\textup{full}}$ given by acting according to $\pi_b$; let $\pb$ denote this distribution.
All probabilities and expectations in the ensuing will be with respect to $\pb$ unless otherwise specified, \emph{e.g.}, by a subscript.

Our data consists of observed trajectories generated by the logging policy: $\mathcal D = \{\tau_H^{(1)},\tau_H^{(2)},\ldots,\tau_H^{(n)}\}$, where each $\tau_H^{(i)}$ is an i.i.d. sample of $\tau_H$ (which does not contain $S_t$), distributed according to $\pb$. Importantly, we emphasize that, although we assume that states are unobserved by the decision maker and are not included in the logged data $\mathcal D$, the logging policy still uses these hidden states, inducing confounding.

Implicit in our notation $\pi_b^{(t)}(a \mid S_t)$ is that the logging policy actions are independent of the past given current state $S_t$. Similarly, the POMDP model is characterized by similar independence assumption with respect to observation and reward emissions, and state transitions. This means that $\pb$ satisfies a Markovian assumption with respect to $S_t$; however, as $S_t$ is unobserved we cannot condition on it and break the past from the future. We visualize the directed acyclic graph (DAG) representing $\pb$ in \cref{fig:pomdp-logging}. In particular,
we have the following conditional independencies in $\pb$: for every $t$,
\begin{align*}
    &O_t \indep \tau^{\textup{full}}_{t-1} \mid S_t \\
    &R_t \indep \tau^{\textup{full}}_{t-1}, O_t \mid S_t, A_t \\
    &S_{t+1} \indep \tau^{\textup{full}}_{t-1}, O_t, R_t \mid S_t, A_t  \\
    &A_t \indep \tau^{\textup{full}}_{t-1} \mid S_t \,.
\end{align*}

Now, let $\pi_e$ be some deterministic \emph{target policy} that we wish to evaluate, which is characterized by a sequence of functions $\pi_e^{(1)},\ldots,\pi_e^{(H)}$, where $\pi_e^{(t)}(O_t,\tau_{t-1}) \in \mathcal A$ denotes the action taken by policy $\pi_e$ at time $t$ given current observation $O_t$ and the past observable trajectory $\tau_{t-1}$.
We visualize the POMDP model under such a policy that only depends on observable data in \cref{fig:pomdp-intervention}. Note that we allow $\pi_e^{(t)}$ to potentially depend on all observable data up to time $t$; this is because the Markovian assumption \emph{does not} hold with respect to the observations $O_t$, so we may wish to consider policies that use all past observable information to best account for the unobserved state. We let $\pe$ denote the distribution over trajectories that would be obtained by following policy $\pi_e$ in the POMDP. Then, given some discounting factor $\gamma \in (0, 1]$, we define the \emph{value} of policy $\pi_e$ as
\begin{equation*}
    v_\gamma(\pi_e) = \sum_{t=1}^H \gamma^{t-1} \epe[R_t] \,,
\end{equation*}
The task OPE under the POMDP model is to estimate $v_\gamma(\pi_e)$ (a function of $\pe$) given $\mathcal D$ (drawn from $\pb$). 

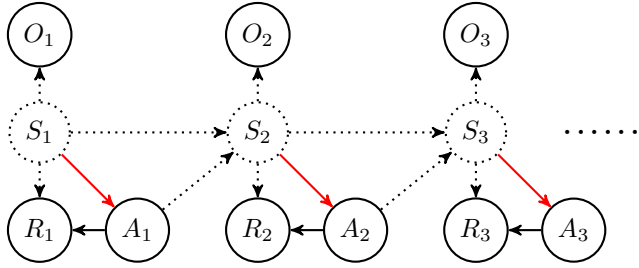
\begin{figure}
\centering
\begin{tikzpicture}[->,>=stealth',auto,node distance=3cm,
  thick,mainnode/.style={circle,draw},dottednode/.style={circle,draw,dotted},noborder/.style={{circle,inner sep=0,outer sep=0}}]
  \node[dottednode] (s1) at (0,0) {$S_1$};
  \node[mainnode] (a1) at (1.3,-1.3) {$A_1$};
  \node[mainnode] (r1) at (0,-1.3) {$R_1$};
  \node[mainnode] (o1) at (0,1.3) {$O_1$};
  \node[dottednode] (s2) at (2.9,0) {$S_2$};
  \node[mainnode] (a2) at (4.2,-1.3) {$A_2$};
  \node[mainnode] (r2) at (2.9,-1.3) {$R_2$};
  \node[mainnode] (o2) at (2.9,1.3) {$O_2$};
  \node[dottednode] (s3) at (5.8,0) {$S_3$};
  \node[mainnode] (a3) at (7.1,-1.3) {$A_3$};
  \node[mainnode] (r3) at (5.8,-1.3) {$R_3$};
  \node[mainnode] (o3) at (5.8,1.3) {$O_3$};
  \node[noborder] (elipses) at (7.5,0) {\textbf{\ldots\ldots}};
  \path[every node/.style={font=\sffamily\tiny}]
    (s1) edge[dotted] node [right] {} (o1)
    (s1) edge[red] node [right] {} (a1)
    (s1) edge[dotted] node [right] {} (r1)
    (a1) edge node [right] {} (r1)
    (s1) edge[dotted] node [right] {} (s2)
    (a1) edge[dotted] node [right] {} (s2)
    (s2) edge[dotted] node [right] {} (o2)
    (s2) edge[red] node [right] {} (a2)
    (s2) edge[dotted] node [right] {} (r2)
    (a2) edge node [right] {} (r2)
    (s2) edge[dotted] node [right] {} (s3)
    (a2) edge[dotted] node [right] {} (s3)
    (s3) edge[dotted] node [right] {} (o3)
    (s3) edge[red] node [right] {} (a3)
    (s3) edge[dotted] node [right] {} (r3)
    (a3) edge node [right] {} (r3);
\end{tikzpicture}
\caption{Graphical representation of the POMDP model under the logging policy $\pi_b$. The red arrows make explicit the dependence of $\pi_b$ on the hidden state. Dashed circles denote variables unobserved in our data.}
\label{fig:pomdp-logging}
\end{figure}

\begin{figure}
\centering
\begin{tikzpicture}[->,>=stealth',auto,node distance=3cm,
  thick,mainnode/.style={circle,draw},bluenode/.style={circle,draw,thin,color=blue},dottednode/.style={circle,draw,dotted},noborder/.style={{circle,inner sep=0,outer sep=0}}]
  \node[dottednode] (s1) at (0,0) {$S_1$};
  \node[mainnode] (a1) at (1.3,-1.3) {$A_1$};
  \node[mainnode] (r1) at (0,-1.3) {$R_1$};
  \node[mainnode] (o1) at (0,1.3) {$O_1$};
  \node[dottednode] (s2) at (2.9,0) {$S_2$};
  \node[mainnode] (a2) at (4.2,-1.3) {$A_2$};
  \node[mainnode] (r2) at (2.9,-1.3) {$R_2$};
  \node[mainnode] (o2) at (2.9,1.3) {$O_2$};
  \node[dottednode] (s3) at (5.8,0) {$S_3$};
  \node[mainnode] (a3) at (7.1,-1.3) {$A_3$};
  \node[mainnode] (r3) at (5.8,-1.3) {$R_3$};
  \node[mainnode] (o3) at (5.8,1.3) {$O_3$};
  \node[noborder] (elipses) at (8.5,0.5) {\textbf{\ldots\ldots}};
  \node[bluenode] (h0) at (-1.6,2.3) {$\tau_0$};
  \node[bluenode] (h1) at (1.3,2.3) {$\tau_1$};
  \node[bluenode] (h2) at (4.2,2.3) {$\tau_2$};
  \node[bluenode] (h3) at (7.1,2.3) {$\tau_3$};
  \path[every node/.style={font=\sffamily\tiny}]
    (h0) edge[thin, blue] node [right] {} (h1)
    (o1) edge[thin, blue] node [right] {} (h1)
    (r1) edge[thin, blue] node [right] {} (h1)
    (a1) edge[thin, blue] node [right] {} (h1)
    (h1) edge[thin, blue] node [right] {} (h2)
    (o2) edge[thin, blue] node [right] {} (h2)
    (r2) edge[thin, blue] node [right] {} (h2)
    (a2) edge[thin, blue] node [right] {} (h2)
    (h2) edge[thin, blue] node [right] {} (h3)
    (o3) edge[thin, blue] node [right] {} (h3)
    (r3) edge[thin, blue] node [right] {} (h3)
    (a3) edge[thin, blue] node [right] {} (h3)
    (o1) edge[red] node [right] {} (a1)
    (o2) edge[red] node [right] {} (a2)
    (h1) edge[red] node [right] {} (a2)
    (o3) edge[red] node [right] {} (a3)
    (h2) edge[red] node [right] {} (a3)
    (s1) edge[dotted] node [right] {} (o1)
    (s1) edge[dotted] node [right] {} (r1)
    (a1) edge node [right] {} (r1)
    (s1) edge[dotted] node [right] {} (s2)
    (a1) edge[dotted] node [right] {} (s2)
    (s2) edge[dotted] node [right] {} (o2)
    (s2) edge[dotted] node [right] {} (r2)
    (a2) edge node [right] {} (r2)
    (s2) edge[dotted] node [right] {} (s3)
    (a2) edge[dotted] node [right] {} (s3)
    (s3) edge[dotted] node [right] {} (o3)
    (s3) edge[dotted] node [right] {} (r3)
    (a3) edge node [right] {} (r3);
\end{tikzpicture}
\caption{Graphical representation of the POMDP model under the evaluation policy $\pi_e$. The red arrows make explicit the dependence of $\pi_e$ on the current observation and previous observable trajectory, and the blue nodes and arrows make explicit the dependence of the observable trajectories on the data.}
\label{fig:pomdp-intervention}
\end{figure}
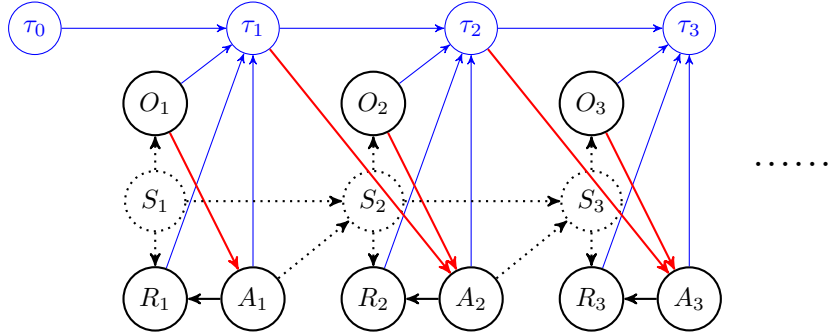

\section{Identification Theory}

Before considering how to actually estimate $v_\gamma(\pi_e)$, we first consider the problem of \emph{identification}, which is the problem of finding some function $\psi$ such that $v_\gamma(\pi_e) = \psi(\pb)$, and is a prerequisite for identificaiton.
This is the first stepping stone because $\pb$ is the most we could hope to ever learn from observing $\mathcal D$.
If such a $\psi$ exists, then we say that $v_\gamma(\pi_e)$ is \emph{identified} with respect to $\pb$. In general, such an identification result is impossible for the OPE problem given unobserved confounding as introduced by our POMDP model. Therefore, we must impose some assumptions on $\pb$ for such identification to be possible.

To the best of our knowledge, the only existing identification result of this kind was presented by \citet{tennenholtz20off} (with a slight generalization given by \citealp{nair2021spectral}), and is only valid in tabular settings where states and observations are discrete. We will proceed first by extending this approach to more general, non-tabular settings.
However, we will note that there are some restrictive limitations to estimation based on this approach.
So, motivated by the limitations, we develop a new and more general identification theory which extends the PCI approach to the sequential setting and easily enables efficient estimation.

\subsection{Identification by Time-Independent Sampling and Its Limitations}
\label{sec:ident-ind}

For our generalization of \citet{tennenholtz20off}, we will consider evaluating policies $\pi_e$ such that $\pi_e^{(t)}(O_t,\tau_t)$ only depends on $O_{1:t}$ and $A_{1:t-1}$; that is, $\pi_e^{(t)}$ can depend on all observed data available at time $t$ except for $O_0$ and past rewards.
First, for each $t \in \{1,\ldots,H\}$, let $D_t = (O_{t-1},O_t,O_{t+1},A_t,R_t)$, and for any such tuple $D = (O, O', O'', A, R)$ define $o(D) = O$, $o'(D) = O'$, $o''(D) = O''$, $a(D) = A$, and $r(D) = R$. In addition, define the shorthand $\pi_e^{(t)}(D_{1:t}) = \pi_e^{(t)}(o'(D_t),\ldots,o'(D_1), a(D_{t-1}),\ldots,a(D_1))$.
Furthermore, let $\pind$ denote the measure on $D_{1:H}$ in which each tuple $D_t$ is sampled \emph{independently} according to its marginal distribution in $\pb$. Note that under this measure the overlapping observations between these tuples (\emph{e.g.} $o'(D_t)$ and $o(D_{t+1})$) may take different values.
Then, given these definitions, we have the following result.

\begin{theorem}
\label{thm:tennenholtz}
Under some regularity conditions detailed in \cref{apx:tennenholtz}, there exist functions $\rho^{(t)}$ defined by conditional moment restrictions under $\pb$, such that for every $t \in \{1,\ldots,H]\}$ we have
\begin{equation*}
    \epe[R_t] = \e_{\pind}\left[ r(D_t) \prod_{s=1}^t \indicator{a(D_s)=\pi^{(t)}(D_{1:s})} \rho^{(s)} \Big( o(D_s), a(D_s), o''(D_{s-1}) \Big) \right] \,.
\end{equation*}
Furthermore, under the conditions of \citet[Theorem 1]{tennenholtz20off}, these regularity conditions are satisfied, and the above is identical to their identification quantity.
\end{theorem}

Since $\pind$ is a function of $\pb$, and $v_\gamma(\pi_e)$ is a function of $\epe[R_1],\ldots,\epe[R_H]$, \cref{thm:tennenholtz} gives a valid identification quantity for $v_\gamma(\pi_e)$.
The full details of the regularity conditions and nuisance functions governing this result are not very important to this paper, so they are deferred along with the proof of this theorem to \cref{apx:tennenholtz}.
For our purposes, the main takeaway of \cref{thm:tennenholtz} is that there exists a natural generalization of \citet[Theorem 1]{tennenholtz20off} to non-discrete settings; while that result was originally expressed as a sum over all possible observable trajectories, we show that it can instead be expressed as the expectation of a simple, estimable quantity whose existence does not depend on discreteness.
Unfortunately, the expectation that naturally arises is under $\pind$ rather than $\pb$.
This means that empirical approximations of this expectation given $n$ i.i.d. samples from $\pb$ would require averaging over $n^s$ terms, introducing a curse of dimension.
Furthermore, this expectation clearly does not have many of the desirable properties for OPE estimating equations held by many OPE estimators in the simpler MDP setting, such as Neyman orthogonality \citep{kallus2020double,kallus2022efficiently}.

\subsection{Identification by Proximal Causal Inference}
\label{sec:ident-pci}

We now discuss an alternative way of obtaining identifiability, via a reduction to a nested sequence of proximal causal inference (PCI) problems of the kind described by \citet{cui2020semiparametric}. These authors considered identifying the average treatment effect (ATE), and other related causal estimands, for binary decision making problems with unmeasured confounding given two independent proxies for the confounders, one of which is conditionally independent from treatments given confounders, and the other of which is independent from outcomes given treatment and confounders. We will in fact leverage the refinement of the PCI approach by \citet{kallus2021causal}, which has strictly weaker assumptions than \citet{cui2020semiparametric}.

Our reduction works by defining random variables $Z_t$ and $W_t$ for each $t \in [H]$ that are measurable w.r.t. the observed trajectory $\tau_H$, as well as defining random variables $U_t$ for each $t \in [H]$ such that $S_t$ is measurable w.r.t. $U_t$. We respectively refer to $Z_t$ and $W_t$ as \emph{negative control actions} and \emph{negative control outcomes}, and we refer to $U_t$ as \emph{confounders}. All triplets $(Z_t,W_t,U_t)$ must be satisfy certain independence properties outlined below. Any definition of such variables that satisfy these independence properties is considered a valid PCI reduction, and we will have various examples of valid PCI reductions for our POMDP model at the end of this section.

To formalize these assumptions, we must first define some additional notation. Let $\ps{t}$ denote the measure on trajectories induced by running policy $\pi_e$ for the first $t-1$ actions, and running policy $\pi_b$ henceforth. Note that according to this definition, $\pb=\ps{1}$, and $\pe=\ps{H+1}$. In addition, let $\es{t}$ and $\pd{t}$ be shorthand for expectation and probability mass under $\ps{t}$ respectively.
We visualize these intervention distributions in the first part of \cref{fig:pt-distribution}.

Next, for each $t \in \{1,\ldots,H\}$ we define $E_t = \pi_e^{(t)}(O_t,\tau_{t-1})$, and $D_t = (Z_t,W_t,A_t,E_t,R_t)$. In addition, we will refer to any random variable $Y_t$ as an \emph{outcome variable at time $t$} if it is measurable w.r.t. $(R_t,D_{t+1:H})$. For any such variable and $a \in \aset$, we use $Y_t(a)$ to denote a random variable with the same distribution that $Y_t$ would have if, possibly counter to fact, action $a$ were taken at time $t$ instead of $A_t$. We note that under $\ps{t}$, we can interpret $Y_t(a)$ as the outcome that would be obtained by applying $\pi_e$ for the first $t-1$ actions, the fixed action $a$ at time $t$, and then $\pi_b$ henceforth (as opposed to the factual outcome $Y_t$ obtained by applying $\pi_e$ for the first $t-1$ actions and $\pi_b$ henceforth). We also note that according to this notation $Y_t(A_t) = Y_t$ always.

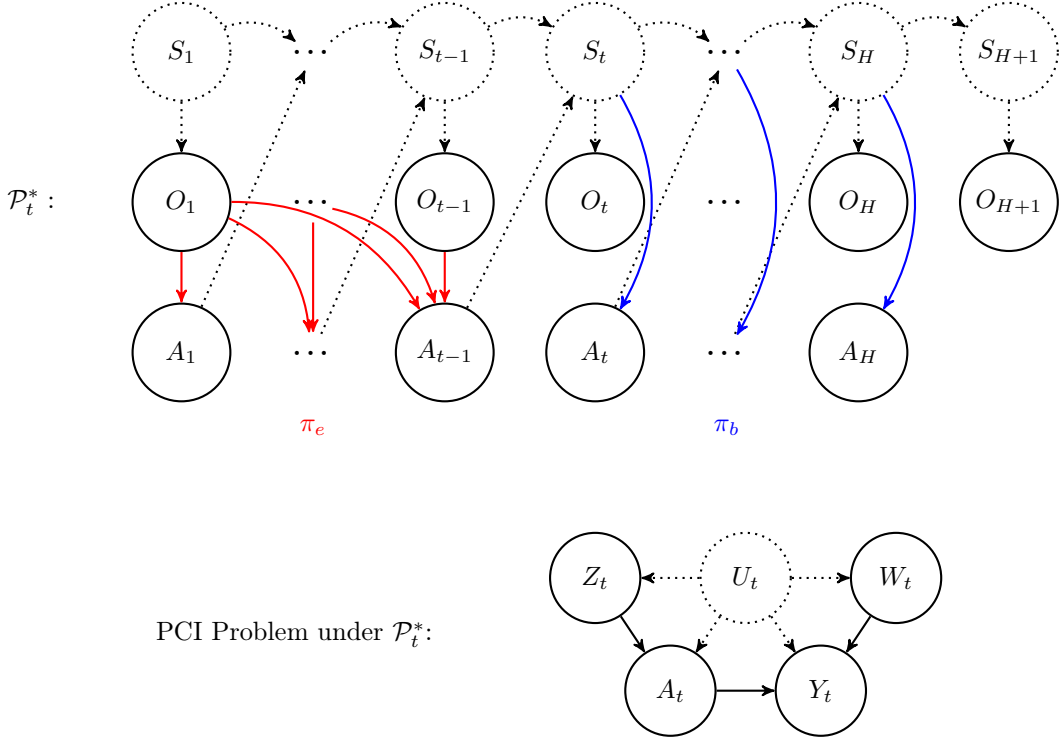
\begin{figure}
\begin{center}
\begin{tikzpicture}[->,>=stealth',auto,node distance=3cm,
  thick,mainnode/.style={circle,draw},dottednode/.style={circle,draw,dotted},noborder/.style={{circle,inner sep=0,outer sep=0}}]
  \node[black,noborder] (ptstar) at (-10.5,2) {$\mathcal{P}_t^*:$};
  \node[black,dottednode,minimum size=1.3cm] (S1) at (-8.5,4) {$S_1$};
  \node[black,mainnode,minimum size=1.3cm] (O1) at (-8.5,2) {$O_1$};
  \node[black,mainnode,minimum size=1.3cm] (A1) at (-8.5,0) {$A_1$};
  \node[black, noborder] (elipses1S) at (-6.75,4) {\textbf{\ldots}};
  \node[black, noborder] (elipses1O) at (-6.75,2) {\textbf{\ldots}};
  \node[black, noborder] (elipses1A) at (-6.75,0) {\textbf{\ldots}};
  \node[red, noborder] (piE) at (-6.75,-1) {$\pi_e$};
  \node[black,dottednode,minimum size=1.3cm] (St1) at (-5,4) {$S_{t-1}$};
  \node[black,mainnode,minimum size=1.3cm] (Ot1) at (-5,2) {$O_{t-1}$};
  \node[black,mainnode,minimum size=1.3cm] (At1) at (-5,0) {$A_{t-1}$};
  \node[black,dottednode,minimum size=1.3cm] (St) at (-3,4) {$S_t$};
  \node[black,mainnode,minimum size=1.3cm] (Ot) at (-3,2) {$O_t$};
  \node[black,mainnode,minimum size=1.3cm] (At) at (-3,0) {$A_t$};
  \node[black, noborder] (elipses2S) at (-1.25,4) {\textbf{\ldots}};
  \node[black, noborder] (elipses2O) at (-1.25,2) {\textbf{\ldots}};
  \node[black, noborder] (elipses2A) at (-1.25,0) {\textbf{\ldots}};
  \node[blue, noborder] (piB) at (-1.25,-1) {$\pi_b$};
  \node[black,dottednode,minimum size=1.3cm] (Sh) at (0.5,4) {$S_H$};
  \node[black,mainnode,minimum size=1.3cm] (Oh) at (0.5,2) {$O_H$};
  \node[black,mainnode,minimum size=1.3cm] (Ah) at (0.5,0) {$A_H$};
  \node[black,dottednode,minimum size=1.3cm] (Sh1) at (2.5,4) {$S_{H+1}$};
  \node[black,mainnode,minimum size=1.3cm] (Oh1) at (2.5,2) {$O_{H+1}$};
  \node[black,dottednode,minimum size=1.2cm] (s) at (-1,-3) {$U_t$};
  \node[black,mainnode,minimum size=1.2cm] (z) at (-3,-3) {$Z_t$};
  \node[black,mainnode,minimum size=1.2cm] (a) at (-2,-4.5) {$A_t$};
  \node[black,mainnode,minimum size=1.2cm] (w) at (1,-3) {$W_t$};
  \node[black,mainnode,minimum size=1.2cm] (y) at (0,-4.5) {$Y_t$};
  \node[black,noborder] (label) at (-7,-3.75) {PCI Problem under $\mathcal{P}_t^*$:};
  \path[every node/.style={font=\sffamily\tiny}]
    (S1) edge[black,dotted] node [right] {} (O1)
    (St1) edge[black,dotted] node [right] {} (Ot1)
    (St) edge[black,dotted] node [right] {} (Ot)
    (Sh) edge[black,dotted] node [right] {} (Oh)
    (Sh1) edge[black,dotted] node [right] {} (Oh1)
    (A1) edge[black,dotted] node [right] {} (elipses1S)
    (elipses1A) edge[black,dotted] node [right] {} (St1)
    (At1) edge[black,dotted] node [right] {} (St)
    (At) edge[black,dotted] node [right] {} (elipses2S)
    (elipses2A) edge[black,dotted] node [right] {} (Sh)
    (S1) edge[black,dotted,bend left] node [right] {} (elipses1S)
    (elipses1S) edge[black,dotted,bend left] node [right] {} (St1)
    (St1) edge[black,dotted,bend left] node [right] {} (St)
    (St) edge[black,dotted,bend left] node [right] {} (elipses2S)
    (elipses2S) edge[black,dotted,bend left] node [right] {} (Sh)
    (Sh) edge[black,dotted,bend left] node [right] {} (Sh1)
    (O1) edge[red] node [right] {} (A1)
    (O1) edge[red,bend left] node [right] {} (elipses1A)
    (elipses1O) edge[red] node [right] {} (elipses1A)
    (elipses1O) edge[red,bend left] node [right] {} (At1)
    (O1) edge[red,bend left] node [right] {} (At1)
    (Ot1) edge[red] node [right] {} (At1)
    (St) edge[blue, bend left] node [right] {} (At)
    (elipses2S) edge[blue, bend left] node [right] {} (elipses2A)
    (Sh) edge[blue, bend left] node [right] {} (Ah)
    (s) edge[black,dotted] node [right] {} (z)
    (s) edge[black,dotted] node [right] {} (a)
    (s) edge[black,dotted] node [right] {} (w)
    (s) edge[black,dotted] node [right] {} (y)
    (z) edge[black] node [right] {} (a)
    (w) edge[black] node [right] {} (y)
    (a) edge[black] node [right] {} (y);
\end{tikzpicture}
\caption{Visual summary of the interventional distributions and corresponding conditional independence assumptions on proxies for our Proximal RL theory. \textbf{Above:} Visual representation of the interventional distribution $\ps{t}$, which is the distribution over trajectories obtained by following the evaluation policy $\pi_e$ for the first $t-1$ actions, and then taking all subsequent actions following $\pi_b$.
\textbf{Below}: Probabilistic graphical representation of the corresponding Proximal Causal Inference decision-making problem at time $t$ under $\ps{t}$, with outcome variable $Y_t = \phi(R_t,D_{t+1:H})$ for arbitrary $\phi$. The variables $Z_t$ and $W_t$ are conditionally independent action-side and outcome-side proxies for the true (unobserved) confounder $U_t$.}
\label{fig:pt-distribution}
\end{center}
\end{figure}

Given these definitions, we are ready to present our core assumptions. Our first assumption is that the confounders $U_t$ are sufficient to induce a particular conditional independence structure between the proxies $Z_t$ and $W_t$, as well as the observable data. Specifically, we assume the following:
\begin{assumption}[Negative Controls]
\label{asm:pci-reduction}
    For each $t \in [H]$ and $a \in \aset$, and any outcome variable $Y_t$ that is measurable w.r.t. $(R_t,D_{t+1:H})$, we have
    \begin{equation*}
        Z_t, A_t \indep_{\ps{t}} W_t, E_t, Y_t(a) \mid U_t \,.
    \end{equation*}
\end{assumption}

We note that these independence assumptions imply that the decision making problem under $\ps{t}$ with confounder $U_t$, negative controls $Z_t$ and $W_t$, action $A_t$, and outcome $(R_t,D_{t+1:H})$ satisfy the PCI problem structure as in \citet{cui2020semiparametric}. We visualize this structure for the problem at time $t$ in \cref{fig:pt-distribution}. In addition, it requires that the action-side proxy $Z_t$ is conditionally independent from the next action $E_t$ that would have been taken under $\pi_e^{(t)}$.
Note also that we may additionally include an observable context variable $X_t$, which may be useful for defining more efficient reductions. In this case, the conditional independence assumption in \cref{asm:pci-reduction} should hold given both $U_t$ and $X_t$, and in everything that follows $Z_t$, $W_t$, and $U_t$ should be replaced with $(Z_t,X_t)$, $(W_t,X_t)$, and $(U_t,X_t)$ respectively, as in \citet{cui2020semiparametric}. However, we omit $X_t$ from the notation in the rest of the paper for brevity.

Next, our results require the existence of some \emph{bridge} functions, as follows.
\begin{assumption}[Bridge Functions Exist]
\label{asm:bridge}
    For each $t \in [H]$ and $a \in \aset$, and any given outcome variable $Y_t = \phi(R_t,D_{t+1:H})$, there exists functions $q^{(t)}$ and $h^{(t,\phi)}$ satisfying
    \begin{align*}
        \es{t}[q^{(t)}(Z_t,A_t) \mid U_t, A_t=a] &= \pd{t}(A_t=a \mid U_t)^{-1} \qquad \text{a.s.} \\
        \text{and} \qquad \es{t}[h^{(t,\phi)}(W_t,A_t) \mid U_t,A_t=a] &= \es{t}[\indicator{E_t=A_t} Y_t \mid U_t,A_t=a] \qquad \text{a.s.} \,.
    \end{align*}
\end{assumption}

Implicit in the assumption is that $\pd{t}(A_t=a \mid U_t)>0$.
We refer to the functions $q^{(t)}$ as \emph{action bridge functions}, and $h^{(t,\phi)}$ as \emph{outcome bridge functions}. These may be seen as analogues of inverse propensity scores and state-action quality functions respectively. As argued previously by \citet{kallus2021causal}, assuming the existence of these functions is more general than the approach taken by \citet{cui2020semiparametric}, who require complex completeness conditions.
We refer readers to \citet{kallus2021causal} for a detailed presentation of conditions under which the existence of such bridge functions can be justified, as well as concrete examples
of bridge functions when the negative controls are discrete, or the negative controls and $Y_t$ are defined by linear models.

In the case of both \cref{asm:pci-reduction,asm:bridge}, the assumption depends on the choice of proxies $Z_t$ and $W_t$, and on the choice of confounders $U_t$. In addition, the parts of $(O_t,\tau_{t-1})$ that $\pi_e^{(t)}$ may depend on determines what variables $E_t$ is a function of, so the evaluation policy $\pi_e$ also affects the validity of \cref{asm:pci-reduction}.
For now we just emphasize this important point, and present our main identification theory, which is valid given these assumptions. However, we will provide some concrete examples of feasible and valid choices of $(Z_t,W_t,U_t)$ that satisfy \cref{asm:pci-reduction} for different kinds of policies $\pi_e$ in \cref{sec:pci-reduction-kinds}. In addition, we provide an in-depth examination of the additional conditions under which \cref{asm:bridge} holds for an example tabular setting in \cref{sec:example-tabular}.

\begin{theorem}
\label{thm:identification-pci}
    Let \cref{asm:pci-reduction,asm:bridge} hold. Define $q^{(t)}$ and $h^{(t)}$ as any solutions to the following equations (which are assumed to hold almost surely)
    \begin{align}
        \label{eq:q}
        \es{t}\Big[q^{(t)}(Z_t,A_t)  \mid W_t,A_t=a\Big] &= \pd{t}(A_t=a \mid W_t)^{-1} \quad \forall a \in \aset\,, \\
        \label{eq:h}
        \es{t}\Big[h^{(t)}(W_t,A_t) \mid Z_t, A_t=a\Big] &= \es{t}\Big[\indicator{E_t=A_t} Y_t \mid Z_t, A_t=a\Big] \quad \forall a \in \aset \,,
    \end{align}
    where $Y_H = R_H$, and for every $t \leq H$ we recursively define
    \begin{equation}
        \label{eq:yt}
        Y_{t-1} = R_{t-1} + \gamma \Big( \sum_{a \in \aset} h^{(t)}(W_{t},a) + q^{(t)}(Z_{t},A_{t}) \Big( \indicator{A_{t}=E_{t}} Y_{t} - h^{(t)}(W_{t},A_{t}) \Big) \Big)\,.
    \end{equation}
    Also, let $\eta_t = \prod_{s=1}^{t-1} \indicator{E_s=A_s} q^{(s)}(Z_s,A_s)$. Then, we have $v_\gamma(\pi_e) = \eb[\psidr(\tau_H)]$, where
    \begin{equation}
        \label{eq:psidr}
        \psidr(\tau_H) = \sum_{t=1}^H \gamma^{t-1} \Big( \eta_{t+1} R_t + \eta_t \sum_{a \in \aset} h^{(t)}(W_t,a) - \eta_t q^{(t)}(Z_t,A_t) h^{(t)}(W_t,A_t) \Big) \,.
    \end{equation}
\end{theorem}

Since $\eb[\psidr(\tau_H)]$ is fully defined by $\pb$, this is a valid identification result. As detailed in our proof, the existence of solutions to \cref{eq:q,eq:h} is guaranteed given our assumptions.
Comparing with \cref{thm:tennenholtz}, this result has many immediate advantages; it is written as an expectation over $\pb$, and so may be analyzed readily using standard semiparametric efficiency theory, and although \cref{eq:h,eq:q} may appear complex given that they are expressed in terms of the intervention distributions $\ps{t}$, this can easily be dealt with as discussed later.
We also observe that \cref{eq:psidr} has a very similar structure to the Double Reinforcement Learning (DRL) estimators for the MDP setting \citep{kallus2020double}, where $h^{(t)}$ and $q^{(t)}$ are used in place of inverse propensity score and quality function terms respectively. This is very promising, since DRL estimators enjoy desirable properties such as semiparametric efficiency in the MDP setting \citep{kallus2020double}. Indeed, in \cref{sec:estimation} we show that similar properties extend to estimators defined based on \cref{eq:psidr}. 

At a high level, the proof of \cref{thm:identification-pci} works by defining a series of of outcome variables $Y_t$ such that, for each PCI problem at time $t \in [H]$ under distribution $\ps{t}$ and with outcome variable $Y_t$, the policy value obtained by intervening at time $t$ with $\pi_e$ is equal to $\epe[R_t + \gamma R_{t+1} + \ldots + \gamma^{H-t} R_H]$. In the base case of $t=H$ this property is trivially satisfied with $Y_t=R_t$, since under $\ps{H}$ all prior actions prior to time $H$ are taken following $\pi_e$. Conversely, for $t<H$, we establish via backward induction that this holds with $Y_t$ defined according to \cref{eq:yt}. Intuitively, this works because the term multiplied by $\gamma$ in \cref{eq:yt} is the doubly robust influence function for the PCI problem at time $t$, so $\es{t}[Y_{t-1}] = \epe[R_{t-1}] + \gamma \es{t+1}[Y_t]$. Similarly, $\psidr(\tau_H)$ is the doubly robust influence function for the PCI problem at $t=1$, and so $\eb[\psidr(\tau_H)] = \es{2}[Y_1] = \ldots = v_\gamma(\pi_e)$.
That is, we recursively apply the improved identification theory of \citet{kallus2021causal} to a nested sequence of PCI problems. In each step of the induction, we apply \cref{asm:pci-reduction,asm:bridge} with the specific outcome variable $Y_t$. 
We provide full proof details in \cref{apx:identification-pci}, where we also present a slightly more general result that allows for alternatives to $\psidr$ that instead resemble importance sampling or direct method estimators for the MDP setting.

\subsection{Specific Proximal Causal Inference Reductions and Resulting Identification}
\label{sec:pci-reduction-kinds}

Next, we provide some discussion of how to actually construct a valid PCI reduction; that is, how to choose $Z_t$, $W_t$, and $U_t$ that satisfy \cref{asm:pci-reduction}.
We provide several options of how this reduction may be performed, and discuss in each case the assumptions that would be required of the POMDP and $\pi_e$ for identification based on our results. In all cases that we consider below, we would need to additionally justify \cref{asm:bridge}, which implicitly requires some additional completeness conditions on the choices of $Z_t$, $W_t$, and $U_t$. Furthermore, we note that the practicality of any given reduction would depend heavily on how well-correlated $W_t$ and $Z_t$ are for each $t$, which in turn would impact how easily the required nuisance functions $q^{(t)}$ and $h^{(t)}$ could be fit. We summarize these reductions in \cref{tab:pci-reductions}.

\begin{table}
    \centering
    \begin{tabular}{lllll}
    \hline
    \textbf{PCI Reduction} & $\bf{Z_t}$ & $\bf{W_t}$ & $\bf{U_t}$ & \textbf{Input to} $\bf{\pi_e^{(t)}}$ \\
    \hline
    curr. and prev. obs.  & $O_{t-1},A_{t-1},R_{t-1}$ & $O_t$ & $S_t$ & $O_t$ \\
    curr. and $k$-prior obs. & $O_{t-k'}$ & $O_t$ & $S_t, S_{t-k'+1}$ & $(O_t,\tau_{t-1}) \setminus \tau_{t-k'}$ \\
    two views of obs. & $O_t^{\textup{priv}}$ & $O_t^{\textup{pub}}$ & $S_t$ & $(O_t,\tau_{t-1}) \setminus O^{\textup{priv}}_{0:t}$ \\
    \hline
    \end{tabular}
    \caption{Summary of some valid PCI reductions for our Proximal RL theory. For each, we provide the explicit reduction in terms of the triplet $(Z_t,W_t,U_t)$, and we summarize what kinds of policies can be evaluated under the respective reduction (\emph{i.e.} what is the allowed input to $\pi_e^{(t)}$). For the second row, recall that $k'=\min(k,t)$, and for the third row, recall that $O_t=(O_t^{\textup{priv}},O_t^{\textup{pub}})$, where $O_t^{\textup{priv}} \indep O_t^{\textup{pub}} \mid S_t$.}
    \label{tab:pci-reductions}
\end{table}

\subsubsection{Current and previous observation}
\label{sec:prev-curr-obs}
Perhaps the most simple kind of PCI reduction would be to define $U_t=S_t$, $W_t=O_t$, and $Z_t=(O_{t-1},A_{t-1},R_{t-1})$. That is, we use the current hidden state as confounders, and we use both the observation of $S_t$ as well as the previous observation, action, reward triple as proxies for $O_t$. For this definition we define $A_0=R_0=\emptyset$. It is easy to verify that this is a valid PCI reduction (\ie satisfying \cref{asm:pci-reduction}) as long as $\pi_e^{(t)}$ depends on $(\tau_t,O_t)$ via $O_t$ only. In addition, it is easy to verify that this reduction remains valid if we replace $Z_t$ with $O_{t-1}$, which gives us a very simple and elegant reduction, at the slight cost of fewer treatment-side proxies.

This kind of reduction may be relevant in applications where the current observation of the state is considered to be rich enough for decision making, but where nonetheless it is possible that confounding is present. One example of such a setting is a noisy observation setting, where $O_t$ is a direct observation of $S_t$ that may be corrupted with some probability, as discussed in more detail in \cref{sec:experiments}. Another example where such a reduction may be desirable is when we wish to consider policies that are functions of $O_t$ only for reasons of simplicity / interpretability.
For example, if we wish to evaluate an automated policy for sepsis management, we may wish that the policy is a simple function of the patient's current state that can be understood and audited by doctors.

\subsubsection{Current and $k$-prior observation} An alternative to the previous reduction would be to define to define $U_t=(S_t,S_{t-k'+1})$, $W_t=O_t$, and $Z_t=O_{t-k'}$, for some integer $k \geq 2$, where $k' = \min(k,t)$. Note that in this reduction we can no longer include any action or reward in $Z_t$, as this would break \cref{asm:pci-reduction} in general given the definition of $\ps{t}$. This reduction allows for any policy where $\pi_e^{(t)}$ depends on $(\tau_t,O_t)$ via the data from the $k$-most recent time steps; \ie $(O_{t-k'+1:t}, A_{t-k'+1:t-1}, R_{t-k'+1:t-1})$.

This kind of reduction would be useful in applications where it is necessary to consider policies that consider a past history of observations, rather than only the most recent observation. For example, if we were considering the task of training an robot to act within an environment that it can only observe part of at each time step through its camera, it may be necessary to consider policies that use several recent observations to build a more accurate map of the environment. However, one limitation of this reduction compared to the previous is that it uses two states as its confounder, which may make \cref{asm:bridge} more difficult to satisfy. In addition, since $Z_t$ and $W_t$ are separated in time, if $k$ is large they may be weakly correlated, making bridge functions more difficult to fit.

\subsubsection{Two views of current observation} Finally, we consider a different kind of reduction, which is valid when we have two separate views of the observation; that is, we can partition each observation $O_t$ as $O_t = (O_t^{\textup{priv}},O_t^{\textup{pub}})$, where $O_t^{\textup{priv}} \indep O_t^{\textup{pub}} \mid S_t$. In this case, we can define $U_t=S_t$, $W_t=O_t^{\textup{pub}}$, and $Z_t=O_t^{\textup{priv}}$. This allows us to evaluate any policy where $\pi_e^{(t)}$ may depend on all of $\tau_t$ except for $O_{0:t}^{\textup{priv}}$.

This kind of reduction could be appealing in many settings. First of all, it may be useful for the same kinds of applications as the previous kind of reduction, as it allows us to consider policies defined on a history of past observations without incurring the costs of the same costs in terms of satisfying \cref{asm:bridge} or estimating bridge functions.
This reduction could be particularly useful when there are some observation variables that cannot be used directly for decision making. For example, in the personalized education example considered in \cref{sec:intro}, there may be certain testing-based metrics that were specifically collected with the logged data, but that would not be available when a policy was deployed. Similarly, in robotics settings as discussed earlier, there may be cheap sensors that are always available, and expensive sensors that are only available in the logged data \citep{pan2020imitation}. In this case, we could include all such unavailable covariates in $O_t^{\textup{priv}}$, and the remaining covariates in $O_t^{\textup{pub}}$, and this would allow policy evaluation with no effective restriction on the kinds of policies considered. Similarly, if certain sensitive covariates were not allowed to be included in policies \eg for ethical reasons, such covariates could be included in $O_t^{\textup{priv}}$.

\subsection{Example: Tabular POMDPs Using Previous and Current Observation as Proxies}
\label{sec:example-tabular}

Finally, we conclude this section with a discussion of our key identification assumptions for a simple tabular case, where we use the previous and current observations as proxies for the unobserved state as described in \cref{sec:prev-curr-obs}. That is, we consider settings where $U_t = S_t$, $Z_t = O_{t-1}$, $W_t=O_t$, and $\Scal$ and $\Ocal$ are both finite. 

As argued previously, this choice of proxies satisfies \cref{asm:pci-reduction} as long as $\pi_e^{(t)}$ depends on $O_t,\tau_{t-1}$ via $O_t$ only. However, it remains to also justify \cref{asm:bridge}. The following proposition allows us to rewrite the bridge equations for this simple setting in terms of some conditional 
probability matrices under the POMDP and evaluation policy $\pi_e$.

\begin{proposition}

Let $P^{(t)}(\Ob \mid \Sb)$ denote the $|\Ocal|$ by $|\Scal|$ matrix of the distribution of $O_t$ given $S_t$ in the POMDP, and let $P_e^{(t)}(\Sb' \mid \Sb)$ denote the $|\Scal|$ by $|\Scal|$ matrix of the distribution of $S_{t-1}$ given $S_t$ under rollout by $\pi_e$. In addition, for any outcome variable $Y_t = \phi(R_t,D_{t+1:H})$ and $a \in \Acal$, let $\EE_t^*[\indicator{E_t=A_t} Y_t \mid \Sb, a]$ denote the $|\Scal|$-length vector of values of $\indicator{E_t=A_t} Y_t$ given $S_t$ and $A_t=a$ under $\ps{t}$, and let $P_t^*(a \mid \Sb)^{-1}$ denote the $|\Scal|$-length vector of values of $P(A_t=a \mid S_t)^{-1}$ under $\ps{t}$. Then, using proxies $Z_t=O_{t-1}$ and $W_t=O_t$, and confounders $U_t=S_t$, the bridge equations in \cref{asm:bridge} for each $t$ correspond to solving
\begin{equation*}
    P_e^{(t)}(\Sb' \mid \Sb)^\top P^{(t)}(\Ob \mid \Sb)^\top q^{(t)}(\Ob,a) = P_t^*(a \mid \Sb)^{-1} \qquad \forall a \in \Acal
\end{equation*}
and
\begin{equation*}
    P^{(t)}(\Ob \mid \Sb)^\top h^{(t,\phi)}(\Ob, a) = \EE_t^*[\indicator{E_t=A_t} Y_t \mid \Sb, a] \qquad \forall a \in \Acal \,,
\end{equation*}
where $q^{(t)}(\Ob,a)$ and $h^{(t,\phi)}(\Ob,a)$ are the $|\Ocal|$-length vector of values of $q^{(t)}(Z_t,a)$ and $h^{(t,\phi)}(W_t,a)$ respectively.

\end{proposition}

This proposition follows trivially by applying the fact that $Z_t=O_{t-1}$, $W_t=O_t$, and $U_t=S_t$, and explicitly expanding out the conditional expectations in the bridge equations in terms of $P_e^{(t)}(\Sb' \mid \Sb)$ and $P^{(t)}(\Ob \mid \Sb)$ given the Markovian property of the POMDP conditioned on the unobserved states.

A trivial corollary of the proposition is that, if $|\Ocal| \geq |\Scal|$, and $P^{(t)}(\Ob \mid \Sb)$ and $P_e^{(t)}(\Sb' \mid \Sb)$ are both full-rank, then the above equations are always solvable for all $a \in \Acal$, no matter the outcome variable $Y_t$. This follows by using any pseudo-inverse for $P_e^{(t)}(\Sb' \mid \Sb)^\top P^{(t)}(\Ob \mid \Sb)^\top$ and $P^{(t)}(\Ob \mid \Sb)^\top$.
The conditions that $|\Ocal| \geq |\Scal|$ and that $P^{(t)}(\Ob \mid \Sb)$ is full rank are independent of the behavior or evaluation policies, and they essentially require that all distributions over states imply different distributions over observations; that is, there are no ``invisible'' aspects of $S_t$ that don't affect $O_t$. Conversely, the assumption that $P_e^{(t)}(\Scal' \mid \Scal)$ is full rank depends on the evaluation policy $\pi_e$. However, it may be justified for \emph{all} possible evaluation policies, for example if the $|\Sb|$ by $|\Sb|$ conditional probability matrix defining the transition kernel $P_T^{(t)}(S_t \mid S_{t-1}, A_{t-1}=a)$ were invertible for every $a \in \Acal$. 
In other words, we can justify \cref{asm:bridge} under some basic conditions on the underlying POMDP, which may be reasoned about on a problem-by-problem basis.

Finally, although the above analysis is specific to our example setting, the intuition is very general; in order for \cref{asm:bridge} to hold, we need that the proxies are sufficiently well correlated with the confounders (\emph{e.g.} that $P^{(t)}(\Ob \mid \Sb)$ and $P_e^{(t)}(\Sb' \mid \Sb)$ are full rank), and that they contain at least as much information as the confounders (\emph{e.g.} that we also have $|\Ocal| \geq |\Scal|$).

\section{Policy Value Estimators}
\label{sec:estimation}

Now we turn from the question of identification to that of estimation. We will focus on estimation of $v_\gamma(\pi_e)$ based on the identification result given by \cref{cor:psidr}.
We will assume in the remainder of this section that we have fixed a valid PCI reduction that satisfies \cref{asm:pci-reduction,asm:bridge}.
A natural approach to estimating $v_\gamma(\pi_e)$ based on \cref{cor:psidr} would be to use an estimator of the kind
\begin{equation}
\label{eq:plugin-estimator}
    \hat v^{(n)}_\gamma(\pi_e) = \frac{1}{n} \sum_{i=1}^n \widehat\psidr(\tau_H^{(i)}) \,,
\end{equation}
where $\widehat\psidr$ is an approximation of $\psidr$ using plug-in estimators for the nuisance functions $h^{(t)}$ and $q^{(t)}$ for each $t$.
Specifically, to eschew assumptions on the nuisance function estimators aside from rates, we will use a cross-fitting estimation technique \citep{chernozhukov2016double,zheng2011cross}. Namely, fixing $K\geq2$, for each $k=1,\dots,K$: (1) for $t=1,\dots,H$, we fit estimators $\hat h^{(t,k)}$ and $\hat q^{(t,k)}$ only on the observed trajectories $i=1,\dots,n$ with $i\neq k-1\,(\operatorname{mod}\,K)$; (2) and then for $i=1,\dots,n$ with $i= k-1\,(\operatorname{mod}\,K)$, we set $\widehat\psidr(\tau_H^{(i)})$ to be $\psidr(\tau_H^{(i)})$ where we replace $h^{(t)},q^{(t)}$ with $\hat h^{(t,k)},\hat q^{(t,k)}$. Then we use these to construct an estimator by taking an average as in \cref{eq:plugin-estimator}.
We discuss exactly how we fit nuisance estimators given trajectory data in \cref{sec:nuisnace}. Until then, for \cref{sec:consistency,sec:efficiency}, we keep this abstract and general: we will only impose assumptions about the rates of convergence of nuisance estimators and that we used cross-fitting so that $\tau_H^{(i)}$ is independent of $\hat h^{(t,k)},\hat q^{(t,k)}$ whenever $i= k-1\,(\operatorname{mod}\,K)$.

\subsection{Consistency and Asymptotic Normality}\label{sec:consistency}

We first consider conditions under which the estimator $\hat v^{(n)}_\gamma(\pi_e)$ is consistent and asymptotically normal. For this, we need to make some assumptions on the quality of our nuisance estimators.

\begin{assumption}
\label{asm:estimation-error}
\textbf{Consistent and bounded nuisance estimates}:
letting $\Psi$ represent $q^{(t)}(Z_t,A_t)$ or $h^{(t)}(W_t,a)$ for any $t \in [H]$ and $a \in \Acal$, we have that for each $k\in[K]$:
\begin{enumerate}
    \item $\|\hat\Psi^{(k)}-\Psi\|_{2,\pb}=o_p(1)$
    \item $\|\hat\Psi^{(k)}\|_{\infty}=O_p(1)$
    \item $\|\Psi\|_\infty<\infty$
\end{enumerate} 
\textbf{Nuisance estimation rates}: The following stochastic bounds hold over the sampling distributions for constructing the estimators $\hat q^{(1,k)}$ and $\hat h^{(t,k)}$ for all $t \in [H]$ and $k \in [K]$:
\begin{enumerate}
\setcounter{enumi}{3}
    \item For each $t\in[H]$, $a\in\aset$, and $k\in[K]$, we have
    \begin{equation*}
        \|\hat q^{(t,k)}(Z_t,A_t)-q^{(t)}(Z_t,A_t)\|_{2,\pb} \|\hat h^{(t,k)}(Z_t,a)-h^{(t)}(Z_t,a)\|_{2,\pb}=o_p(n^{-1/2})
    \end{equation*}
    \item for each $t\in[H]$, $t'<t$, $a\in\aset$, and $k\in[K]$, we have
    \begin{equation*}
        \|\hat q^{(t',k)}(Z_{t'},A_{t'})-q^{(t')}(Z_{t'},A_{t'})\|_{2,\pb} \|\hat h^{(t,k)}(Z_t,a)-h^{(t)}(Z_t,a)\|_{2,\pb}=o_p(n^{-1/2})
    \end{equation*}
    \item for each $t\in[H]$, $t'<t$, and $k\in[K]$, we have
    \begin{equation*}
        \|\hat q^{(t',k)}(Z_{t'},A_{t'})-q^{(t')}(Z_{t'},A_{t'})\|_{2,\pb} \|\hat q^{(t,k)}(Z_t,A_t)-q^{(t)}(Z_t,A_t)\|_{2,\pb}=o_p(n^{-1/2})
    \end{equation*}
\end{enumerate}

\end{assumption}

Essentially, \cref{asm:estimation-error} requires that the nuisances $q^{(t)}$ and $h^{(t)}$ are estimated consistently in terms of the $L_{2,\pb}$ functional norm for each $t$, and that the corresponding product-error terms converge faster than $n^{-1/2}$ rate. This could be achieved, for example, if each nuisance by itself were estimated at a $o_p(n^{-1/4})$ rate, which notably permits slower-than-parametric rates and is obtainable for many non-parametric machine-learning-based methods \citep{chernozhukov2016double}.
In particular, there is a very established line of work on establishing rates like these for conditional moment problems, like those defining $q^{(t)}$ and $h^{(t)}$, in terms of projected error (\emph{e.g.} obtaining rates for $\|\EE[\hat h^{(t,k)}(W_t,A_t) - h^{(t)}(W_t,A_t) \mid Z_t,A_t]\|_2$,) using \emph{e.g.} sieve methods \citep{chen2009efficient,chen2012estimation} or minimax methods with general machine learning classes \citep{dikkala2020minimax}. These can be translated to corresponding rates for the actual $L_2$ error (\emph{e.g.} $\|\hat h^{(t,k)}(W_t,A_t) - h^{(t)}(W_t,A_t)\|_2$) given assumptions on so-called ``ill-posedness'' measures (see \emph{e.g.} \citet{chen2012estimation},) which can be used to ensure our required rates. Alternatively, there exist methods that can directly obtain $L_2$ error rates for such conditional moment problems, by leveraging so-called ``source conditions'' \citep[Definition 3.4]{carrasco2007linear}, for example using regularized sieve methods \citep{florens2011identification}, neural nets with Tikhonov regularization \citep{liao2020provably}, or kernel methods with spectral regulariztion \citep{wang2022spectral}.
Also note that the product-rate condition allows for some trade off where, if some nuisances can be estimated faster, then other nuisances can be estimated even slower than $o_p(n^{-1/4})$.
In addition, we require a technical boundedness condition on the uniform norm of the errors and of the true nuisances themselves. Given this, we can now present our main consistency and asymptotic normality theorem.

\begin{theorem}
\label{thm:consistency}
Let the conditions of \cref{thm:identification-pci} be given, and assume that the nuisance functions plugged into $\hat v^{(n)}_\gamma(\pi_e)$ are estimated using cross fitting. Furthermore, suppose that the nuisance estimation for each cross-fitting fold satisfies \cref{asm:estimation-error}. Then, we have 
\begin{align*}
&\sqrt{n}(\hat v^{(n)}_\gamma(\pi_e) - v_\gamma(\pi_e)) \to \mathcal N(0, \sigma^2_{\text{DR}})\quad\text{in distribution},\\
    &\text{where}\quad\sigma^2_{\text{DR}} = \e_{\pb}[(\psidr(\tau_H) - v_\gamma(\pi_e))^2]
\end{align*}
\end{theorem}

The key step in proving \cref{thm:consistency} is to establish that $\psidr$ enjoys Neyman orthogonality with respect to all nuisance functions, and in particular characterizing the unique product structure of the bias. Having established this, we proceed by applying the machinery of theorem 3.1 of \citet{chernozhukov2016double}.
We refer the reader to the appendix for the detailed proof. 

One technical note about this theorem is that there may be multiple $q^{(t)}$ and $h^{(t)}$ that solve \cref{eq:q,eq:h}, which creates some ambiguity in both \cref{asm:estimation-error} and the definition of $\psidr(\tau_H)$. This is important, since the ambiguity in the definition of $\psidr(\tau_H)$ affects the value of the asymptotic variance $\sigma^2_{\text{DR}}$. In this case, we implicitly assume that \cref{asm:estimation-error} holds for some arbitrarily given solutions $q^{(t)}$ and $h^{(t)}$ for each $t \in [H]$, and that $\sigma^2_{\text{DR}}$ is defined using the same $q^{(t)}$ and $h^{(t)}$ solutions. Thus, our consistency result in \cref{thm:consistency} holds even when bridge functions are non-unique.

Finally, we briefly consider how this variance grows in terms of $H$. Since $\varphi_{DR}(\tau_H)$ consists of a sum of $H$ terms, each of which is multiplied by $\eta_t = \prod_{s'=0}^{t-1} q^{(s')}(Z_{s'},A_{s'}) \indicator{E_{s'}=A_{s'}}$, we can generally bound the efficient asymptotic variance by $\sum_{t=1}^H \prod_{s=1}^t \|q^{(s)}(Z_s,A_s)\|_\infty ( \|q^{(t)}(Z_t,A_t)\|_2 + \sum_{a \in \Acal} \|h^{(t)}(W_t,a)\|_2 + \|q^{(t)}(Z_t,A_t)\|_\infty \|h^{(t)}(W_t,A_t)\|_2)$. Therefore, assuming that all functions $h^{(t)}(W_t,A_t)$ and $q^{(t)}(Z_t,A_t)$ have $\|\cdot\|_\infty$ norm of the same order $H$ grows, the asymptotic variance should grow roughly as $\mathcal{O}(H^2)$ as $H \to \infty$. On the other hand, if the inverse problems for $q^{(t)}$ and $h^{(t)}$ grow increasingly ill-conditioned as $t$ increases, then the norms of these functions may grow, in which case the growth of asymptotic variance may be worse than quadratic.

\subsection{Semiparametric Efficiency}
\label{sec:efficiency}

We now consider the question of \emph{semiparametric efficiency} of our OPE estimators. Semiparametric efficiency is defined relative to a model $\mathcal M$, which is a set of allowed distributions such that $\pb \in \mathcal M$.
Roughly speaking, we say that an estimator is semiparametrically efficient w.r.t. $\mathcal M$ if it is regular (meaning invariant to $O_p(1/\sqrt{n})$ perturbations to the data-generating process that keep it inside $\mathcal M$), and achieves the minimum asymptotic variance of all regular estimators. We provide a summary of semiparametric efficiency as it pertains to our results in \cref{apx:semiparametric}, but for the purposes of this section it suffices to say that, under conditions we establish, there exists a function $\psieff \in L_{2,\pb}(\tau_H)$, called the ``efficient influence function'' w.r.t. $\mathcal M$, and that an estimator $\hat v^{(n)}_\gamma(\pi_e)$ is efficient w.r.t. $\mathcal M$ if and only if $\sqrt{n}(\hat v^{(n)}_\gamma(\pi_e) - v_\gamma(\pi_e)) = n^{-1/2} \sum_{i=1}^n\psieff(\tau_H^{(i)}) + o_p(1)$, that is, asymptotically it looks like simple sample average of this function. 

One complication in considering models of distributions on $\tau_H$ is that technically the definition of $v_\gamma(\pi_e)$ depends on the full distribution of $\tau_H^{\textup{full}}$. In the case that the distribution of $\tau_H$ corresponds to the logging distribution induced by some behavior policy and underlying POMDP that satisfies \cref{asm:bridge}, it is clear from \cref{thm:identification-pci} that using \emph{any} nuisances satisfying the required conditional moments will result in the same policy value estimate $v_\gamma(\pi_e)$. However, if we allow for distributions on $\tau_H$ that do not necessarily satisfy such conditions, as is standard in the literature on policy evaluation, it may be the case that different solutions for $h^{(t)}$ and $q^{(t)}$ result in different values of $\e_{\mathcal P}[\psidr(\tau_H)]$. To avoid such issues, we consider a model of distributions where the nuisances and corresponding policy value estimate are uniquely defined, as follows.
\begin{definition}[Model and Target Parameter]
\label{def:model}
Define $\mt{0}$ as the set of all distributions on $\tau_H$, and for each $t \geq 1$ recursively define:
\begin{enumerate}
    \item $\eta_{t,\mathcal P} = \prod_{s=1}^{t-1} q^{(s)}_{\mathcal P}(Z_s,A_s) \indicator{A_s = E_s}$
    \item $P^*_{t, \mathcal P}(A_t \mid W_t) = \e_{\mathcal P}[\eta_{t, \mathcal P} \mid W_t, A_t] P_{\mathcal P}(A_t \mid W_t)$
    \item $(T_{t,\Pcal} g)(W_t,A_t) = \e_{\mathcal P}[\eta_{t,\mathcal P} g(Z_t, A_t) \mid W_t, A_t]$ for all $g \in L_{2,\Pcal}(Z_t,A_t)$
    \item $\mt{t} = \mt{t-1} \cap \{\mathcal P : T_{t,\mathcal P} \text{ is invertible and } P^*_{t,\mathcal P}(A_t \mid W_t)^{-1} \in L_{2,\mathcal P}(W_t,A_t)\}$
    \item $q^{(t)}_{\mathcal P}(Z_t,A_t) = T_{t,\mathcal P}^{-1} \left( P^*_{t,\mathcal P}(A_t \mid W_t)^{-1} \right)$
\end{enumerate}
where (1-3) are defined for $\mathcal P \in \mt{t-1}$, and (5) for $\mathcal P \in \mt{t}$. Furthermore, let $T^*_{t,\mathcal P}$ denote the adjoint of $T_{t,\mathcal P}$, define $Y_H=R_h$, and for each $t \in [H]$ and $\mathcal P \in \mt{t}$ recursively define
\begin{enumerate}
      \setcounter{enumi}{5}
    \item $\mu_{t,\mathcal P}(Z_t,A_t) = \e_{\mathcal P}[\eta_{t,\mathcal P} \indicator{A_t=E_t} Y_{t,\mathcal P} \mid Z_t, A_t]$
    \item $h^{(t)}_{\mathcal P}(W_t,A_t) = (T^*_{t,\mathcal P})^{-1}\left( \mu_{t,\mathcal P}(Z_t,A_t)\right)$
    \item $\psi_{t,\Pcal} = \sum_{a \in \aset} h^{(t)}_{\mathcal P}(W_t,a) + q^{(t)}_{\mathcal P}(Z_t,A_t) \left( \indicator{A_t=E_t} Y_{t,\mathcal P} - h^{(t)}_{\mathcal P}(W_t,A_t) \right)$
    \item $Y_{t-1,\mathcal P} = R_{t-1} + \gamma \psi_{t,\Pcal}$
\end{enumerate}
where the latter is only defined for $t>1$. Finally, let $\mpci = \mt{H}$, and for each $\mathcal P \in \mpci$ define
\begin{equation*}
    V(\mathcal P) = \e_{\mathcal P} \left[ \sum_{a \in \aset} h^{(1)}_{\mathcal P}(W_1,a) \right]\,.
\end{equation*}
\end{definition}

We note that this definition is not circular, since $\eta_{1,\mathcal P} = 1$ for every $\mathcal P$, and so we can concretely define the first set of quantities in the order they are listed above for each $t \in [H]$ in ascending order, and the second set in descending order of $t$. We note that in the case that $\mathcal P = \pb$, it is straightforward to reason that $\eta_{t,\pb}$, $q_{\pb}^{(t)}$, $h_{\pb}^{(t)}$, and $Y_{t,\pb}$ agree with the corresponding definitions in \cref{thm:identification-pci,cor:psidr}: $T_{t,\pb}$ and $T^*_{t,\pb}$ correspond to standard conditional expectation operators under $\ps{t}$, $P^*_{t,\pb}(A_t \mid W_t) = \pd{t}(A_t \mid W_t)$, and $V(\pb) = v_\gamma(\pi_e)$. Therefore, $\mpci$ is a natural model of observational distributions where the required nuisances are uniquely defined, and $V(\mathcal P)$ is a natural and uniquely defined generalization of $v_\gamma(\pi_e)$ for distributions $\mathcal P$ that do not necessarily correspond to actual logging distributions satisfying \cref{asm:bridge}.

Finally, we assume the following the following on the actual observed distribution $\pb$.
\begin{assumption}
\label{asm:semiparametric-reg}
    For every sequence of distributions $\mathcal P_n$ that converge in law to $\pb$, there exists some integer $N$ such that for all $n\geq N$ and $t \in [H]$ such that $T_{t,\mathcal P_n}$ and $T^*_{t,\mathcal P_n}$ are invertible. Furthermore, for all such sequences and $t \in [H]$ we also have
    \begin{enumerate}
        \item $\liminf_{n \to \infty} \inf_{\|f(Z_t,A_t)\|_{1,\mathcal P_n} \geq 1} \| T_{t,\mathcal P_n} f(Z_t,A_t)\|_{1,\mathcal P_n} > 0$
        \item $\liminf_{n \to \infty} \inf_{\|g(W_t,A_t)\|_{1,\mathcal P_n} \geq 1} \| T^*_{t,\mathcal P_n} g(W_t,A_t)\|_{1,\mathcal P_n} > 0$
        \item $\limsup_{n \to \infty} \|P^*_{t,\mathcal P_n}(A_t \mid W_t)^{-1}\|_\infty < \infty$\,.
    \end{enumerate}
    In addition, for each $t \in [H]$ the distribution $\pb$ satisfies
    \begin{enumerate}
      \setcounter{enumi}{3}
        \item $\inf_{\|f(Z_t,A_t)\|_{2,\pb} \geq 1} \| T_{t,\mathcal P_n} f(Z_t,A_t)\|_{2,\pb} > 0$
        \item $\inf_{\|g(W_t,A_t)\|_{2,\mathcal P_n} \geq 1} \| T^*_{t,\mathcal P_n} g(W_t,A_t)\|_{2,\pb} > 0$\,.
    \end{enumerate}
\end{assumption}
The condition that $T_{t,\mathcal P_n}$ and $T^*_{t,\mathcal P_n}$ are invertible for large $n$ ensures that the model $\mpci$ is locally saturated at $\pb$, and the additional conditions ensure that the nuisance functions can be uniformly bounded within parametric submodels. 
These are very technical conditions used in our semiparametric efficiency proof, and it may be possible to relax them.
We note also that in discrete settings, these conditions follow easily given $\pb \in \mpci$, since in this setting the conditions can be characterized in terms of the entries or eigenvalues of some probability matrices being bounded away from zero, which by continuity must be the case when $\mathcal P_n$ is sufficiently close to $\pb$. 
Importantly, the locally saturated condition on $\mpci$ at $\pb$ means that the relevant tangent space is unrestricted. (See \cref{apx:tangent-space} for a discussion of issues with the tangent space in past work in the absence of local saturation.)

Given this setup, we can now present our main efficiency result.
\begin{theorem}
\label{thm:efficiency}
Suppose that $\pb$ is the observational distribution given by a POMDP and logging policy that satisfies the conditions of \cref{thm:identification-pci}, and let \cref{asm:semiparametric-reg} be given. Then, $\psidr(\tau_H) - v_\gamma(\pi_e)$ is the efficient influence function for $V(\mathcal P)$ at $\mathcal P = \pb$.
\end{theorem}

Finally, the following corollary combines this result with \cref{thm:consistency}, which shows that under the same conditions, if the nuisances are appropriately estimated then the resulting estimator will achieve the semiparametric efficiency bound relative to $\mpci$.
\begin{corollary}
Let the conditions of \cref{thm:consistency,thm:efficiency} be given. Then, the estimator $\hat v^{(n)}_\gamma(\pi_e)$ is semiparametrically efficient w.r.t. $\mpci$.
\end{corollary}

\subsection{Nuisance Estimation}\label{sec:nuisnace}

Finally, we conclude this section with a discussion of how we may actually estimate $q^{(t)}$ and $h^{(t)}$. The conditional moment equations \cref{eq:q,eq:h} defining these nuisances are defined in terms of the intervention distributions $\ps{t}$, which are not directly observable. Therefore, we provide the following lemma, which re-frames these as a nested series of conditional moment restrictions under $\pb$.

\begin{lemma}
\label{lem:nuisance-estimation}
Let the conditions of \cref{thm:identification-pci} be given. Then, for any collection of functions $q^{(1)},\ldots,q^{(H)}$ and $h^{(1)},\ldots,h^{(H)}$, these functions satisfy \cref{eq:q,eq:h} for every $t \in [H]$ if and only if for every $t \in [H]$ we have
\begin{align*}
    \e_{\pb} \left[ \eta_t \left(  g(W_t, A_t) q^{(t)}(Z_t, A_t) - \sum_{a \in \mathcal A} g(W_t, a) \right) \right] &= 0 \qquad \forall \text{ measurable } g \\
    \text{and} \qquad \e_{\pb} \left[ \eta_t f(Z_t, A_t) \Big( h^{(t)}(W_t,A_t) - \indicator{E_t=A_t} Y_t \Big) \right] &= 0 \qquad \forall \text{ measurable } f \,,
\end{align*}
where $\eta_t$ and $Y_t$ are defined as in \cref{thm:identification-pci}.
\end{lemma}

We can observe that the moment restrictions defining $q^{(t)}$ for each $t$ depend only on $q^{(t')}$ for $t' < t$, and those defining $h^{(t)}$ for each $t$ depend on $h^{(t')}$ for $t' > t$ and on $q^{(t'')}$ for every $t'' \neq t$. This suggests a natural order for estimating these nuisances, of $q^{(1)}$ through $q^{(H)}$ first, and then $h^{(H)}$ through $h^{(1)}$. We now take this approach, solving an estimate of the continuum of moment conditions in each round. (An alternative approach may be to jointly solve for all $2H$ nuisances together.)
Set 
\begin{align*}
    U^{(q,t)}(q,g) &= \hat\eta_t \left(  g(W_t, A_t) q(Z_t, A_t) - \sum_{a \in \mathcal A} g(W_t, a) \right) \\
    U^{(h,t)}(h,f) &= \hat\eta_t f(Z_t, A_t) \left( h(W_t,A_t) - \indicator{E_t=A_t} \hat Y_t \right) \,,
\end{align*}
where $\hat\eta_t$ and $\hat Y_t$ are estimated by plugging in the preceding nuisance estimators (in the ordering described above).
Following \citet{bennett2023variational}, the continuum of moment conditions $\{q: \e_{\pb}U^{(q,t)}(q,g)=0 \ \forall g\}$ or $\{h: \e_{\pb}U^{(h,t)}(h,f)=0 \ \forall f\}$ can be efficiently solved using a regularized, variational reformulation of the optimally weighted generalized method of moments \citep{hansen1982large}, known as the variational method of moments (VMM).
This gives our following proposed estimators for solving for this nuisance bridge functions:
\begin{proposition}
\label{prop:nuisance-estimators}
Our VMM estimators for the nuisance functions $q^{(1)},\ldots,q^{(H)}$ and $h^{(1)},\ldots,h^{(H)}$ take the form
\begin{align*}
    q^{(t)} &= \argmin_{q \in \mathcal Q^{(t)}} \sup_{g \in \mathcal G^{(t)}}~ \e_n[U^{(q,t)}(q,g)] - \frac{1}{4} \e_n[U^{(q,t)}(\tilde q_t, g)^2] + \mathcal R^{(q,t)}(q) - \mathcal R^{(g,t)})(g), \\
    h^{(t)} &= \argmin_{h \in \mathcal H^{(t)}} \sup_{f \in \mathcal F^{(t)}}~ \e_n[U^{(h,t)}(h,f)] - \frac{1}{4} \e_n[U^{(h,t)}(\tilde h_t, f)^2] + \mathcal R^{(h,t)}(h) - \mathcal R^{(f,t)}(f) \,,
\end{align*}
and can be sequentially solved for in the order $q^{(1)}$ through $q^{(H)}$ then $h^{(H)}$ through $h^{(1)}$,
where $\mathcal Q^{(t)}$ and $\mathcal H^{(t)}$ are hypothesis classes for the functions $q^{(t)}$ and $h^{(t)}$ respectively, $\mathcal G^{(t)}$ and $\mathcal F^{(t)}$ are some critic function classes corresponding to the set of moments we are enforcing, $\mathcal R^{(q,t)}$, $\mathcal R^{(g,t)}$, $\mathcal R^{(h,t)}$, and $\mathcal R^{(f,t)}$ are regularizers, and $\tilde q^{(t)}$ and $\tilde h^{(t)}$ are some prior estimates of $q^{(t)}$ and $h^{(t)}$ which are arbitrarily defined and need not necessarily be consistent.
\end{proposition}

There are many existing methods for solving empirical minimax equations of these kinds for different kinds of function classes $\Qcal^{(t)}$ and $\Hcal^{(t)}$, as well as different kinds of corresponding critic classes $\Gcal^{(t)}$ and $\Fcal^{(t)}$. In particular, in \cref{apx:nuisance-estimation} we provide a detailed derivation and description of an efficient process for solving these equations when the two critic classes are given by Reproducing Kernel Hilbert Spaces (RKHSs), and we regularize them using squared RKHS norm. Note that this approach is very generic, and allows for any function classes $\Qcal^{(t)}$ and $\Hcal^{(t)}$ that we can efficiently minimize convex losses over.

\section{Experiments}
\label{sec:experiments}

Finally, we present a series of experiments to demonstrate our method and theory.
We present two sets of experiments. First, we present a simple toy scenario, where we explore the behavior of the methodology and provide a ``proof of concept'' of our theory. Second, motivated by the findings of of our first experiments, we benchmark our methodology in a confounded variation of the more complex ``sepsis simulator'' environment of \citet{oberst2019counterfactual}, which is a better reflection of real application. For full details of all experiments, see our code at \url{https://github.com/CausalML/ProximalRL}.

\subsection{Experiment 1: Toy Scenario}

\subsubsection{Experimental Setup}

\begin{figure}[t]
    \centering
    \begin{tabular}{cc}
         \includegraphics[width=0.47\textwidth]{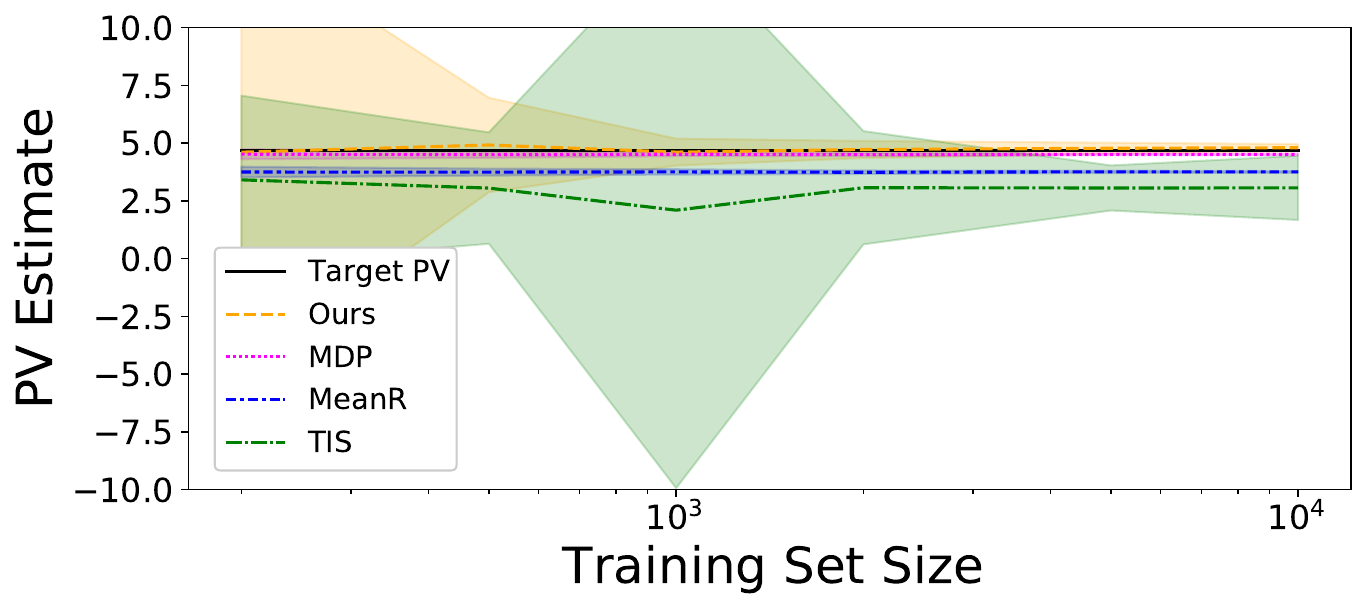} & 
         \includegraphics[width=0.47\textwidth]{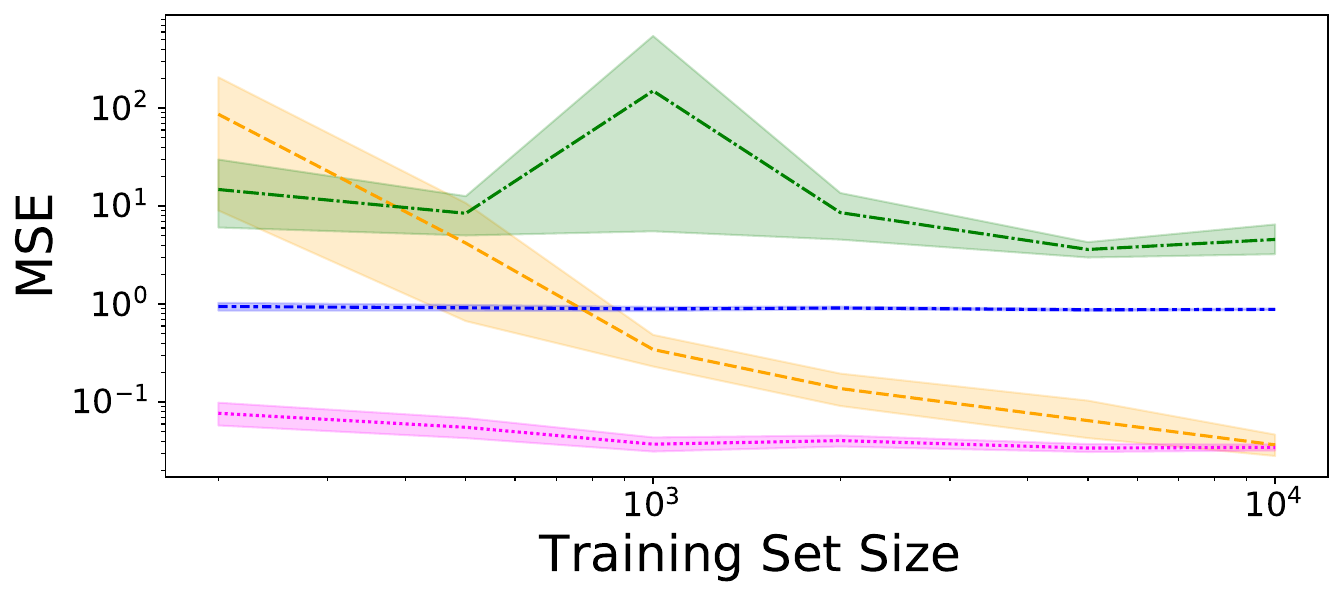} \\
         \includegraphics[width=0.47\textwidth]{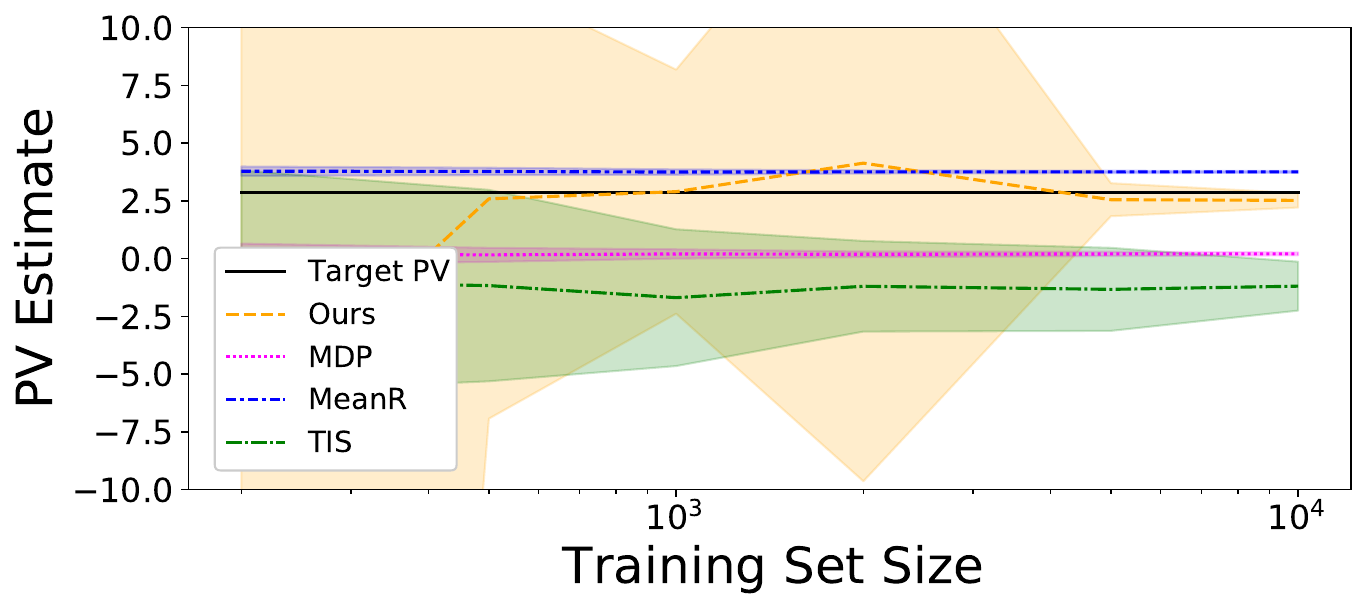} & 
         \includegraphics[width=0.47\textwidth]{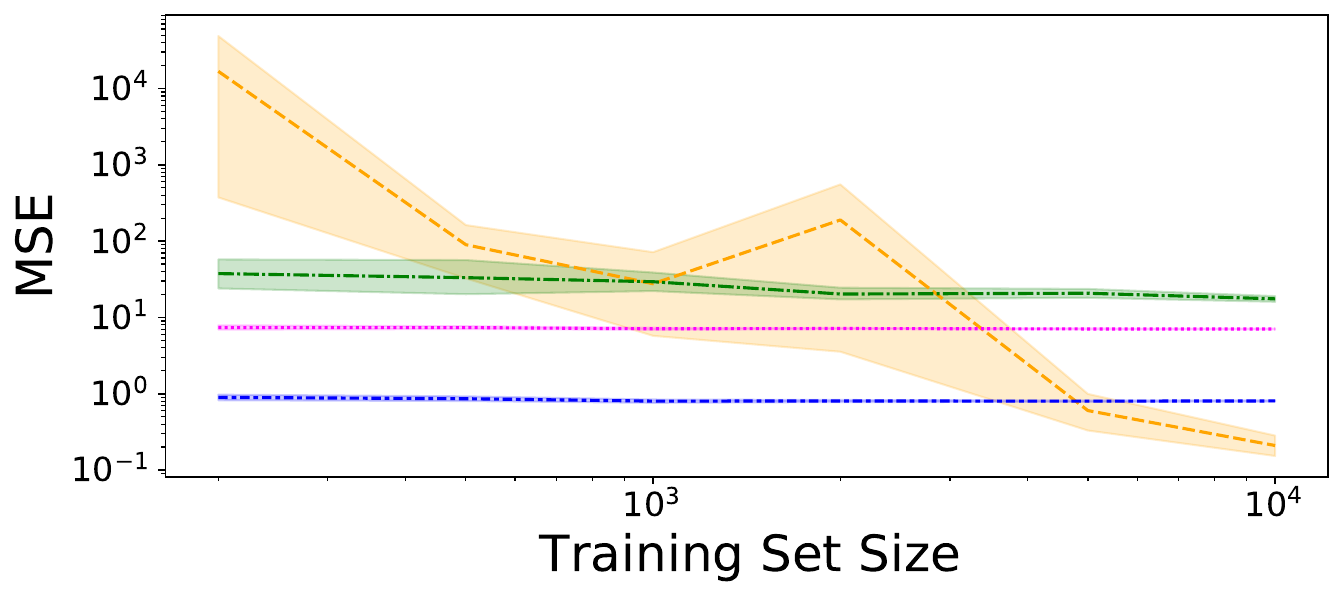} \\
         \includegraphics[width=0.47\textwidth]{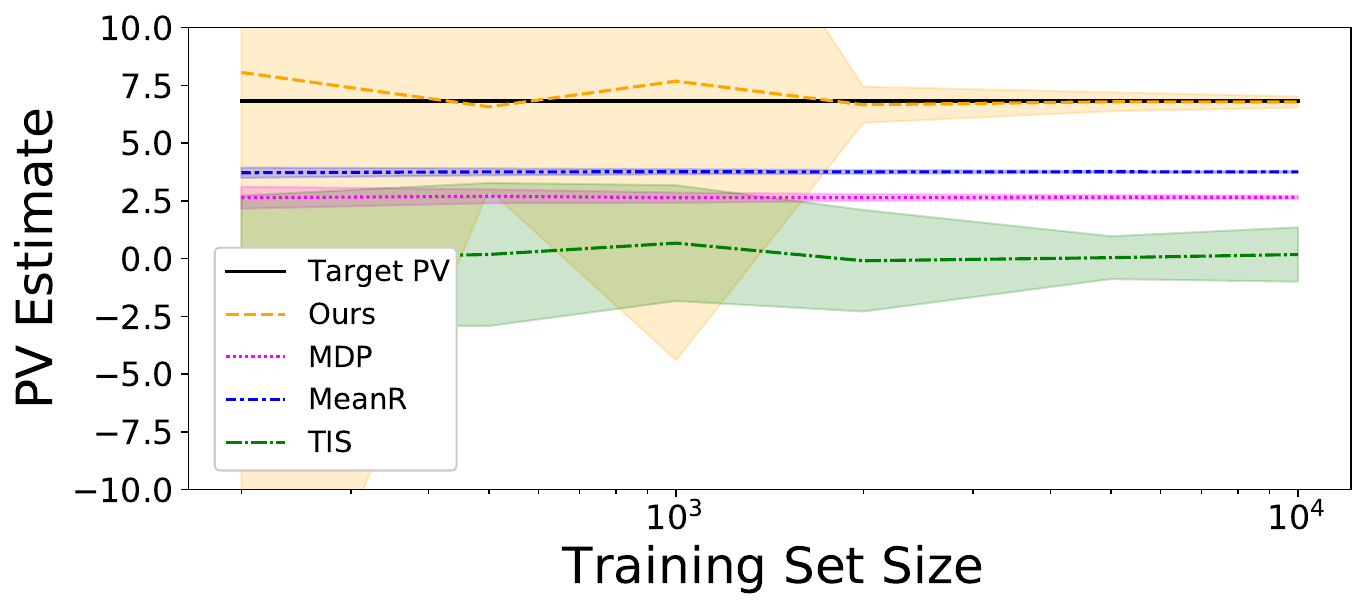}  &
         \includegraphics[width=0.47\textwidth]{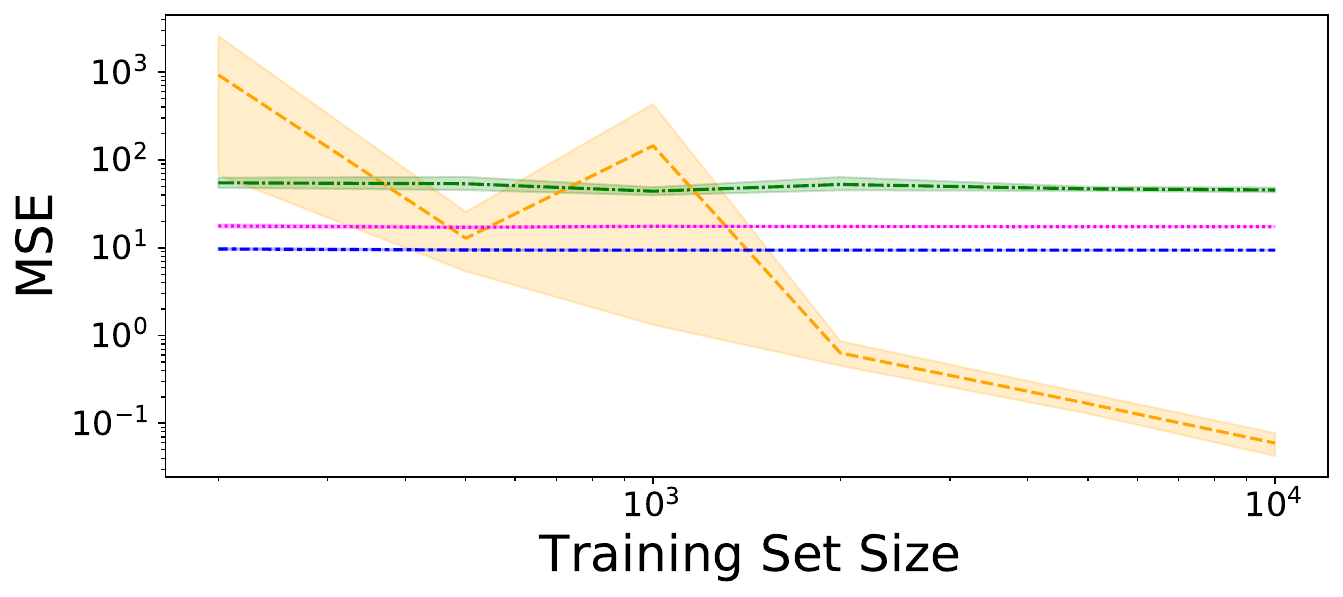}  \\
    \end{tabular}
    \caption{Results of our Proximal RL experiment on the \noisyobs environment with $\epsnoise=0.2$. In the top, middle, and bottom rows we display results for $\pieasy$, $\pihard$, and $\pioptim$ respectively. On the left we display the mean policy value estimate of each method, where the solid black line corresponds to the true policy value, and the shaded regions correspond to one standard deviations of the policy value estimates. On the right we display the corresponding mean squared error of these estimates, where the shaded regions correspond to 95\% confidence intervals for these values.}
    \label{fig:pomdp-results}
\end{figure}

For our first experiment, we consider a simple POMDP, which we refer to as \noisyobs, which is a time-homogeneous POMDP with three states, two actions, and three observation values. We denote these by $\sset = \{s_1,s_2,s_3\}$, $\aset = \{a_1,a_2\}$, and $\oset = \{o_1,o_2,o_3\}$. We detail the state transition, reward, and initial state distribution of the POMDP in \cref{apx:experiment-details}.
The observation emission process for \noisyobs is given $P_O^{(t)}(o_i \mid s_j) = \indicator{i=j}(1 - 3\epsnoise/2) + \epsilon/2$, where $\epsnoise$ is a parameter of the POMDP. This models a noisy observation of the state, since we observe the correct state with probability $1-\epsnoise$, or a randomly selected incorrect state otherwise. Thus if $\epsnoise=0$ there is no confounding, and greater $\epsnoise$ indicate more noisy measurements.

We collected logged data using a time-homogeneous behavioral policy $\pinoisyobs$, with a horizon length $H=3$.
We considered three different evaluation policies $\pieasy$, $\pihard$, and $\pioptim$, which are all also time-homogeneous and depend only on the current observation, and are detailed in \cref{apx:experiment-details}. These polices are so named because $\pieasy$ and $\pihard$ are are designed to have high and low overlap with the logging policy respectively, and $\pioptim$ is the optimal policy when $\epsnoise$ is sufficiently small. Therefore these cover a wide range of different kinds of policies.
In all cases, we set $\gamma=1$.

We performed policy evaluation with the following methods:
(1) \textbf{Ours} is the efficient estimator discussed in \cref{sec:estimation}, with nuisance estimation performed using the sequential procedure described in \cref{sec:nuisnace};
(2) \textbf{MeanR} is a naive unadjusted baseline given by $\frac{1}{n} \sum_{i=1}^n \sum_{t=1}^H \gamma^t R_t^{(i)}$;
(3) \textbf{MDP} is a model-based baseline given by fitting a tabular MDP to the observed data, treating the observations as states, and computing the value of $\pi_e$ on this model; and
(4) \textbf{TIS} is a baseline based on the result in \cref{thm:tennenholtz}, with estimated plugged-in nuisances and replacing the expectation under $\pind$ with its empirical analogue.
We provide more detail about each of these methods in \cref{apx:experiment-details}. In the case of our method, we used a simplified version of the ``current and previous observation'' PCI reduction given by the first row of \cref{tab:pci-reductions}, where $Z_t=O_{t-1}$ and $W_t=O_t$, which is valid since we are considering evaluation policies that only depend on $O_t$.

\subsubsection{Results}

We now present results policy evaluation for for the above scenario and policies, using both our method and the above benchmarks. Specifically, for each $n \in \{200,500,1000,2000,5000,10000\}$, $\pi_e \in \{\pieasy, \pihard, \pioptim\}$, and $\epsilon \in \{0, 0.2\}$ we repeated the following process $100$ times: (1) we sampled $n$ trajectories with horizon length $H=3$, behavior policy $\pinoisyobs$ and noise level $\epsnoise=\epsilon$; and (2) estimated $v_1(\pi_e)$ using these $n$ trajectories for each method. 

In \cref{fig:pomdp-results} we display results for the confounded case where $\epsnoise=0.2$ (\ie, POMDP setting). Here, we see that our method is consistent, while the \textsc{MDP} method, which is only designed to work in MDP settings, is not. The only exception is for estimating the value of $\pieasy$, however this is only because \textsc{MDP} just happens to have very small bias for estimating this policy. 
While our method is consistent, it does have more variance than the \textsc{MDP} benchmark as it tackles a much more complex estimation problem.
As expected, the unadjusted \textsc{MeanR} benchmark is inconsistent as it only estimates the value of the logging policy. Finally, despite our identification theory in \cref{sec:ident-ind}, the \textsc{TIS} method in general performs very poorly. This is unsurprising, since as discussed in \cref{sec:ident-ind} the identification result (as an expectation over $\pind$) may not lend itself to good estimation by plugging in empirical estimates into the identification formula.
For comparison, in \cref{apx:experiment-details}, we present additional results for the \emph{un}confounded case, $\epsnoise=0$ (\ie, MDP setting), where we see that the \textsc{MDP} baseline becomes consistent due to the absence of confounding and that our method remains consistent and has less variance than in the POMDP setting shown here but still more than the \textsc{MDP} baseline, which is expected as it still solves a more complex estimation problem in order to adapt to both the MDP and POMDP settings.

\subsection{Experiment 2: Sepsis Management}
\label{sec:experiments-sepsis}

\subsubsection{Experimental Setup}
Next, we consider a more ``real world''-inspired scenario. Specifically, we consider a scenario based on the sepsis management simulator of \citet{oberst2019counterfactual}. Their environment considers the active management of sepsis for patients, whose state is described by heart rate, blood pressure, oxygen concentration, glucose level, and whether the patient is diabetic. At each time step, the action taken consists of three binary components: whether to place the patient on/off antibiotics, whether to place them on/off vasopressors, and whether to place them on/off a ventilator, giving a total of 8 unique actions. After taking each action, we receive a reward based on the number of components of the state taking values within safe ranges, with a maximum reward of $1$ if all indicators are safe and the patient is off all three treatments, and a minimum reward of $-10$ if three more more indicators are unsafe, with various intermediate values. The system uses almost identical parameters as in \citet{oberst2019counterfactual} with some minor modifications, and we provide a more detailed description in the appendix.

In order to introduce confounding, we only observe a censored version of the state; for each patient, with $25\%$ probability we do not observe whether or not that patient is diabetic (\emph{i.e.} in all observations for that patient the ``diabetic'' indicator is set to ``False'' regardless of whether the patient is diabetic or not). That is, the true state contains both an indicator of whether the patient is diabetic or not and whether their diabetes status is censored, but for the observed state we instead only observed a possibly censored diabetes indicator. Since all other components of the state are discrete, this means that both state and observation spaces are discrete (\emph{i.e.} tabular), with a total state space size of $|\Scal| = 2880$, and observation space size of $|\Ocal| = 1440$.

We experimented on this scenario over a time horizon of $H=3$ and a discount factor of $\gamma=1$. We first constructed our behavioral policy $\pi_b$ by computing the optimal policy in the true POMDP $\pi^\star$, and defining $\pi_b$ by introducing $\epsilon$-greedy sampling to $\pi^\star$ with $\epsilon=0.1$; that is, we defined $\pi_b = 0.9 \pi^\star + 0.1 \pi_{\textup{unif}}$, where $\pi_{\textup{unif}}$ is a policy that takes all 8 actions with equal probability. Then, we sampled 10,000 observational trajectories using $\pi_b$, and defined $\pi_e$ to be the predicted optimal policy fit on these trajectories using dynamic programming on a simple count-based tabular MDP model, treating the observations $O_t$ as the true states $S_t$. Note that since the observations $O_t$ are confounded, we expect that $\pi_e$ should \emph{not} necessarily be an estimate of the actual optimal policy $\pi^\star$.

Next, given the fixed policies $\pi_b$ and $\pi_e$ coming from the first stage of the experiment, we repeated the following procedure $50$ times: (1) we sampled 10,000 observational trajectories using $\pi_b$; and (2) we estimated $v_1(\pi_e)$ using those trajectories as input for all methods. We performed policy evaluation with our method, as well as the \textsc{MeanR} and \textsc{MDP} benchmarks, as in the previous experiment. In the case of our method, we experimented with a large range of hyperparameter values, as detailed in the appendix. In addition, we used the proxies $Z_t=(G_{t-1},X_t)$ and $W_t=(G_t,X_t)$, where $O_t = (G_t,X_t)$ is a partition of the observation into information about diabetes ($G_t$) and non-diabetes information ($X_t$); see appendix for more details.

Finally, since we had observed in our prior experiments that our method could be sensitive to hyperparameter values, and also since we lack ground truth so cannot set these ``fairly'' using \emph{e.g.} cross-validation, we experimented with the following heuristic procedure automatic hyperparameter selection: (1) we first estimate the policy value using all 81 different possible hyperparameter values; (2) we throw away all estimates that take values outside of the range of observed reward values; and (3) we take the median of the remaining estimates. This heuristic is based on the observation from our prior experiments that, as long as hyperparameter values are within reasonable ranges, our method typically gives estimates that are either fairly accurate, or wildly out-of-bound. We estimated policy value using this heuristic separately for each of the $50$ experimental replications.

\subsubsection{Results}

\begin{table}[t]
    \centering
    \begin{tabular}{ccccc}
        \hline
         Method & $\hat v_1(\pi_e)$ & Bias & RMSE & Improvement Acc.  \\
         \hline
         \textsc{Ours} (best hyper.) & $-2.370 \pm 0.597$ & $-0.096$ & $0.599$ & 82\% \\
         \textsc{Ours} (auto hyper.) & $-2.459 \pm 0.182$ & $-0.184$ & $0.258$ & 100\% \\
         \textsc{MDP} & $-1.261 \pm 0.054$ & $1.014$ & $1.015$ & 0\% \\
         \textsc{MeanR} & $-1.799 \pm 0.025$ & $0.476$ & $0.477$ & --- \\
        \hline
    \end{tabular}
    \caption{Results of our Proximal RL experiments on the sepsis management environment. For reach method, we list the average policy value prediction (with one standard deviation error), along with the empirical bias and root mean squared error. In addition, for each method other than \textsc{MeanR}, we list the method's accuracy of predicting whether $v_1(\pi_e) > v_1(\pi_b)$ or not. For reference, the true policy values were $v_1(\pi_e) = -2.275$ and $v_1(\pi_b) = -1.799$.}
    \label{tab:sepsis-results}
\end{table}

We present the main results of this second experiment in \cref{tab:sepsis-results}.
There we present results for our method with the single best set of hyperparameters out of all tested (in terms of mean squared error across the 50 replications), as well using the automatic hyperparameter selection heuristic described above.
We can first observe that using the single best hyperparameter setup gives policy value predictions that are approximately unbiased, but with very high variance. Qualitatively, this variance seems to be partially explained by unstable predictions in a minority of cases.
On the other hand, our automatic hyperparameter heuristic gives estimates results in slightly higher bias, but much lower variance, and therefore much lower mean squared error.
This strong performance of our heuristic versus choosing the best single set of hyperparameters is extremely encouraging, since unlike picking a ``best'' hyperparameter combination, the heuristic is actually feasible in practice, as it does not require any ground truth information for hyperparameter selection.
Finally, as in the prior experiments, the benchmark methods, which either do not take into account confounding (\textsc{MDP}), or are completely non-causal (\textsc{MeanR}), both give extremely biased estimates with low variance.

Next, we note that in practice we are often more concerned about predicting whether $\pi_e$ is an improvement on $\pi_b$ or not, rather than the exact policy value of $\pi_e$.
Accurately answering this question is important in many applications, where the baseline policy $\pi_b$ reflects current best practices or business as usual, and $\pi_e$ represents a proposed new policy. For example, here we could think of $\pi_b$ representing how physicians currently manage sepsis, and $\pi_e$ as a proposed automated algorithm for sepsis management. We have $v_1(\pi_e) \approx -2.275$ and $v_1(\pi_b) \approx -1.799$, so we would like any method of policy evaluation to be able to correctly predict that the new proposed algorithm ($\pi_e$) is worse than standard physician care ($\pi_b$).
Specifically, we evaluate each method by what percentage of the time the policy value estimate is smaller than the observational mean reward (\textsc{MeanR}), as the latter is an unbiased estimate of $v_1(\pi_b)$.
We list these results in the final column in \cref{tab:sepsis-results}. We note that our method with the best hyperparameters usually correctly predicts that $\pi_e$ is worse than $\pi_b$, and with our automatic hyperparameter selection heuristic this prediction is \emph{always} correct. On the other hand, the \textsc{MDP} benchmark, which fails to take into account confounding from the censored diabetes measurements, always incorrectly predicts that $\pi_e$ is an improvement on $\pi_b$.

\section{Conclusion}

In this paper, we discussed the problem of OPE in an unknown POMDP as a model for the problem of offline RL with general unobserved confounding. 
First, we analyzed the recently proposed approach for identifying the policy value for tabular POMDPs \citep{tennenholtz20off}. We showed that while it could be placed within a more general framework and extended to continuous settings, it suffers from some theoretical limitations due to the unusual form of the identification formulation, which brings its usefulness for constructing estimators with good theoretical properties into question.
Motivated by this, we proposed a new framework for identifying the policy value, by sequentially reducing the problem to a series of proximal causal inference problems.
Furthermore, we extended this identification framework to a framework of estimators based on double machine learning and cross-fitting \citep{chernozhukov2016double}, and showed that under appropriate conditions such estimators are asymptotically normal and semiparametrically efficient.
Finally, we constructed a concrete algorithm for implementing such an estimator,
and provided an empirical proof of concept of our theory by applying algorithm in a toy synthetic setting with confounding due to noisy measurements, as well as a complex spepsis management setting with confounding due to missing measurements of diabetes.

Perhaps the most significant scope for future work on this topic is in the development of more practical algorithms. 
Indeed, although our experiments were only intended as a proof of concept of our methods and theory, they also show that our actual proposed estimators can often have high variance even in a simple toy POMDP with a moderate number (\eg, 1000) of trajectories. 
There may be ways to improve on this; for example it may be beneficial to solve the conditional moment problems defining the $q^{(t)}$ and $h^{(t)}$ functions simultaneously rather than sequentially as we proposed, which may result in cascading errors.
Another important topic for future work would be to explore hyperparameter optimization strategies, such as the heuristic method we proposed for our sepsis experiments; although we found this heuristic worked well empirically, it may introduce other challenges such as dealing with post-selection inference.

Another area where there is significant scope for future work is on the topic of semiparametric efficiency.
Extending our model to allow for multiple nuisances, in a way where the parameter of interest is still well-defined, is an important open challenge. Additional issues are discussed in \cref{apx:tangent-space}.

Finally, in terms of future work, there is the problem of how to actually apply our theory as well as policy value estimators in real-world sequential decision making problems involving unmeasured confounding. Although our work is largely theoretical, we hope that it will be impactful in motivating progress toward solving such real-world challenges in practice.

\bibliography{ref}
\newpage
\appendix

\section{Identification by Time-Independent Sampling}
\label{apx:tennenholtz}

In this appendix, we present a general identification result given by \cref{thm:identification-ind}. Then, we present a specialization of this result to the discrete setting in \cref{lem:ind-tabular}. We do not provide a separate proof of \Cref{thm:tennenholtz}, since it follows immediately from \cref{lem:ind-tabular}. We also note that in this section we will use the notation $Z_t = o(D_t)$, $W_t = o'(D_t)$, and $X_t = o''(D_t)$, which is not to be confused with the $Z_t$ and $W_t$ notation used in our PCI identification theory.

First, before we present these results, we establish the following completeness assumption which they depend on, and is the missing technical assumption referenced by \cref{thm:tennenholtz}.

\begin{assumption}[Completeness]
\label{asm:completeness-ind}
For each $t \in \{1,\ldots,H\}$ and $a \in \mathcal A$, if $\e_{\pb}[g(S_t) \mid O_t, A_t=a] = 0$ almost surely for some function $g$, then $g(S_t) = 0$ almost surely.
\end{assumption}

This assumption is fundamental to this identification approach, and essentially requires that $O_t$ captures all degrees of variation in $S_t$. In the case that states and observations are finite, it is necessary that $O_t$ have at least as many categories as $S_t$ for this condition to hold. Given this, we are ready to present our first identification result.

\begin{theorem}
\label{thm:identification-ind}
    Let \cref{asm:completeness-ind} hold, and suppose that for each $t \in \{1,\ldots,H\}$ there exists a function $\rho^{(t)}: \mathcal S \times \mathcal A \times \mathcal S \mapsto \mathbb R$, such that for every measure $f$ on $W_t$ that is absolutely continuous with respect to $\pb$ and every $a \in \aset$, we have almost surely
    \begin{equation}
    \label{eq:q-ind}
        \e \left[ \int \rho^{(t)}(Z_t,A_t,x) df(x) \ \middle| \ W_t, A_t=a \right] = P(A_t=a \mid W_t)^{-1} \left(\frac{d f}{d \pb}\right)(W_t) \,,
    \end{equation}
    where $df / d\pb$ denotes the Radon-Nikodym derivative of $f$ with respect to $\pb$. Then, for each $s \in \{1, \ldots, H\}$ we have
    \begin{equation*}
        \epe[R_s] = \e_{\pind}\left[ R_s \prod_{t=1}^s \indicator{A_t=E_t} \rho^{(t)}(Z_t,A_t,X_{t-1}) \right]
    \end{equation*}
\end{theorem}

We note that this result identifies $v_\gamma(\pi_e)$ for any given $\gamma$, since by construction $\pind$ is identified with respect to $\pb$, and this allows us to express $v_\gamma(\pi_e)$ as a function of $\pind$. Note that implicit in the assumptions is that $P(A_t=a \mid W_t)>0$.

We call this result a \emph{time-independent sampling} result, since it is written as an expectation with respect to $\pind$, where data at each time point is sampled independently. We note that the moment equations given by \cref{eq:q-ind} in general are very complicated, and it is not immediately clear under what conditions this equation is even solvable. In the tabular setting, we present the following lemma which provides an analytic solution to \cref{eq:q-ind} and makes clear the connection to \citet{tennenholtz20off}.

\begin{lemma}
\label{lem:ind-tabular}
Suppose that $O_t$ is discrete with $k$ categories for every $t$, and without loss of generality let the support of $O_t$ be denoted by $\{1,\ldots,k\}$. In addition, for each $t \in \{1,\ldots,s\}$ and $a \in \mathcal A$, let $Q^{(t,a)}$ denote the $k \times k$ matrix defined according to
\begin{equation*}
    Q^{(t,a)}_{x,y} = P_{\pb}(O_t=x \mid A_t=a, O_{t-1}=y) \,.
\end{equation*}
Then, assuming $Q^{(t,a)}$ is invertible for each $t$ and $a$, \cref{eq:q-ind} is solved by
\begin{equation*}
    \rho^{(t)}(z,a,x) = \frac{((Q^{(t,a)})^{-1})_{z,x}}{P(O_{t-1}=z, A_t=a)} \,.
\end{equation*}
Furthermore, plugging this solution into the identification result of \cref{thm:identification-ind} is identical to Theorem 1 of \citet{tennenholtz20off}.
\end{lemma}

We also note that, in the case that the matrices $Q^{(t,a})$ defined above are invertible, it easily follows that \cref{asm:completeness-ind} holds, as long as $S_t$ has no more than $k$ categories.

\begin{proof}[Proof of \cref{thm:identification-ind}]

We will prove this result for arbitrary fixed $s \in [H]$. Define
\begin{align*}
    Y_s &= R_s \\
    Y_t &= \phi^{(t+1)}(Z_{t+1}, A_{t+1}, W_{t+1}, E_{t+1}, X_t, Y_{t+1}) \qquad \forall t \in [s-1] \,,
\end{align*}
where
\begin{equation*}
    \phi^{(t)}(z,a,w,e,x,y) = \rho^{(t)}(z,a,x) \indicator{a=e} y\,.
\end{equation*}

Now, by these definitions we need to prove that
\begin{equation*}
    \epe[R_s] = \e_{\pind}[Y_0]\,.
\end{equation*}
where $Y_0 = \phi^{(1)}(Z_1,A_1,W_1,E_1,X_0,Y_1)$.

We will proceed via a recursive argument. In order to set up our key recursion, we first define some additional notation. First, let $\ps{t}$ denote the intervention distribution introduced in \cref{sec:ident-pci}, and let $\psind{t}$ denote the measure on $D_{1:H}$ defined by a mixture between $\ps{t+1}$ and $\pind$, where
\begin{enumerate}
    \item $\{W_{1:t-1}\}$, $\{X_{1:t-1}\}$ $\{A_{1:t-1}\}$, and $\{R_{1:t-1}\}$ are jointly sampled from $\ps{t}$
    \item $\{Z_1,\ldots,Z_{H}\}$, $\{W_{t},\ldots,W_{H}\}$, $\{X_{t},\ldots,X_{H}\}$ $\{A_{t},\ldots,A_{H}\}$, and $\{R_{t},\ldots,R_{H}\}$ are jointly sampled from $\pind$.
\end{enumerate}

Given this setup, the inductive relation we would like to prove is
\begin{equation}
\label{eq:ind-recursion}
    \e_{\psind{t}}[\phi^{(t)}(Z_t,A_t,W_t,E_t,X_{t-1},Y_t)] = \e_{\psind{t+1}}[Y_t] \qquad \forall t \in [s]
\end{equation}

We note that if \cref{eq:ind-recursion} holds, then via chaining this relation and the recursive definitions of $Y_t$, we would instantly have our result, since $\e_{\psind{1}}[\phi^{(1)}(Z_1,A_1,W_1,E_1,W_0,Y_1)] = \e_{\psind{1}}[Y_0] = \e_{\pind}[Y_0]$, and $\e_{\psind{s+1}}[R_s] = \epe[R_s]$. Therefore, it only remains to prove that \cref{eq:ind-recursion} holds.

Next, by the assumption on $\phi^{(t)}$ in the theorem statement, we have
\begin{align*}
    &\e_{\psind{t}}\left[\rho^{(t)}(Z_t,A_t,X_{t-1}) \left(\frac{d \pb}{d \psind{t+1}}\right)(W_t) \ \middle| \ W_t, A_t=a\right] \\
    &= \left(\frac{d \pb}{d \psind{t+1}}\right)(W_t) \e_{\psind{t}}\left[ \int_x f_{t-1}(x) \rho^{(t)}(Z_t,A_t,x) \ \middle| \ W_t, A_t=a\right] \\
    &= \left(\frac{d \pb}{d \psind{t+1}}\right)(W_t) \left(\frac{d \psind{t+1}}{d \pb}\right)(W_t) P(A_t=a \mid W_t)^{-1} \\
    &= P(A_t=a \mid W_t)^{-1} \,,
\end{align*}
where in this derivation $f_{t-1}$ denotes the density of $X_{t-1}$ under $\psind{t}$, which we note is the same as the density of $W_t$ under $\psind{t+1}$. Given this, applying the independence assumptions of our POMDP framework we have
\begin{align*}
    &\e_{\psind{t}}[P(A_t=a \mid S_t)^{-1} \mid W_t, A_t=a] \\
    &= \e_{\psind{t}}[P(A_t=a \mid S_t, W_t)^{-1} \mid W_t, A_t=a] \\
    &= \int_s \frac{P(S_t = s \mid W_t, A_t=a)}{P(A_t=a \mid W_t, S_t=s)} ds \\
    &= \int_s \frac{P(A_t=a \mid W_t, S_t=s) P(S_t=s \mid W_t)}{P(A_t=a \mid W_t, S_t=s) P(A_t=a \mid W_t)} ds \\
    &= P(A_t=a \mid W_t)^{-1} \\
    &= \e_{\psind{t}}\left[\rho^{(t)}(Z_t,A_t,X_{t-1}) \left(\frac{d \pb}{d \psind{t+1}}\right)(W_t) \ \middle| \ W_t, A_t=a\right] \\
    &= \e_{\psind{t}}\left[ \e_{\psind{t}}\left[\rho^{(t)}(Z_t,A_t,X_{t-1}) \left(\frac{d \pb}{d \psind{t+1}}\right)(W_t) \ \middle| \ S_t, A_t=a \right] W_t, A_t=a\right] \,.
\end{align*}

Given this, it then follows from \cref{asm:completeness-ind} that
\begin{equation*}
    \e_{\psind{t}}\left[\rho^{(t)}(Z_t,A_t,X_{t-1}) \left(\frac{d \pb}{d \psind{t+1}}\right)(W_t) \ \middle| \ S_t, A_t=a \right] = P(A_t=a \mid S_t)^{-1} \,,
\end{equation*}
which holds almost surely for each $a \in \aset$, and therefore also holds replacing $a$ with $A_t$.

Finally, applying this previous equation, we have
\begin{align*}
    &\e_{\psind{t}}[\phi^{(t)}(Z_t,A_t,W_t,E_t,X_{t-1},Y_t] \\
    &= \e_{\psind{t}}[\rho^{(t)}(Z_t, A_t, X_{t-1}) \indicator{A_t = E_t} Y_t] \\
    &= \e_{\psind{t}} \Bigg[ \e_{\psind{t}}\left[\rho^{(t)}(Z_t,A_t,X_{t-1}) \left(\frac{d \pb}{d \psind{t+1}}\right)(W_t) \ \middle| \ S_t, A_t, W_t, E_t, Y_t \right] \\
    &\qquad\qquad\qquad \cdot \left(\frac{d \psind{t+1}}{d \pb}\right)(W_t) \indicator{A_t = E_t} Y_t \Bigg] \\
    &= \e_{\psind{t}}\Bigg[\e_{\psind{t}}\left[\rho^{(t)}(Z_t,A_t,X_{t-1}) \left(\frac{d \pb}{d \psind{t+1}}\right)(W_t) \ \middle| \ S_t, A_t, \right] \\
    &\qquad\qquad\qquad \cdot \left(\frac{d \psind{t+1}}{d \pb}\right)(W_t) \indicator{A_t = E_t} Y_t \Bigg] \\
    &= \e_{\psind{t}}\left[ P(A_t \mid S_t)^{-1}  \left(\frac{d \psind{t+1}}{d \pb}\right)(W_t) \indicator{A_t = E_t} Y_t \right] \\
    &= \e_{\psind{t}}\left[ \sum_a \ \frac{P(A_t \mid S_t, W_t, E_t, Y_t)}{P(A_t \mid S_t)}  \left(\frac{d \psind{t+1}}{d \pb}\right)(W_t) \indicator{E_t = a} Y_t(E_t) \right] \\
    &= \e_{\psind{t}}\left[ \left(\frac{d \psind{t+1}}{d \pb}\right)(W_t) Y_t(E_t) \right] \\
    &= \e_{\psind{t+1}}[Y_t] \,,
\end{align*}
where the third and sixth equalities follow from the independence assumptions of the POMDP given $S_t$. In this derivation we use the potential outcome notation $Y_t(a)$ to denote the value $Y_t$ would have taken if we intervened on the $t$'th action with value $a$ (and the subsequent values of $X_t$ and $R_t$ are possibly changed accordingly; note that this intevention does not change the values of $Z_t$ or $W_t$ since these represent observations at time $t-1$ and $t$ respectively.) The final equality follows because replacing $Y_t$ with $Y_t(E_t)$ effectively updates the mixture distribution $\psind{t}$ so that $A_t$, $X_t$, and $R_t$ are included in the set variables sampled according to $\ps{t+1}$, rather than in the set of those sampled according to $\pind$. Furthermore, integrating over the Radon-Nikodym derivative $(d \psind{t+1} / d \pb)(W_t)$ effectively further updates the mixture distribution so that $W_t$ is also included in the set sampled according to $\ps{t+1}$, since the distribution of $W_t$ under $\pb$ is the same as the distribution of $W_t$ under $\psind{t}$. That is, these two terms effectively replace integration under $\psind{t}$ with integration under $\psind{t+1}$. This establishes \cref{eq:ind-recursion}, and therefore as discussed above the theorem follows by recursion.

\end{proof}

\begin{proof}[Proof of \cref{lem:ind-tabular}]

First we establish the required property of this definition of $\rho^{(t)}$. Since observations are tabular, the required property is equivalent to
\begin{align*}
    \e\left[\sum_{x \in \mathcal O} f(x) \rho^{(t)}(Z_t, A_t, x) \ \middle| \ W_t, A_t=a \right] &= \frac{f(W_t)}{P(W_t)} P(A_t \mid W_t)^{-1} \\
    &= \frac{f(W_t)}{P(A_t=a, O_t=W_t)} \,,
\end{align*}
almost surely for every discrete probability distribution $f$ over the observation space. Now, recalling that $Q^{(t,a)}_{x,y} = P(O_t=x \mid A_t=a, O_{t-1}=y)$, plugging the definition of $\rho^{(t)}$ into the LHS above, we have
\begin{align*}
    &\e\left[\sum_{x \in \mathcal O} f(x) \rho^{(t)}(Z_t, A_t, x) \ \middle| \ W_t, A_t=a \right] \\
    &=  \e\left[\sum_{x \in \mathcal O} f(x) P(O_{t-1}=Z_t, A_t=a)^{-1} (Q^{(t,a)})^{-1}_{Z_t,x} \ \middle| \ W_t, A_t=a \right] \\
    &=  \sum_{x,z \in \mathcal O} f(x) P(O_{t-1}=z, A_t=a)^{-1} P(O_{t-1}=z \mid O_t=W_t, A_t=a) (Q^{(t,a)})^{-1}_{z,x} \\
    &=  \sum_{x,z \in \mathcal O} \frac{f(x) P(O_t=W_t \mid O_{t-1}=z, A_t=a) P(O_{t-1}=z \mid A_t=a)}{P(O_{t-1}=z, A_t=a) P(O_t=W_t \mid A_t=a)} (Q^{(t,a)})^{-1}_{z,x} \\
    &=  \sum_{x,z \in \mathcal O} \frac{f(x) P(O_{t-1}=z \mid A_t=a)}{P(O_{t-1}=z, A_t=a) P(O_t=W_t \mid A_t=a)} Q^{(t,a)}_{W_t,z} (Q^{(t,a)})^{-1}_{z,x} \\
    &=  \sum_{x \in \mathcal O} \frac{f(x)}{P(A_t=a) P(O_t=W_t \mid A_t=a)} \sum_{z \in \mathcal O} Q^{(t,a)}_{W_t,z} (Q^{(t,a)})^{-1}_{z,x} \\
    &=  \sum_{x \in \mathcal O} \frac{f(x)}{P(O_t=W_t, A_t=a)} \indicator{W_t=x}  \\
    &= \frac{f(W_t)}{P(A_t=a, O_t=W_t)}\,,
\end{align*}
which establishes the required property of $\rho^{(t)}$.

Now, for the second part of the theorem, we first note that in terms of our notation and under our (w.l.o.g.) assumption that the target policy is deterministic, \citet[Theorem 1]{tennenholtz20off} is equivalent to
\begin{align*}
    \epe[R_s] = \sum_{o_{1:s} \in \oset^s, a_{1:s} \in \aset^s} &\left( \prod_{t=1}^s \indicator{a_t = E_t(o_{1:t}, a_{1:t-1})} \right) \\
    &\cdot \sum_{z \in \oset} \eb[R_s \mid O_s=o_s,A_s=a_s,O_{s-1}=z] \\
    &\qquad \cdot P(O_s=o_s \mid A_s=a,O_{s-1}=z) \Omega(o_{1:s}, a_{1:s})_z \,,
\end{align*}
where $E_t(o_{1:t},a_{1:t-1})$ denotes the action taken by $\pi_e$ given $O_{1:t} = o_{1:t}$, and $A_{1:t-1} = a_{1:t-1}$, andwe define
\begin{align*}
    \Omega(o_{1:s},a_{1:s}) &= \prod_{t=1}^s \Xi_{s-t+1}(o_{1:s-t+1}, a_{1:s-t+1}) \\
    \Xi_t(o_{1:t}, a_{1:t})_{z,z'} &= \sum_{x \in \oset} (Q^{(t,a_t)})^{-1}_{z,x} P(O_t=x, O_{t-1}=o_{t-1} \mid A_{t-1}=a_{t-1}, O_{t-2}=z') \\ 
    \Xi_1(o_{1:t}, a_{1:t})_{z} &= \sum_{x \in \oset} (Q^{(1,a_1)})^{-1}_{z,x} P(O_1=x) \,.
\end{align*}

We note that the term we refer to as $\Omega$ was called the same in \citet{tennenholtz20off}, and the terms we refer to as $\Xi$ were called $W$, and we explicitly write out the matrix multiplication in the definitions of the $\Xi$ terms. Next, plugging the definition of $\Omega$ into the above equation for $\epe[R_s]$, and explicitly writing out the sums implied by the multiplication of the $\Xi_t$ terms, and re-arranging terms, we obtain

\begin{align*}
    \epe[R_s] = \sum_{\substack{o_{1:s} \in \oset^s, a_{1:s} \in \aset^s \\ z_{1:s} \in \oset^s, x_{0:s-1} \in \oset^s}} &\left( \prod_{t=1}^s \indicator{a_t = E_t(o_{1:t}, a_{1:t-1})} \right) \\
    &\cdot \eb[R_s \mid O_s=o_s,A_s=a_s,O_{s-1}=z_s ] \\
    &\cdot \left( \prod_{t=1}^s (Q^{(t,a_t)})^{-1}_{z_t,x_{t-1}} P(A_t=a_t, O_{t-1}=z_t)^{-1} \right) \\
    &\cdot \left( \prod_{t=1}^{s-1} P(O_t=o_t, A_t=a_t, O_{t-1}=z_t, O_{t+1}=x_t) \right) \\
    &\cdot P(O_s=o_s, A_s=a_s, O_{s-1}=z_s) P(O_0=x_0) \,.
\end{align*}

Now, we note that $(Q^{(t,a_t)})^{-1}_{z_t,x_{t-1}} P(A_t=a_t, O_{t-1}=z_t)^{-1} = \rho^{(t)}(z_t,a_t,x_{t-1})$, and that summing over the product of terms $\prod_{t=1}^{s-1} P(O_t=o_t, A_t=a_t, O_{t-1}=Z_t, O_{t+1}=x_t)$ and $P(O_s=o_s, A_s=a_s, O_{s-1}=Z_s)$ and $P(O_0=x_0)$ is equivalent to integrating over $\pind$, where $z_t$, $a_t$, $x_t$, and $o_t$ correspond to $Z_t$, $A_t$, $X_t$, and $W_t$ respectively. Re-writing the previous equation as an expectation and simplifying based on this gives us
\begin{align*}
    \epe[R_s] &= \e_{\pind}\left[\eb[R_s \mid W_s, A_s, Z_s] \prod_{t=1}^s \indicator{A_t=E_t} \rho^{(t)}(Z_t,A_t,X_{t-1})\right] \\
    &= \e_{\pind}\left[ \e_{\pind}\left[ R_s \prod_{t=1}^s \indicator{A_t=E_t} \rho^{(t)}(Z_t,A_t,X_{t-1}) \ \middle| \ W_s,A_s,Z_s \right] \right] \\
    &= \e_{\pind}\left[ R_s \prod_{t=1}^s \indicator{A_t=E_t} \rho^{(t)}(Z_t,A_t,X_{t-1}) \right] \,,
\end{align*}
where the second equation follows since the distribution of $R_s$ given $W_s$, $A_s$, and $Z_s$ is the same under $\pb$ and $\pind$, and because $R_s$ is independent of $\prod_{t=1}^s \indicator{A_t=E_t} \rho^{(t)}(Z_t,A_t,X_{t-1})$ given $(W_s,A_s,Z_s)$ under $\pind$. We note that the final equation is our identification result from \cref{thm:identification-ind}, and so we conclude.

\end{proof}

\section{Identification by Proximal Causal Inference}
\label{apx:identification-pci}

In this section we will present a slightly more general theorem than \cref{thm:identification-pci}, which is the following.

\begin{theorem}
\label{thm:identification-pci-general}
    Let \cref{asm:pci-reduction,asm:bridge} hold. For each $s \in \{1,\ldots,H\}$ recursively define $Y^{(s)}_s = R_s$, and $Y^{(s)}_{t-1} = \phi^{(t,s)}(Z_{t},W_{t},A_{t},E_{t},Y_{t})$ for each $t \leq s$, where the function $\phi^{(t,s)}$ is allowed to take one of the following three forms:
    \begin{align*}
        \phi^{(t,s)}_{\text{Reg}}(Z_t,W_t,A_t,E_t,Y_t^{(s)}) &= \sum_{a \in \mathcal A} h^{(t,s)}(W_t, a) \\
        \phi^{(t,s)}_{\text{IS}}(Z_t,W_t,A_t,E_t,Y_t^{(s)}) &= q^{(t)}(Z_t,A_t) \indicator{A_t=E_t} Y_t^{(s)} \\
        \phi^{(t,s)}_{\text{DR}}(Z_t,W_t,A_t,E_t,Y_t^{(s)}) &= \sum_{a \in \mathcal A} h^{(t,s)}(W_t,a) \\
        &\qquad + q^{(t)}(Z_t,A_t) \left( \indicator{A_t=E_t} Y_t^{(s)} - h^{(t,s)}(W_t,A_t) \right) \,,
    \end{align*}
    where $h^{(t,s)}$ and $q^{(t)}$ are solutions to, respectively,
    \begin{align*}
        \es{t}[q^{(t)}(Z_t,A_t)  \mid W_t,A_t=a] &= \pd{t}(A_t=a \mid W_t)^{-1} \quad \text{a.s.} \quad \forall a \in \aset\,, \\
        \es{t}[h^{(t,s)}(W_t,A_t) \mid Z_t,A_t=a] &= \es{t}[\indicator{A_t=E_t} Y_t^{(s)} \mid Z_t,A_t=a] \quad \text{a.s.} \quad \forall a \in \aset \,,
    \end{align*}
    which we show must exist.

    Then, we have $\epe[R_s] = \eb[Y_0^{(s)}]$ for each $s \in \{1,\ldots,H\}$. 
\end{theorem}

Furthermore, the following corollary makes the connection between \cref{thm:identification-pci} and \cref{thm:identification-pci-general} clear.

\begin{corollary}
\label{cor:psidr}
Let \cref{asm:pci-reduction,asm:bridge} hold, and let $Y_t$, $h^{(t)}$, and $\eta_t$ be defined as in \cref{thm:identification-pci}.
Then, we have $v_{\gamma}(\pi_e) = \eb[\psiis(\tau_H)] = \eb[\psireg(\tau_H)] = \eb[\psidr(\tau_H)]$, where
\begin{align*}
    \psiis(\tau_H) &= \sum_{t=1}^H \gamma^{t-1} \eta_{t+1} R_t  \\
    \psireg(\tau_H) &= \sum_{a \in \aset} h^{(1)}(W_1,a) \\
    \psidr(\tau_H) &= \sum_{t=1}^H \gamma^{t-1} \left( \eta_{t+1} R_t + \eta_t \sum_{a \in \mathcal A} h^{(t)}(W_t, a) - \eta_t q^{(t)}(Z_t,A_t) h^{(t)}(W_t,A_t) \right) \,.
\end{align*}
\end{corollary}

This corollary follows directly from \cref{thm:identification-pci-general}, noting that for any collection of variables $Y^{(s)}_0$ satisfying the conditions of \cref{thm:identification-pci-general} we have $v_\gamma(\pi_e) = \eb[\sum_{s=1}^H \gamma^{s-1} Y^{(s)}_0]$. For $\psiis$, $\psireg$, and $\psidr$ the specific result arises by using $\phi^{(t,s)}_{\text{IS}}$, $\phi^{(t,s)}_{\text{Reg}}$, or $\phi^{(t,s)}_{\text{DR}}$ respectively for each $(t,s)$, and we also use the fact that for every $t \geq 1$ we have $Y_t = \sum_{s=t}^H \gamma^{s-t} Y_t^{(s)}$, and therefore $h^{(t)} = \sum_{s=t}^H \gamma^{s-t} h^{(t,s)}$.

We note that $\psiis$ and $\psireg$ have very similar structures to importance sampling and direct method estimators for the MDP setting, where the $h^{(t)}$ terms are similar to the quality function terms, and the $\eta_t$ and $\nu_t$ terms are similar to the importance sampling terms. Also, as already discussed in \cref{sec:ident-pci}, $\psidr$ has a very similar structure to Double Reinforcement Learning (DRL) estimators for the MDP setting \citep{kallus2020double}.

Before we present the proof of \cref{thm:identification-pci-general}, we establish some additional notation and some helper lemmas. Using similar notation to \citet{kallus2021causal}, for any $t \in [H]$ and $\phi \in L_{2,\ps{t}}(R_t,D_{t+1:H})$ we define the sets
\begin{align*}
    \mathbb Q^{(t)} &= \{q \in L_{2,\ps{t}}(Z_t,A_t) : \es{t}[q(Z_t,A_t) -  \pd{t}(A_t \mid U_t)^{-1} \mid U_t, A_t=a] = 0 \\
    &\qquad \text{a.s.} \quad \forall a \in \aset \} \\
    \mathbb H^{(t,\phi)} &= \{h \in L_{2,\ps{t}}(W_t,A_t) : \es{t}[h(W_t,A_t) - \indicator{A_t=E_t} Y_t \mid U_t, a_t=a] = 0 \\
    &\qquad \text{a.s.} \quad \forall a \in \aset \} \\
    \mathbb Q^{(t)}_{\text{obs}} &= \{q \in L_{2,\ps{t}}(Z_t,A_t) : \es{t}[q(Z_t,A_t) - \pd{t}(A_t \mid W_t)^{-1} \mid W_t, A_t=a] = 0 \\
    &\qquad \text{a.s.} \quad \forall a \in \aset \} \\
    \mathbb H^{(t,\phi)}_{\text{obs}} &= \{h \in L_{2,\ps{t}}(W_t,A_t) : \es{t}[h(W_t,A_t) - \indicator{A_t=E_t}Y_t \mid W_t, A_t=a] = 0 \\
    &\qquad \text{a.s.} \quad \forall a \in \aset \} \,,
\end{align*}
where $Y_t = \phi(R_t,E_{t+1:H})$.

First, we will prove an important claim from \cref{sec:ident-pci}, which is that \cref{asm:bridge} implies that \cref{eq:q,eq:h} both have solutions. This claim is formalized by the following lemma.

\begin{lemma}
\label{lem:bridge-subset}
    Under \cref{asm:pci-reduction} and for each $t \in [H]$ and $\phi \in L_{2,\ps{t}}(R_t,D_{t+1:H})$ we have $\mathbb Q^{(t)} \subseteq \mathbb Q^{(t)}_{\text{obs}}$ and $\mathbb H^{(t,\phi)} \subseteq \mathbb H^{(t,\phi)}_{\text{obs}}$.
\end{lemma}

\begin{proof}[Proof of \cref{lem:bridge-subset}]
    First, suppose that $q^{(t)} \in \mathbb Q^{(t)}$. Then we have
    \begin{align*}
        \es{t}[q^{(t)}(Z_t,A_t) \mid W_t, A_t=a] &= \es{t}[ \es{t}[q^{(t)}(Z_t,A_t) \mid U_t, W_t, A_t=a] \mid W_t, A_t=a] \\
        &= \es{t}[ \es{t}[q^{(t)}(Z_t,A_t) \mid U_t, A_t=a] \mid W_t, A_t=a] \\
        &= \es{t}[ \pd{t}(A_t=a \mid U_t)^{-1} \mid W_t, A_t=a] \\
        &= \int \frac{\pd{t}(U_t=u \mid W_t, A_t=a)}{\pd{t}(A_t=a \mid U_t=u)} d\mu(u) \\
        &= \int \frac{\pd{t}(A_t=a \mid W_t, U_t=u) \pd{t}(U_t=u \mid W_t)}{\pd{t}(A_t=a \mid U_t=u)\pd{t}(A_t=a \mid W_t)} d\mu(u) \\
        &= \pd{t}(A_t=a \mid W_t)^{-1} \int \pd{t}(U_t=u \mid W_t) d\mu(u) \\
        &= \pd{t}(A_t=a \mid W_t)^{-1} \,,
    \end{align*}
    where in the second and sixth equalities we apply the independence assumptions from \cref{asm:pci-reduction}, in the third equality we apply the fact that $q^{(t)} \in \mathbb Q^{(t)}$, and the fifth equality follows from Bayes' rule. Therefore, $q^{(t)} \in \mathbb Q^{(t)}_{\text{obs}}$.
    
    Second, suppose that $h^{(t)} \in \mathbb H^{(t,\phi)}$. Then we have 
    \begin{align*}
        \es{t}[h^{(t)}(W_t,A_t) \mid Z_t, A_t=a] &= \es{t}[ \es{t}[h^{(t)}(W_t,A_t) \mid U_t, Z_t, A_t=a] \mid Z_t, A_t=a] \\
        &= \es{t}[ \es{t}[h^{(t)}(W_t,A_t) \mid U_t, A_t=a] \mid Z_t, A_t=a] \\
        &= \es{t}[ \es{t}[\phi(R_t,D_{t+1:H}) \mid U_t, A_t=a] \mid Z_t, A_t=a] \\
        &= \es{t}[ \es{t}[\phi(R_t,D_{t+1:H}) \mid U_t, Z_t, A_t=a] \mid Z_t, A_t=a] \\
        &= \es{t}[\phi(R_t,D_{t+1:H}) \mid Z_t, A_t=a] \,,
    \end{align*}
    where in the second and fourth equalities we apply the independence assumptions from \cref{asm:pci-reduction}, and in the third equality we apply the fact that $h^{(t)} \in \mathbb H^{(t)}$. Therefore, $h^{(t)} \in \mathbb H^{(t,\phi)}_{\text{obs}}$.
    
\end{proof}

Next, we establish the following pair of lemmas, which allow us to establish that $\phi^{(t,s)}_{\text{IS}}$ and $\phi^{(t,s)}_{\text{Reg}}$ satisfy an important recursive property in the case that $q^{(t)} \in \mathbb Q^{(t)}$ or $h^{(t)} \in \mathbb H^{(H)}$ respectively.

\begin{lemma}
\label{lem:phi-is}
Suppose that $q^{(t)} \in \mathbb Q^{(t)}$, let $Y_t = \phi(R_t,D_{t+1:H})$, and let \cref{asm:pci-reduction} be given. Then, we have
\begin{equation*}
    \es{t}[q^{(t)}(Z_t,A_t) \indicator{A_t=E_t} Y_t] = \es{t+1}[Y_t] \,.
\end{equation*}
\end{lemma}

\begin{lemma}
\label{lem:phi-reg}
Suppose that $h^{(t)} \in \mathbb H^{(t,\phi)}$, let $Y_t = \phi(R_t,D_{t+1:H})$, and let \cref{asm:pci-reduction} be given. Then, we have
\begin{equation*}
    \es{t} \left[ \sum_{a \in \aset} h^{(t)}(W_t,a) \right] = \es{t+1}[Y_t] \,.
\end{equation*}
\end{lemma}

\begin{proof}[Proof of \cref{lem:phi-is}]
Given  that $q^{(t)} \in \mathbb Q^{(t)}$, we have
\begin{align*}
    &\es{t}[q^{(t)}(Z_t,A_t) \indicator{A_t=E_t} Y_t]  \\
    &= \es{t}[ \es{t}[q^{(t)}(Z_t,A_t) \mid U_t, A_t, E_t, Y_t(1),\ldots,Y_t(m)] \indicator{A_t=E_t} Y_t] \\
    &= \es{t}[ \es{t}[q^{(t)}(Z_t,A_t) \mid U_t, A_t] \indicator{A_t=E_t} Y_t] \\
    &= \es{t}[ \pd{t}(A_t \mid U_t)^{-1} \indicator{A_t=E_t} Y_t] \\
    &= \es{t}[ \pd{t}(A_t \mid U_t)^{-1} \indicator{A_t=E_t} Y_t(E_t)] \\
    &= \es{t} \left[ \sum_{a \in \aset} \frac{\pd{t}(A_t=a \mid U_t, E_t, Y_t(1),\ldots,Y_t(m))}{\pd{t}(A_t=a \mid U_t)} \indicator{E_t=a} Y_t(E_t) \right] \\
    &= \es{t} \left[ \sum_{a \in \aset} \indicator{E_t=a} Y_t(E_t) \right] \\
    &= \es{t}[ Y_t(E_t) ] \\
    &= \es{t+1}[ Y_t ] \,,
\end{align*}
where in the second and sixth equalities we apply the independence assumptions from \cref{asm:pci-reduction}, in the third equality we apply the fact that $q^{(t)} \in \mathbb Q^{(t)}$, in the fourth equality we apply the fact that $Y_t = Y_t(A_t)$, and in the final equality we apply the fact that by definition intervening on the $t$'th action with $E_t$ under $\ps{t}$ is by definition equivalent to $\ps{t+1}$.
\end{proof}

\begin{proof}[Proof of \cref{lem:phi-reg}]
Given  that $h^{(t,\phi)} \in \mathbb Q^{(t)}$, we have
\begin{align*}
    \es{t}\left[ \sum_{a \in \aset} h^{(t)}(W_t,a) \right] &= \es{t}\left[ \es{t}\left[ \sum_{a \in \aset} h^{(t)}(W_t,a) \ \middle| \ U_t \right] \right] \\
    &= \es{t}\left[ \es{t}\left[ \sum_{a \in \aset} h^{(t)}(W_t,A_t) \ \middle| \ U_t, A_t=a \right] \right] \\
    &= \es{t}\left[ \es{t}\left[ \sum_{a \in \aset} \indicator{E_t=A_t} Y_t \ \middle| \ U_t, A_t=a \right] \right] \\
    &= \es{t}\left[ \es{t}\left[ \sum_{a \in \aset} \indicator{E_t=a} Y_t(E_t) \ \middle| \ U_t, A_t=a \right] \right] \\
    &= \es{t}\left[ \es{t}\left[ \sum_{a \in \aset} \indicator{E_t=a} Y_t(E_t) \ \middle| \ U_t \right] \right] \\
    &= \es{t}\left[ \sum_{a \in \aset} \indicator{E_t=a} Y_t(E_t) \right] \\
    &= \es{t}[ Y_t(E_t) ] \\
    &= \es{t+1}[ Y_t ] \\
\end{align*}
where in the second and sixth equalities we apply the independence assumptions from \cref{asm:pci-reduction}, in the third equality we apply the fact that $h^{(t)} \in \mathbb H^{(t,\phi)}$, in the fourth equality we apply the fact that $Y_t = Y_t(A_t)$, and in the final equality we apply the fact that by definition intervening on the $t$'th action with $E_t$ under $\ps{t}$ is by definition equivalent to $\ps{t+1}$.

\end{proof}

Now, by the previous two lemmas, we would be able to establish identification via backward induction, if it were the case that the functions $q^{(t)}$ and $h^{(t,s)}$ used for identification were actually members of $\mathbb Q^{(t)}$ and $\mathbb H^{(t,\phi)}$ (for $\phi$ such that $\phi(R_t,D_{t+1:H}) = Y_t^{(s)}$). However, instead we assumed that $q^{(t)} \in \mathbb Q^{(t)}_{\text{obs}}$ and $h^{(t)} \in \mathbb H^{(t,\phi)}_{\text{obs}}$, so some additional care must be taken. The next lemma and its corollaries allow us to remedy this issue.

\begin{lemma}
\label{lem:bridge-observed-cross}
Let $q^{(t)} \in \mathbb Q^{(t)}_{\text{obs}}$ and $h^{(t)} \in \mathbb H^{(t,\phi)}_{\text{obs}}$ be chosen arbitrarily, for some given $Y_t = \phi(R_t,D_{t+1:H})$. Then we have
\begin{equation*}
    \es{t}[q^{(t)}(Z_t,A_t) \indicator{A_t=E_t} Y_t] = \es{t}\left[ \sum_{a \in \aset} h^{(t)}(W_t,a) \right] \,.
\end{equation*}
\end{lemma}

\begin{proof}[Proof of \cref{lem:bridge-observed-cross}]
We have
\begin{align*}
    \es{t}[q^{(t)}(Z_t,A_t) \indicator{A_t=E_t} Y_t] &= \es{t}[ q^{(t)}(Z_t,A_t) \es{t}[ \indicator{A_t=E_t} Y_t \mid Z_t, A_t] ] \\
    &= \es{t}[ q^{(t)}(Z_t,A_t) \es{t}[ h^{(t)}(W_t,A_t) \mid Z_t, A_t] ] \\
    &= \es{t}[ q^{(t)}(Z_t,A_t) h^{(t)}(W_t,A_t) ] \\
    &= \es{t}[ \es{t}[q^{(t)}(Z_t,A_t) \mid W_t,A_t] h^{(t)}(W_t,A_t) ] \\
    &= \es{t}[ \pd{t}(A_t \mid W_t)^{-1} h^{(t)}(W_t,A_t) ] \\
    &= \es{t} \left[ \sum_{a \in \aset} \frac{\pd{t}(A_t=a \mid W_t)}{\pd{t}(A_t=a \mid W_t)} h^{(t)}(W_t,a) \right] \\
    &= \es{t} \left[ \sum_{a \in \aset} h^{(t)}(W_t,a) \right] \,.
\end{align*}

\end{proof}

\begin{corollary}
\label{cor:phi-is}
Suppose that $q^{(t)} \in \mathbb Q^{(t)}_{\text{obs}}$, let $Y_t = \phi(R_t,D_{t+1:H})$, and let \cref{asm:pci-reduction,asm:bridge} be given. Then, we have
\begin{equation*}
    \es{t}[q^{(t)}(Z_t,A_t) \indicator{A_t=E_t} Y_t] = \es{t+1}[Y_t] \,.
\end{equation*}
\end{corollary}

\begin{corollary}
\label{cor:phi-reg}
Suppose that $h^{(t)} \in \mathbb H^{(t,\phi)}_{\text{obs}}$, let $Y_t = \phi(R_t,D_{t+1:H})$, and let \cref{asm:pci-reduction,asm:bridge} be given. Then, we have
\begin{equation*}
    \es{t}\left[ \sum_{a \in \aset} h^{(t)}(W_t,a) \right] = \es{t+1}[Y_t] \,.
\end{equation*}
\end{corollary}

\Cref{cor:phi-is} follows because from \cref{asm:bridge} there must exist some $h^{(t)} \in \mathbb H^{(t,\phi)}$, and by \cref{lem:bridge-subset} we know that $h^{(t)} \in \mathbb H^{(t,\phi)}_{\text{obs}}$, so therefore applying \cref{lem:bridge-observed-cross} and then \cref{lem:phi-reg} we have
\begin{equation*}
    \es{t}[q^{(t)}(Z_t,A_t) \indicator{A_t=E_t} Y_t] = \es{t}\left[ \sum_{a \in \aset} h^{(t)}(W_t,a) \right] = \es{t+1}[Y_t]\,.
\end{equation*}
\Cref{cor:phi-reg} follows by an almost identical logic, since by \cref{asm:bridge} there must exist some $q^{(t)} \in \mathbb Q^{(t)}$, and by \cref{lem:bridge-subset} we also know that $q^{(t)} \in \mathbb Q^{(t)}_{\text{obs}}$. Therefore, applying \cref{lem:bridge-observed-cross} and then \cref{lem:phi-is} we have
\begin{equation*}
    \es{t}\left[ \sum_{a \in \aset} h^{(t)}(W_t,a) \right] = \es{t}[q^{(t)}(Z_t,A_t) \indicator{A_t=E_t} Y_t]  = \es{t+1}[Y_t]\,.
\end{equation*}

These corollaries are sufficient to construct our inductive proof for our main identification result, in the case of $\phi^{(t,s)}_{\text{IS}}$ and $\phi^{(t,s)}_{\text{Reg}}$. However, for the case of $\phi^{(t,s)}_{\text{DR}}$ we need to establish one final lemma before presenting our main proof.

\begin{lemma}
\label{lem:phi-dr}
Suppose that either $q^{(t)} \in \mathbb Q^{(t)}_{\text{obs}}$ and $h^{(t)} \in L_{2,\ps{t}}(W_t,A_t)$ \emph{or} $h^{(t)} \in \mathbb H^{(t,\phi)}_{\text{obs}}$ and $q^{(t)} \in L_{2,\ps{t}}(Z_t,A_t)$. In addition, let $Y_t = \phi(R_t,D_{t+1:H})$, and let \cref{asm:pci-reduction,asm:bridge} be given. Then, we have
\begin{equation*}
    \es{t}\left[ q^{(t)}(Z_t,A_t) (\indicator{A_t=E_t} Y_t - h^{(t)}(W_t,A_t)) + \sum_{a \in \aset} h^{(t)}(W_t,a) \right] = \es{t+1}[Y_t] \,.
\end{equation*}
\end{lemma}

\begin{proof}[Proof of \cref{lem:phi-dr}]

First consider the case where $q^{(t)} \in \mathbb Q^{(t)}_{\text{obs}}$ and $h^{(t)} \in L_{2,\ps{t}}(W_t,A_t)$. In this case, we have
\begin{align*}
    &\es{t}\left[ q^{(t)}(Z_t,A_t) (\indicator{A_t=E_t} Y_t - h^{(t)}(W_t,A_t)) + \sum_{a \in \aset} h^{(t)}(W_t,a) \right] \\
    &= \es{t}t[ q^{(t)}(Z_t,A_t) \indicator{A_t=E_t} Y_t] + \es{t}\left[ \sum_{a \in \aset} h^{(t)}(W_t,a) \right] \\
    &\qquad - \es{t}[ q^{(t)}(Z_t,A_t) h^{(t)}(W_t,A_t) ] \\
    &= \es{t+1}[Y_t] + \es{t}\left[ \sum_{a \in \aset} h^{(t)}(W_t,a) \right] - \es{t}[ q^{(t)}(Z_t,A_t) h^{(t)}(W_t,A_t) ] \,,
\end{align*}
where in the second equality we apply \cref{cor:phi-is}. Now, given $q^{(t)} \in \mathbb Q^{(t)}_{\text{obs}}$ we can further establish
\begin{align*}
    \es{t}[ q^{(t)}(Z_t,A_t) h^{(t)}(W_t,A_t) ] &= \es{t}\left[ \es{t}[q^{(t)}(Z_t,A_t) \mid W_t,A_t] h^{(t)}(W_t,A_t) \right] \\
    &= \es{t}\left[ \pd{t}(A_t \mid W_t)^{-1} h^{(t)}(W_t,A_t) \right] \\
    &= \es{t}\left[ \sum_{a \in \aset} h^{(t)}(W_t,A_t) \right]\,.
\end{align*}
Thus, plugging this into the previous equation we have
\begin{equation*}
    \es{t}\left[ q^{(t)}(Z_t,A_t) (\indicator{A_t=E_t} Y_t - h^{(t)}(W_t,A_t)) + \sum_{a \in \aset} h^{(t)}(W_t,a) \right] = \es{t+1}[Y_t] \,.
\end{equation*}

Next, instead consider the case where $h^{(t)} \in \mathbb H^{(t,\phi)}_{\text{obs}}$ and $q^{(t)} \in L_{2,\ps{t}}(Z_t,A_t)$. In this case, we have
\begin{align*}
    &\es{t}\left[ q^{(t)}(Z_t,A_t) (\indicator{A_t=E_t} Y_t - h^{(t)}(W_t,A_t)) + \sum_{a \in \aset} h^{(t)}(W_t,a) \right] \\
    &= \es{t} \left[ q^{(t)}(Z_t,A_t) \es{t}[ \indicator{A_t=E_t} Y_t - h^{(t)}(W_t,A_t) \mid Z_t, A_t] \right] + \es{t}\left[ \sum_{a \in \aset} h^{(t)}(W_t,a) \right] \\
    &= 0 + \es{t+1}[Y_t] \,,
\end{align*}
where the second equality follows from \cref{cor:phi-reg} and the fact that $h^{(t)} \in \mathbb H^{(t,\phi)}_{\text{obs}}$.  Therefore, under either conditions we have our desired result.

\end{proof}

Now that we have established these preliminary lemmas, we are ready to present the main proof.

\begin{proof}[Proof of \cref{thm:identification-pci-general}]

First, we have assumed \cref{asm:pci-reduction,asm:bridge}, as well as the fact that $q^{(t)} \in \mathbb Q^{(t)}_{\text{obs}}$ and $h^{(t)} \in \mathbb H^{(t,\phi)}_{\text{obs}}$, so it follows from \cref{cor:phi-is,cor:phi-reg,lem:phi-dr} that for any of the choices of $\phi^{(t+1,s)}_{\text{IS}}$, $\phi^{(t+1,s)}_{\text{Reg}}$, or $\phi^{(t+1,s)}_{\text{DR}}$ for defining each $Y_t^{(s)}$ term (for $t<s$) we have
\begin{equation*}
    \es{t}[\phi^{(t,s)}(Z_t,W_t,A_t,E_t,Y_t^{(s)})] = \es{t+1}[Y_t^{(s)}]\,,
\end{equation*}
which holds for every $t < s$. Furthermore, plugging in the recursive definition of $Y_{t-1}^{(s)}$, the previous equation is equivalent to
\begin{equation*}
    \es{t}[Y_{t-1}^{(s)}] = \es{t+1}[Y_t^{(s)}]\,,
\end{equation*}
which again holds for every $t < s$. Therefore, by backward induction we have
\begin{equation*}
    \es{1}[Y_0^{(s)}] = \es{s+1}[R_s]\,.
\end{equation*}
However, by construction $\ps{1} = \pb$, and the distribution of $R_s$ under $\ps{s+1}$ is the same as under $\pe$, so therefore we have $\eb[Y_0^{(s)}] = \epe[R_s]$, as required.

\end{proof}

\section{Proof of Consistency and Asymptotic Normality}
\label{apx:consistency-proof}

We will prove this theorem by appealing to \citet[Theorem 3.1]{chernozhukov2016double}. Therefore, this proof will consist of establishing the conditions of this theorem. We will first present a lemma establishing the Newman orthogonality property of this influence function, which not only is a condition of \citet[Theorem 3.1]{chernozhukov2016double} but an important property in its own right, before presenting the rest of the proof.

In what follows below, for any generic quantity $\Psi$ that depends on our nuisance functions, we will use the notation $\hat\Psi$ to refer to the value of $\Psi$ using the estimated nuisance functions $\hat q^{(t)}$ and $\hat h^{(t)}$ in place of $q^{(t)}$ and $h^{(t)}$ respectively for each $t \in [H]$, and define $\Delta\Psi = \hat\Psi - \Psi$. In addition, for any $r \in [0,1]$ we let $\Psi|_r$ refer to the value of $\Psi$ using the nuisances $q_r{(t)} = q^{(t)} + r\Delta q^{(t)}$ and $h_r^{(t)} = h^{(t)} + r\Delta h^{(t)}$ in place of $q^{(t)}$ and $h^{(t)}$ respectively for each $t \in [H]$, and define $\Delta_r\Psi=\Psi|_r-\Psi$. We note that according to these definitions, $\Psi = \Psi|_0$, $\hat\Psi=\Psi|_1$, and $\Delta\Psi=\Delta_1\Psi$.
In what follows below we will treat $\Delta q^{(t)}$ and $\Delta h^{(t)}$ as non-random square integrable functions with the same signature as $q^{(t)}$ and $h^{(t)}$ respectively for each $t \in [H]$, which may take arbitrary values. This is in contrast to previous sections, where $\hat\Psi$ was treated as a random quantity with respect to the sampling distribution of the $n$ i.i.d. behavior trajectories.
Finally, we note that it is trivial to verify that for any pair of quantities $\Psi$ and $\Psi'$ we have $\Delta_r(\Psi + \Psi') = \Delta_r\Psi + \Delta_r\Psi'$, and $\Delta_r(\Psi\Psi') = (\Delta_r\Psi)\Psi' + \Psi(\Delta_r\Psi') + (\Delta_r\Psi)(\Delta_r\Psi')$, which we will frequently apply in the derivations below without further explanation.

\begin{lemma}
\label{lem:orthogonality}

Under the conditions of \cref{thm:identification-pci}, as well as the additional assumption that $\|q^{(t)}(Z_t,A_t)\| < \infty$ and $\|h^{(t)}(W_t,A_t)\| < \infty$ for each $t \in [H]$,
$\psidr$ satisfies Neyman orthogonality with respect to the nuisances $q^{(t)}$ and $h^{(t)}$ for all $t \in [H]$. More concretely, for any arbitrary functions $\Delta q^{(t)}(Z_t,A_t)$ and $\Delta h^{(t)}(W_t,A_t)$ for each $t \in [H]$ that have finite supremum norm, we have
\begin{equation*}
    \left. \frac{\partial}{\partial r}\right|_{r=0} \eb[\psidr(\tau_H)|_r] = 0\,.
\end{equation*}

\end{lemma}

\begin{proof}[Proof of \cref{lem:orthogonality}]

First, we note that
\begin{equation*}
    \psidr(\tau_H)|_r = \sum_{a \in \aset} h^{(1)}_r(W_1,a) + q^{(1)}_r(Z_1,A_1) \left(\indicator{A_1=E_1}Y_{1,r} - h^{(1)}_r(W_1,A_1)\right)\,,
\end{equation*}
where
\begin{equation*}
    Y_{t-1,r} = R_{t-1} + \gamma \left( \sum_{a \in \aset} h^{(t)}_r(W_t,a) + q^{(t)}_r(Z_t,A_t) \left( \indicator{A_t=E_t}Y_{t,r} - h^{(t)}_r(W_t,A_t) \right) \right) \\
\end{equation*}
for all $t \in \{2,3,\ldots,H\}$, and
\begin{equation*}
    Y_{H,r} = R_H \,.
\end{equation*}
Therefore, we have
\begin{align*}
    \left. \frac{\partial}{\partial r} \right|_{r=0} \psidr(\tau_H)|_r &= \sum_{a \in \aset} \Delta h^{(1)}(W_1,a) - q^{(1)}(Z_1,A_1) \Delta h^{(1)}(W_1,A_1)  \\
    &\quad + \left( \Delta q^{(1)}(Z_1,A_1) \right) \left( \indicator{A_1=E_1} Y_{1,0} - h^{(1)}(W_1,A_1) \right) \\
    &\quad + q^{(1)}(Z_1,A_1) \indicator{A_1=E_1} \left. \frac{\partial}{\partial r} \right|_{r=0} Y_{1,r} \,.
\end{align*}
Now, following an identical argument as in \cref{lem:bridge-observed-cross}, for any $t \in [H]$ and arbitrary functions $\tilde h$ and $\tilde q$ with he same signature as $h^{(t)}$ and $q^{(t)}$ we have
\begin{align}
    \label{eq:tilde-h}
    \es{t} \left[ \sum_{a \in \aset} \tilde h(W_t,a) - q^{(t)}(Z_t,A_t) \tilde h(W_t,A_t) \right] &= 0 \\
    \label{eq:tilde-q}
    \es{t} \left[ \tilde q(Z_t,A_t)  \left( \indicator{A_t=E_t} Y_{t,0} - h^{(t)}(W_t,A_t) \right) \right] &= 0 \,.
\end{align}
In addition, it is straightforward to argue via backward induction on $t$ that
\begin{equation*}
    \es{t} \left[ q^{(t)}(Z_t,A_t) \indicator{A_t=E_t)} \left. \frac{\partial}{\partial r} \right|_{r=0} Y_{t,r} \right] = 0\,,
\end{equation*}
for every $t \in [H]$. For the base case where $t=H$, this is straightforward since $Y_{H,r} = R_H$, which doesn't depend on $r$. Otherwise, for the inductive case, we first note that following \cref{lem:phi-is}, for any $t\leq H$ we have
\begin{equation*}
    \es{t-1} \left[ q^{(t-1)}(Z_{t-1},A_{t-1}) \indicator{A_{t-1}=E_{t-1})} \left. \frac{\partial}{\partial r} \right|_{r=0} Y_{t-1,r} \right] = \es{t} \left[ \left. \frac{\partial}{\partial r} \right|_{r=0} Y_{t-1,r} \right] \,,
\end{equation*}
and furthermore we have
\begin{align*}
    \left. \frac{\partial}{\partial r} \right|_{r=0} Y_{t-1,r} &= \gamma \sum_{a \in \aset} \Delta h^{(t)}(W_t,a) - \gamma q^{(t)}(Z_t,A_t) \Delta h^{(t)}(W_t,A_t)  \\
    &\quad + \gamma \Delta q^{(t)}(Z_t,A_t)  \left( \indicator{A_t=E_t} Y_{t,0} - h^{(t)}(W_t,A_t) \right) \\
    &\quad + \gamma q^{(t)}(Z_t,A_t) \indicator{A_t=E_t} \left. \frac{\partial }{\partial r} \right|_{r=0} Y_{t,r} \,.
\end{align*}
Now, given \cref{eq:tilde-h,eq:tilde-q} it immediately follows that the first three terms have mean zero under $\es{t}$. Furthermore, the final term also has mean zero by the inductive assumption. This completes the backward induction, and therefore putting the above together we have 
\begin{equation*}
    \eb\left[ \left. \frac{\partial}{\partial r}\right|_{r=0} \psidr(\tau_H)|_r \right] = 0\,.
\end{equation*}

It only remains to argue that we can swap the order of differentiation and integration in this equation. By the mean value theorem, we have
\begin{equation*}
    \frac{\psidr(\tau_H)|_r - \psidr(\tau_H)}{r} = \left. \frac{\partial}{\partial r}\right|_{r=r'} \psidr(\tau_H)|_r \,,
\end{equation*}
for some $r' \in (0, r)$. Therefore, it follows that
\begin{equation*}
    \left. \frac{\partial}{\partial r}\right|_{r=0} \eb\left[  \psidr(\tau_H)|_r \right] = \lim_{r' \to 0} \eb\left[ \left. \frac{\partial}{\partial r} \right|_{r=r'} \psidr(\tau_H)|_r \right]\,.
\end{equation*}
Now, clearly $\left. \frac{\partial}{\partial r} \right|_{r=r'} \psidr(\tau_H)|_r$ converges point-wise to $\left. \frac{\partial}{\partial r} \right|_{r=0} \psidr(\tau_H)|_r$ as $r' \to 0$. Furthermore, it follows trivially from the definitions of $q_r^{(t)}$ and $h_r^{(t)}$ that $\left. \frac{\partial}{\partial r} \right|_{r=r'} q^{(t)}_r = \Delta q^{(t)}$, and $\left. \frac{\partial}{\partial r} \right|_{r=r'} h^{(t)}_r = \Delta h^{(t)}$ for every $t \in [H]$ and $r' \in [0, 1]$. Therefore, for any $r' \in [0,1]$ we have
\begin{align*}
    &\left\| \left. \frac{\partial}{\partial r} \right|_{r=r'} \psidr(\tau_H)|_r \right\|_\infty  \\
    &\leq \sum_{a \in \aset} \| \Delta h^{(1)}(W_1,a) \|_\infty \\
    &\quad + \|\Delta q^{(1)} (Z_1, A_1) \|_\infty \Big( \|Y_{1,r'}\|_\infty + \|h^{(1)}(W_1,A_1)\|_\infty + \|\Delta h^{(1)}(W_1,A_1)\|_\infty \Big) \\
    &\quad + \Big( \|q^{(1)}(Z_1,A_1)\|_\infty + \|\Delta q^{(1)}(Z_1,A_1)\|_\infty \Big) \\
    &\qquad\qquad \cdot \left( \| \Delta h^{(1)}(W_1,A_1) \|_\infty + \left\| \left. \frac{\partial}{\partial r} \right|_{r=r'} Y_{1,r} \right\|_\infty \right) \,.
\end{align*}
Now, by assumption we know that $q^{(t)}$, $h^{(t)}$, $\Delta q^{(t)}$, and $\Delta h^{(t)}$ all have bounded supremum norm, so therefore $\left. \frac{\partial}{\partial r} \right|_{r=r'} (\psidr(\tau_H))_r$ is uniformly bounded in supremum norm over $r' \in [0,1]$, as long as we can uniformly bound $\|Y_{1,r'}\|_\infty$ and $\| \left. \frac{\partial}{\partial r} \right|_{r=r'} Y_{1,r}\|_\infty$. Therefore, our final result would follow by the dominated convergence theorem as long as we can show these two bounds. In order to show this, bote that for any $t < H$ and $r' \in [0,1]$ we can bound
\begin{align*}
    \left\| Y_{t,r}  \right\|_\infty &\leq \| R_t \|_\infty + \gamma \sum_{a \in \aset} \left( \| h^{(t+1)}(W_{t+1},a) \|_\infty + \| \Delta h^{(t+1)}(Z_{t+1},a) \|_\infty \right) \\
    &\quad + \gamma \left( \| q^{(t+1)}(W_{t+1},A_{t+1}) \|_\infty + \| \Delta q^{(t+1)}(Z_{t+1},A_{t+1}) \|_\infty \right) \\
    &\qquad \cdot \left( \| Y_{t+1, r} \|_\infty + \| h^{(t+1)}(W_{t+1},a) \|_\infty + \| \Delta q^{(t+1)}(Z_{t+1},A_{t+1}) \|_\infty \right) \,,
\end{align*}
and
\begin{align*}
    &\left\| \left. \frac{\partial}{\partial r} \right|_{r=r'} Y_{t,r} \right\|_\infty \\
    &\leq \gamma \sum_{a \in \aset} \| \Delta h^{(t+1)}(W_1,a) \|_\infty \\
    &\quad + \gamma \|\Delta q^{(t+1)} (Z_{t+1}, A_{t+1}) \|_\infty \\
    &\qquad \cdot \left( \|Y_{t+1,r'}\|_\infty + \|h^{(t+1)}(W_{t+1},A_{t+1})\|_\infty + \|\Delta h^{(t+1)}(W_{t+1},A_{t+1})\|_\infty \right) \\
    &\quad + \gamma \left( \| q^{(t+1)}(Z_{t+1},A_{t+1}) \|_\infty + \| \Delta q^{(t+1)}(Z_{t+1},A_{t+1}) \|_\infty \right) \\
    &\qquad \cdot \left( \| \Delta h^{(t+1)}(W_{t+1},A_{t+1}) \|_\infty + \left\| \left. \frac{\partial}{\partial r} \right|_{r=r'} Y_{t+1,r} \right\|_\infty \right) \,,
\end{align*}
and for $t=H$ we have $\|Y_{H,r'}\|_\infty = \|R_H\|_\infty$ and $\left\| \left. \frac{\partial}{\partial r} \right|_{r=r'} Y_{H,r} \right\|_\infty = 0$ for all $r'$. Therefore, given this and the fact that all rewards are bounded by assumption, it follows from backward induction that $\|Y_{1,r'}\|_\infty$ and $\| \left. \frac{\partial}{\partial r} \right|_{r=r'} Y_{1,r}\|_\infty$ are uniformly bounded over $r' \in [0,1]$. Thus, the final result follows by the dominated convergence theorem as discussed above.

\end{proof}

\begin{proof}[Proof of \cref{thm:consistency}]

We will establish this proof based on directly applying \citet[Theorem 3.1]{chernozhukov2016double}. We note that this theorem immediately implies \cref{thm:consistency}, as long as we can establish that its conditions hold. We will proceed by establishing these conditions one by one. Specifically, we need to establish their Assumption 3.1 and Assumption 3.2.

We note that \citet[Theorem 3.1]{chernozhukov2016double} considers linear scores, where the parameter of interest $\theta$ is the solution to
\begin{equation}
\label{eq:linear-score}
    \e[\psi(\tau_H;\theta,\xi)] = \e[\psi^a(\tau_H;\xi) \theta + \psi^b(\tau_H;\xi)] = 0\,,
\end{equation}
where $\xi$ represents the nuisance functions, and $\psi(\tau_H;\theta,\xi) = \psi^a(\tau_H;\xi) \theta + \psi^b(\tau_H;\xi)$. In our case, the notation above corresponds to our quantities of interest as follows: $\theta = v_\gamma(\pi_e)$ is the target policy value; $\xi = q^{(1)},\ldots,q^{(H)},h^{(1)},\ldots,h^{(H)}$; $\psi^a(\tau_H;\xi) = -1$; and $\psi^b(\tau_H;\xi)$ is equal to $\psidr(\tau_H)$ with the true nuisances replaced with $\xi$.

Let $\Xi(c,c',c'')$ denote the set of all possible nuisances $\xi$ such that, for all $t \in [H]$,  $a \in \aset$ and $\Psi \in \{q^{(t)}(Z_t,A_t), h^{(t)}(W_t,a)\}$, we have:
\begin{align*}
    \|\Delta \Psi\|_\infty &\leq c \\
    \|\Delta \Psi\|_{2,\pb} &\leq c' \\
    \|\Delta q^{(t)}(Z_t,A_t) \Delta h^{(t)}(W_t,a)\|_{2,\pb} &\leq c'' \\
    \|\Delta q^{(t')}(Z_{t'},A_{t'}) \Delta \Psi\|_{2,\pb} &\leq c'' \quad \forall t' < t \,.
\end{align*}
In addition, let $\xi_0$ denote the value of the true nuisances, and let $\hat\xi_n$ denote the estimated nuisances using $n$ trajectories $\tau_H^{(1)},\ldots,\tau_H^{(n)}$. We note that by \cref{asm:estimation-error} there must exist some constant $c^*$ and sequences $c'_n \in o(1)$ and $c''_n \in o(n^{-1/2})$ such that $P(\hat\xi_n \in \Xi(c^*,c'_n,c''_n)) \to 1$ as $n \to \infty$, and let $\Xi_n = \Xi(c^*,c'_n,c''_n)$. We also note that according to the above definitions $\psidr(\tau_H) = \psi^b(\tau_H;\xi_0)$, and $\xi_0 \in \Xi_n$ for all $n$.

First, let us establish \citet[Assumption 3.1]{chernozhukov2016double}. For (i) and (ii) we note that $v_\gamma(\pi_e)$ satisfies \cref{eq:linear-score} trivially given \cref{thm:identification-pci}, and that this score is linear in $\theta$. For (iii) we note that following an almost identical argument as in the proof of \cref{lem:orthogonality} by applying the dominated convergence theorem, it is trivial that
$\e[\psidr(\tau_H)]$ is twice Gateaux differentiable for any $\xi - \xi_0 \in \Xi_n$, since for all such directions we have $\|\Delta q^{(t)}(Z_t,A_t)\|_\infty < \infty$ and $\|\Delta h^{(t)}(W_t,A_t)\|_\infty < \infty$. Condition (iv) holds trivially with nuisance set $\Xi_n$ for \emph{any} sequence $\delta_n$, since by \cref{lem:orthogonality} we have exact Neyman orthogonality. Finally, (v) holds trivially for \emph{any} $c_0 \in (0,1)$ and $c_1 \in (1,\infty)$, since $\psi^a(\tau_H;\xi)$ is constant.

Next, let us establish \citet[Assumption 3.2]{chernozhukov2016double}. For (i), we note that by \cref{asm:estimation-error} we have that $\hat\xi_n \in \Xi_n$ with probability approaching one. For (ii), we note that given the definition of $\Xi_n$ and the fact that rewards are bounded, it trivially follows that for \emph{any} $q>2$ we have
\begin{equation*}
    \sup_{\xi \in \Xi(c^*)} \| \psi^b(\tau_H;\xi) \|_{q,\pb} \leq \sup_{\xi \in \Xi(c^*)} \| \psi^b(\tau_H;\xi) \|_\infty < \infty\,,
\end{equation*}
from which (ii) follows trivially, since we showed that part (v) of \citet[Assumption 3.2]{chernozhukov2016double} holds for \emph{any} arbitrary $c_1 \in (1,\infty)$.
For (iii), we can first note that clearly
\begin{equation*}
    \sup_{\xi \in \Xi(c^*)} | \eb[\psi^a(\tau_H;\xi)] - \eb[\psi^a(\tau_H;\xi_0)] | = 0\,,
\end{equation*}
since $\psi^a$ is constant. Second, we note that for any $\hat\xi \in \Xi_n$, we have
\begin{align*}
    &\|\psi(\tau_H;v_\gamma(\pi_e),\hat \xi) - \psi(\tau_H;v_\gamma(\pi_e),\xi_0) \|_{2,\pb} \\
    &= \|\Delta \psi^b(\tau_H;\xi) \|_{2,\pb} \\
    &\leq \sum_{a \in \aset} \left\| \Delta h^{(1)}(W_1,a) \right\|_{2,\pb} + \|q^{(1)}(Z_1,A_1)\|_\infty \left( \left\|\Delta h^{(1)}(W_1,A_1) \right\|_{2,\pb} + \left\|\Delta Y_1 \right\|_{2,\pb} \right) \\
    &\qquad + \|\Delta q^{(1)}(Z_1,A_1)\|_{2,\pb} \left( \left\|h^{(1)}(W_1,A_1) \right\|_\infty + \left\|Y_1 \right\|_\infty \right) \\
    &\qquad + \|\Delta q^{(1)}(Z_1,A_1)\|_{2,\pb} \left( \left\|\Delta h^{(1)}(W_1,A_1) \right\|_{2,\pb} + \left\|\Delta Y_1 \right\|_{2,\pb} \right) \\
    &\leq |\aset| c'_n + c^*(c'_n + \left\|\Delta Y_1 \right\|_{2,\pb}) + (c^* + \left\|\Delta Y_1 \right\|_\infty) c'_n + (c'_n + \left\|\Delta Y_1 \right\|_{2,\pb}) c'_n \,,
\end{align*}
and furthermore for each $1 < t \leq H$ we have
\begin{align*}
    &\left\|\Delta Y_{t-1} \right\|_{2,\pb} \\
    &\leq \gamma \sum_{a \in \aset} \left\| \Delta h^{(t)}(W_t,a) \right\|_{2,\pb} + \gamma \|q^{(t)}(Z_t,A_t)\|_\infty \left( \left\|\Delta h^{(t)}(W_t,A_t) \right\|_{2,\pb} + \left\|\Delta Y_t \right\|_{2,\pb} \right) \\
    &\qquad + \gamma \|\Delta q^{(t)}(Z_t,A_t)\|_{2,\pb} \left( \left\|h^{(t)}(W_t,A_t) \right\|_\infty + \left\|Y_t \right\|_\infty \right) \\
    &\qquad + \gamma \|\Delta q^{(t)}(Z_t,A_t)\|_{2,\pb} \left( \left\|\Delta h^{(t)}(W_t,A_t) \right\|_{2,\pb} + \left\|\Delta Y_t \right\|_{2,\pb} \right) \\
    &\leq \gamma |\aset| c'_n + \gamma c^*(c'_n + \left\|\Delta Y_t \right\|_{2,\pb}) + \gamma (c^* + \left\|\Delta Y_t \right\|_\infty) c'_n + \gamma (c'_n + \left\|\Delta Y_t \right\|_{2,\pb}) c'_n \,.
\end{align*}
Now, as argued in the proof of \cref{lem:orthogonality}, it easily follows by backward induction on $t$ that $\|Y_t\|_\infty < \infty$. Therefore, given the above equation it easily follows again by backward induction on $t$ that $\|\Delta Y_t \|_{2,\pb} \leq \delta^{(1)}_n$ and $\|\Delta \psi^b(\tau_H;\xi) \|_{2,\pb} \leq \delta^{(1)}_n$ for all $\hat\xi \in \Xi_n$, where the sequence $\delta^{(1)}_n \in o(1)$ does not depend on $\hat\xi$. Thus, we have
\begin{equation*}
    \sup_{\xi \in \Xi_n} \left\|\psi(\tau_H;v_\gamma(\pi_e),\hat \xi) - \psi(\tau_H;v_\gamma(\pi_e),\xi_0) \right\|_{2,\pb} \leq \delta^{(1)}_n\,.
\end{equation*}
Similarly, recalling that
\begin{equation*}
    \psidr(\tau_H) = \sum_{t=1}^H \gamma^{t-1} \left( \eta_{t+1}R_t + \eta_t \sum_{a \in \aset} h^{(t)}(W_t,a) - \eta_t q^{(t)}(Z_t,A_t) h^{(t)}(W_t,A_t) \right)\,,
\end{equation*}
for any $\hat\xi \in \Xi_n$ we also have
\begin{align*}
    &\frac{\partial^2}{\partial r^2} \psi^b(\tau_H;\xi|_r) \\
    &=  \sum_{t=1}^H \gamma^{t-1} \frac{\partial^2}{\partial r^2} \eta_{t+1,r} \left(R_t + \sum_{a \in \aset} h_r^{(t)}(W_t,a) - q_r^{(t)}(Z_t,A_t) h_r^{(t)}(W_t,A_t) \right)  \\
    &\qquad + \sum_{t=1}^H \sum_{a \in \aset} \gamma^{t-1} \frac{\partial}{\partial r} \eta_{t,r}  \Delta h^{(t)}(W_t,a)  \\
    &\qquad + \sum_{t=1}^H \gamma^{t-1} \frac{\partial}{\partial r} \eta_{t,r} \left( \Delta q^{(t)}(Z_t,A_t) h_r(W_t,A_t) + \Delta h^{(t)}(W_t,A_t) q_r^{(t)}(Z_t,A_t) \right) \,,
\end{align*}
and
\begin{align*}
    \frac{\partial}{\partial r} \eta_{t,r} &= \sum_{s \in [t]} \Delta q^{(s)}(Z_s,A_s) b_r(s) \\
    \frac{\partial^2}{\partial r^2} \eta_{t,r} &= \sum_{s \in [t]} \sum_{s' \in [t] \setminus \{s\}} \Delta q^{(s)}(Z_s,A_s) \Delta q^{(s')}(Z_{s'},A_{s'}) b_r(s,s')
\end{align*}
where
\begin{align*}
    b_r(s) &= \indicator{E_s=A_s} \prod_{h \in [t] \setminus \{s\}} \indicator{E_{h}=A_{h}} q_r^{(h)}(Z_{h},A_{h}) \\
    b_r(s,s') &= \indicator{E_s=A_s} \indicator{E_{s'}=A_{s'}} \prod_{h \in [t] \setminus \{s,s'\}} \indicator{E_{h}=A_{h}} q_r^{(h)}(Z_{h},A_{h})\,.
\end{align*}
Therefore, it easily follows from the above and the definition of $\Xi_n$ that for all $\hat\xi \in \Xi_n$ we have
\begin{align*}
    \left\| \frac{\partial^2}{\partial r^2} \psi^b(\tau_H;\xi|_r) \right\|_\infty \leq k_1 \\
    \left\| \frac{\partial^2}{\partial r^2} \psi^b(\tau_H;\xi|_r) \right\|_{1,\pb} \leq k_2 c''_n \,,
\end{align*}
for some constants $k_1$ and $k_2$ that don't depend on $\hat\xi$ or on $r$. Therefore, we can apply the dominated convergence theorem to obtain
\begin{equation*}
    \sup_{r \in (0,1), \xi \in \Xi_n} \left| \eb{\psi(\tau_H;v_\gamma(\pi_e),r\xi + (1-r)\xi_0)} \right| \leq k_2 c''_n \,.
\end{equation*}
Putting the above together, we have condition (iii) with $\delta_n = \max(\delta^{(1)}_n, k_2 \sqrt{n} c''_n)$. Finally, condition (iv) follows since by assumption the variance of $\psidr(\tau_H)$ is non-zero.

Therefore, we have established the conditions of \citet[Theorem 3.1]{chernozhukov2016double}, so applying this theorem we have
\begin{equation*}
    \sqrt{n}(\hat v^{(n)}_\gamma(\pi_e) - v_\gamma(\pi_e)) = \frac{1}{\sqrt{n}} \sum_{i=1}^n \psidr(\tau_H^{(i)}) + o_p(1)\,,
\end{equation*}
which by the central limit theorem and Slutsky's theorem converges in distribution to $\mathcal N(0,\sigma^2_{\textup{DR}})$.

\end{proof}

\section{Background on Semiparametric Efficiency Theory}
\label{apx:semiparametric}

In this appendix we provide a brief review of semiparametric efficiency theory, as relevant for the theory in this paper. We will consider a random variable $X \in \mathcal X$, a model (set of distributions) $\mathcal M$, where each $P \in \mathcal M$ defines a distribution for $X$, and some scalar parameter $v : \mathcal M \mapsto \mathbb R$. Also let $\mu$ denote some dominating measure such that $P \ll \mu$ for every $P \in \mathcal P$, and denote the corresponding density as $dP/d\mu$.
Given i.i.d. observations $X_1,\ldots,X_n$ sampled from some $P_0 \in \mathcal M$, semiparametric efficiency theory concerns itself with the limits on the estimation of $v(P_0)$, given that the estimator is required to be consistent and ``well behaved'' (defined concretely below) at all $P$ in a neighborhood of $P_0$ in the model $\mathcal M$.

\subsection{Definitions}

\begin{definition}[Influence function of estimators]
An estimator sequence $\hat v_n(X_{1:n})$ is asymptotically linear (AL) with influence function (IF) $\psi_{P_0}(X)$ if
\begin{equation*}
    \sqrt{n}(\hat v_n(X_{1:n}) - v(P_0)) = \frac{1}{\sqrt{n}} \sum_{i=1}^n \psi_{P_0}(X) + o_p(1)
\end{equation*}
where $\e_{P_0}[\psi_{P_0}(X)] = 0$.
\end{definition}

\begin{definition}[One-dimensional submodel and its score function]
A one-dimensional submodel of $\mathcal M$ passing through $P$ is a set of distributions $\{P_\epsilon : \epsilon \in U\} \subseteq \mathcal M$, where:
\begin{enumerate}
    \item $P_0 = P$
    \item The score function $s(X;\epsilon) = (d / d \epsilon) \log( (d P_\epsilon / d \mu)(X))$ exists
    \item There exists $u>0$ \st $\int \sup_{|\epsilon| \leq u} |s(X;\epsilon)| (d P_\epsilon / d \mu)(X) d \mu(X) < \infty$ and $\e[\sup_{|\epsilon| \leq u} s(X;\epsilon)^2] < \infty$\,.
\end{enumerate}
Also, we define $s(X) = s(X;0)$, which we refer to as the score function of the submodel at $P_0$, Note that by property (3) we have $s(X) \in L_{2,P}(X)$. We also note that these conditions on the parametric sub-model are slightly stronger than those in some related work; these are needed to prove our semiparametric efficiency results with full rigor, and our definitions below should be interpreted \wrt such well-behaved submodels.
\end{definition}

\begin{definition}[Tangent space]
The tangent space of $\mathcal M$ at $P_0$ is the linear closure of the score function at $P_0$ of all one-dimensional submodels of $\mathcal M$ passing through $P_0$.
\end{definition}

Note that the tangent space is always a cone, since we can always redefine any one-dimensional parametric submodel replacing $\epsilon$ with any scalar multiple of $\epsilon$.

\begin{definition}[Pathwise differentiability]
A functional $v : \mathcal M \mapsto \mathbb R$ is pathwise differentiable at $P_0$ wrt $\mathcal M$ if there exists a mean-zero function $\psi_{P_0}(X)$, such that any one-dimensional submodel $\{P_\epsilon\}$ of $\mathcal M$ passing through $P_0$ with score function $s(X)$ satisfies
\begin{equation*}
    \left. \frac{d v(P_\epsilon)}{d \epsilon} \right|_{\epsilon=0} = \e[\psi_{P_0}(X) s(X)]\,.
\end{equation*}
\end{definition}

The function $\psi_{P_0}(X)$ is called a gradient of $v(P_0)$ at $P_0$ wrt $\mathcal M$. The efficient IF (EIF, or canonical gradient) of $v(P_0)$ wrt $\mathcal M$ is the unique gradient $\tilde \psi_{P_0}(X)$ of $v(P_0)$ at $P_0$ wrt $\mathcal M$ that belongs to the tangent space at $P_0$ wrt $\mathcal M$.

Finally, we define regular estimators, which are those whose limiting distribution is robust to local changes to the data generating process. This is what we alluded to above by ``well behaved'' estimators. Note that restricting attention to regular estimators excludes pathological behavior such as that of the super-efficient Hodges estimator. 

\begin{definition}[Regular estimators]
An estimator sequence $\hat v_n$ is called regular at $P_0$ for $v(P_0)$ wrt $\mathcal M$ if there exists a limiting probability measure $L$ such that, for any one-dimensional submodel $\{P_\epsilon\}$ of $\mathcal M$ passing through $P_0$, we have
\begin{equation*}
    \sqrt{n}(\hat v_n(X_{1:n}) - v(P_{1/\sqrt{n}})) \to L
\end{equation*}
in distribution as $n \to \infty$, where $X_{1:n}$ are distributed i.i.d. according to $P_{1/\sqrt{n}}$.
\end{definition}

Note that this property holds even if $\{P_\epsilon\}$ is chosen adversarially in response to $\hat v_n$.

\subsection{Characterizations}

The following characterizes some important equivalences based on the above definitions. The following are based on \citet[Theorm 3.1]{van1991differentiable}. 

\begin{theorem}[Influence functions are gradients]
Suppose that $\hat v_n(X_{1:n})$ is an AL estimator of $v(P_0)$ with influence function $\psi_{P_0}(X)$, and that $v(P_0)$ is pathwise differentiable at $P_0$ wrt $\mathcal M$. Then $\hat v_n(X_{1:n})$ is a regular estimator of $v(P_0)$ at $P_0$ wrt $\mathcal M$ if and only if $\psi_{P_0}(X)$ is a gradient of $v(P_0)$ at $P_0$ wrt $\mathcal M$.
\end{theorem}

\begin{corollary}[Characterization of the EIF]
The EIF wrt $\mathcal M$ is the projection of any gradient wrt $\mathcal M$ onto the tangent space wrt $\mathcal M$.
\end{corollary}

\subsection{Strategy to calculate the EIF}

Given the above, the following is a natural strategy to calculate the EIF:
\begin{enumerate}
    \item Calculate a gradient $\psi_{P_0}(X)$ of the target parameter $v(P_0)$ wrt $\mathcal M$
    \item Calculate the gradient space wrt $\mathcal M$
    \item Either:
    \begin{enumerate}
        \item Show that $\psi_{P_0}(X)$ already lies in the above tangent space, or
        \item Project $\psi_{P_0}(X)$ onto the tangent space
    \end{enumerate}
\end{enumerate}

The first part of the above can often be done by explicitly computing the derivative of $v(P_\epsilon)$ wrt $\epsilon$, and re-arranging this into the form $\e[\psi_{P_0}(X) s(X)]$ for some function $\psi_{P_0}(X)$.

\subsection{Optimalities}

Finally, we describe the optimal properties of the EIF $\tilde \psi_{P_0}(X)$. We define the \emph{efficiency bound} as the variance of the EIF, $\textup{var}_{P_0}[\tilde \psi_{P_0}(X)]$, which has the following interpretations. First, the efficiency bound gives a lower-bound on the risk of any estimator in a local asymptotic minimax sense \citep[Theorem 25.20]{van2000asymptotic}.

\begin{theorem}[Local Asymptotic Minimax (LAM) theorem]
\label{thm:lam}
Let $v(P_0)$ be pathwise differentiable at $P_0$ wrt $\mathcal M$, with the EIF $\tilde \psi_{P_0}(X)$. Then, for any estimator sequence $\hat v_n(X_{1:n})$, and any symmetric quasi-convex loss function $l: \mathbb R \mapsto [0, \infty)$, we have
\begin{align*}
    &\sup_{m \in \mathbb N, \{P_\epsilon^{(1)}\}, \ldots, \{P_\epsilon^{(m)}\}} \lim_{n \to \infty} \sup_{k \in [m]} \e_{P_{1/\sqrt{n}}^{(k)}} \left[ l\left(\sqrt{n}\left\{\hat v_n(X_{1:n}) - v(P_{1/\sqrt{n}})\right\}\right) \right] \\
    &\geq \int l(u) d\mathcal N(0, \textup{var}_{P_0}[\tilde \psi_{P_0}(X)])\,,
\end{align*}
where $\{P_\epsilon^{(1)}\}, \ldots, \{P_\epsilon^{(m)}\}$ are one-dimensional submodels of $\mathcal M$ passing through $P_0$.
\end{theorem}

In other words, if we allow for adversarial local perturbations to the data generating process that are consistent with $\mathcal M$, then the  worst-case risk of \emph{any} estimator (not necessarily regular) is lower-bounded by that of a regular and asymptotic estimator whose influence function is the EIF. This interpretation follows because, given the above definition of regular estimators and the central limit theorem, the limiting distribution of such a regular and AL estimator is $\mathcal N(0, \textup{var}_{P_0}[\tilde \psi_{P_0}(X)])$ under any such local perturbations. Note that this theorem also implies the following, possibly easier-to-interpret corollary.
\begin{corollary}
Under the same assumptions as \cref{thm:lam}, we have
\begin{align*}
    &\inf_{\delta > 0} \liminf_{n \to \infty} \sup_{Q \in \mathcal M, d_{\textup{TV}}(Q,P_0) \leq \delta} \e_Q \left[ l\left(\sqrt{n}\left\{\hat v_n(X_{1:n}) - v(Q)\right\}\right) \right] \\
    &\geq \int l(u) d\mathcal N(0, \textup{var}_{P_0}[\tilde \psi_{P_0}(X)])\,,
\end{align*}
where $d_{\textup{TV}}(\cdot,\cdot)$ is the total variation distance, and $\mathcal N(\mu, \sigma^2)$ denotes a normal distribution with mean $\mu$ and variance $\sigma^2$.
\end{corollary}

Second, the efficiency bound gives a lower-bound on the risk of any regular estimator, in a strict non-minimax sense \citep[Theorem 25.21]{van2000asymptotic}.

\begin{theorem}[Convolution Theorem]
Let $l: \mathbb R \mapsto [0, \infty)$ be a symmetric quasi-convex loss function. Let $v(P_0)$ be pathwise differentiable at $P_0$ wrt $\mathcal M$ with EIF $\tilde \psi_{P_0}(X)$, and let $\hat v_n(X_{1:n})$ be a regular estimator sequence for $v(P_0)$ at $P_0$ wrt $\mathcal M$, with limiting distribution $L$. Then, we have
\begin{equation*}
    \int l(u) dL(u) \geq \int l(u) d\mathcal N(0, \textup{var}_{P_0}[\tilde \psi_{P_0}(X)]) \,.
\end{equation*}
\end{theorem}

Equality holds obviously when $L = \mathcal N(0, \textup{var}_{P_0}[\tilde \psi_{P_0}(X)])$, which as discussed above follows when $\hat v_n(X_{1:n})$ is regular and AL with influence function given by the EIF.

We note that in our interpretations of both the LAM and Convolution Theorems, we argued that if an estimator is regular and AL with influence function $\tilde \psi_{P_0}(X)$ then it will achieve the corresponding bound. The following final theorem shows that the latter property alone is both necessary and sufficient \citep[Theorem 25.23]{van2000asymptotic}.

\begin{theorem}
Let $v(P_0)$ be pathwise differentiable at $P_0$ wrt $\mathcal M$, and let $\tilde \psi_{P_0}(X)$ be the EIF. Then an estimator sequence is efficient (regular wrt $\mathcal M$ and with limiting distribution $\mathcal N(0, \textup{var}_{P_0}[\tilde \psi_{P_0}(X)])$) if and only if it is AL with influence function $\tilde \psi_{P_0}(X)$.
\end{theorem}

\section{Semiparametric Efficiency Theory for Proximal RL Estimator}
\label{apx:semiparametric-model}

Here, we detail the missing theory for our semiparaparametric efficiency theory in \cref{sec:efficiency}. In particular, we provide some additional minor lemmas that are needed to prove \cref{thm:efficiency}, and then end this appendix with the theorem's proof.

First, for any $u>0$ let us define the following, which is a set of random variables indexed by some $\epsilon$ satisfying a particular boundedness condition.
\begin{align}
    \fbounded(u) = \Bigg\{ f_\epsilon(\tau_H) &: \sup_{|\epsilon'| \leq u} \|f_{\epsilon'}(\tau_H)\|_\infty < \infty \nonumber \\
    \label{eq:fbounded}
    &\quad\text{and}\quad \sup_{|\epsilon'| \leq u} \left\| \frac{1}{\epsilon'} \left( f_{\epsilon'}(\tau_H) - f_0(\tau_H) \right) \right\|_{2,\mathcal P_0} < \infty \Bigg\} \,.
\end{align}
In particular, we would like to show that $q_\epsilon^{(t)}$ and $h_\epsilon^{(t)}$ belong to this set for some $u>0$. This is formalized by the following lemma.

\begin{lemma}
\label{lem:convergent-nuisances}
Let \cref{asm:semiparametric-reg} be given. Then there exists $u > 0$ such that $q_\epsilon^{(t)}(Z_t,A_t) \in \fbounded(u)$ and $h_\epsilon^{(t)}(W_t,A_t) \in \fbounded(u)$ for every $t \in [H]$.
\end{lemma}

This condition allows us to apply dominated convergence theorem arguments in computing the path derivative of $V(\mathcal P_\epsilon)$. In particular, we do not assume that $q_\epsilon^{(t)}$ or $h_\epsilon^{(t)}$ are differentiable w.r.t. $\epsilon$, so the second part of the definition of $\fbounded(u)$ lets us deal with finite-difference terms.

Before we prove \cref{lem:convergent-nuisances}, we must first establish some helper lemma. First, the following establishes that the set $\fbounded$ is closed w.r.t. addition and multiplication.

\begin{lemma}
\label{lem:boundedness-algebra}
Let $\fbounded(u)$ be defined as in \cref{eq:fbounded} for each, and suppose that $f_\epsilon(\tau_H) \in \fbounded(u)$ and $g_\epsilon(\tau_H) \in \fbounded(u)$ for some $u>0$. Then $f_\epsilon(\tau_H) + g_\epsilon(\tau_H) \in \fbounded(u)$ and $f_\epsilon(\tau_H) g_\epsilon(\tau_H) \in \fbounded(u)$.
\end{lemma}

\begin{proof}[Proof of \cref{lem:boundedness-algebra}]
Let $c_1,c_2,d_1,d_2 < \infty$ be constants such that
\begin{align*}
    \sup_{|\epsilon| \leq u} \|f_\epsilon(\tau_H)\|_\infty &\leq c_1 \qquad\qquad
    \sup_{|\epsilon| \leq u} \left\| \frac{1}{\epsilon} \left( f_\epsilon(\tau_H) - f_0(\tau_H) \right) \right\|_{2,\mathcal P_0} \leq d_1 \\
    \sup_{|\epsilon| \leq u} \|g_\epsilon(\tau_H)\|_\infty &\leq c_2 \qquad\qquad
    \sup_{|\epsilon| \leq u} \left\| \frac{1}{\epsilon} \left( g_\epsilon(\tau_H) - g_0(\tau_H) \right) \right\|_{2,\mathcal P_0} \leq d_2\,.
\end{align*}
First consider the case of $f_\epsilon(\tau_H) + g_\epsilon(\tau_H)$. It easily follows from the triangle inequality that
\begin{align*}
    &\sup_{|\epsilon| \leq u} \|f_\epsilon(\tau_H) + g_\epsilon(\tau_H)\|_\infty \leq c_1 + c_2 \\
    &\sup_{|\epsilon| \leq u} \left\| \frac{1}{\epsilon} \left( (f_\epsilon(\tau_H) + g_\epsilon(\tau_H)) - (f_0(\tau_H) + g_0(\tau_H)) \right) \right\|_{2, \mathcal P_0} \leq d_1 + d_2 \,,
\end{align*}
which clearly establishes that $f_\epsilon(\tau_H) + g_\epsilon(\tau_H) \in \fbounded(u)$.

Now, consider the case of $f_\epsilon(\tau_H) g_\epsilon(\tau_H)$. For the first required bound, we clearly have
\begin{equation*}
    \sup_{|\epsilon| \leq u} \|f_\epsilon(\tau_H) g_\epsilon(\tau_H)\|_\infty  \leq c_1 c_2 \,.
\end{equation*}
The second required bound requires slightly more work. There, we have
\begin{align*}
    &\sup_{|\epsilon| \leq u} \left\| \frac{1}{\epsilon} \left( f_\epsilon(\tau_H)g_\epsilon(\tau_H) - f_0(\tau_H)g_0(\tau_H) \right) \right\|_{2, \mathcal P_0} \\
    &\leq  \sup_{|\epsilon| \leq u} \left\| \frac{1}{\epsilon} \left( f_\epsilon(\tau_H)g_\epsilon(\tau_H) - f_0(\tau_H)g_\epsilon(\tau_H) \right) \right\|_{2, \mathcal P_0} \\
    &\qquad + \sup_{|\epsilon| \leq u} \left\| \frac{1}{\epsilon} \left( f_0(\tau_H)g_\epsilon(\tau_H) - f_0(\tau_H)g_0(\tau_H) \right) \right\|_{2, \mathcal P_0} \\
    &\leq  c_2 \sup_{|\epsilon| \leq u} \left\| \frac{1}{\epsilon} \left( f_\epsilon(\tau_H) - f_0(\tau_H) \right) \right\|_{2, \mathcal P_0} + c_1 \sup_{|\epsilon| \leq u} \left\| \frac{1}{\epsilon} \left( g_\epsilon(\tau_H) - g_0(\tau_H) \right) \right\|_{2, \mathcal P_0} \\
    &\leq  c_2 d_1 + c_1 d_2 \,.
\end{align*}
Therefore, we also have $f_\epsilon(\tau_H) g_\epsilon(\tau_H) \in \fbounded(u)$.
\end{proof}

Next, the following lemmas allow us to bound functions in the range of $(T_{t,\epsilon})^{-1}$ and $(T^*_{t,\epsilon})^{-1}$.

\begin{lemma}
\label{lem:t-inv-bound}
Under the conditions of \cref{thm:efficiency}, there exists some $u>0$ and constants $C_1,C_2<\infty$ such that for every $t \in [H]$, and functions $f_\epsilon^*(Z_t,A_t)$ and $g_\epsilon^*(W_t,A_t)$ indexed by $\epsilon$, we have
\begin{align*}
    \sup_{|\epsilon| \leq u} \|(T_{t,\epsilon})^{-1} g_\epsilon^*(W_t,A_t) \|_\infty &\leq C_1 \sup_{|\epsilon| \leq u} \| g_\epsilon^*(W_t,A_t) \|_\infty \\
    \sup_{|\epsilon| \leq u} \|(T^*_{t,\epsilon})^{-1} f_\epsilon^*(Z_t,A_t) \|_\infty &\leq C_2 \sup_{|\epsilon| \leq u} \| f_\epsilon^*(Z_t,A_t) \|_\infty \,.
\end{align*}
\end{lemma}

\begin{proof}[Proof of \cref{lem:t-inv-bound}]
For the first required bound, we have
\begin{align*}
    &\sup_{|\epsilon| \leq u} \|(T_{t,\epsilon})^{-1} g_\epsilon^*(W_t,A_t) \|_\infty \\
    &= \sup_{|\epsilon| \leq u, \|f(Z_t,A_t)\|_{1,\mathcal P_{\epsilon}} \leq 1} \e_\epsilon[f(Z_t,A_t) (T_{t,\epsilon})^{-1} g_\epsilon^*(W_t,A_t) ] \\
    &= \sup_{|\epsilon| \leq u, \|f(Z_t,A_t)\|_{1,\mathcal P_{\epsilon}} \leq 1} \e_\epsilon[g_\epsilon^*(W_t,A_t) (T^*_{t,\epsilon})^{-1}f(Z_t,A_t)] \\
    &\leq \sup_{|\epsilon| \leq u} \|g_\epsilon^*(W_t,A_t)\|_\infty \sup_{\|f(Z_t,A_t)\|_{1,\mathcal P_\epsilon} \leq 1} \|(T^*_{t,\epsilon})^{-1}f(Z_t,A_t)\|_{1,\mathcal P_\epsilon} \,.
\end{align*}
where the first equality follows from the fact that the $1$- and $\infty$-norms are dual, the second equality follows from the fact that the inverse of the adjoint is the adjoint of the inverse, and the inequality follows from Hölders inequality. Next, we can further bound
\begin{align*}
    \sup_{\|f(Z_t,A_t)\|_{1,\mathcal P_\epsilon} \leq 1} \|(T^*_{t,\epsilon})^{-1}f(Z_t,A_t)\|_{1,\mathcal P_\epsilon} &= \sup_{f(Z_t,A_t)} \frac{\|(T^*_{t,\epsilon})^{-1}f(Z_t,A_t)\|_{1,\mathcal P_\epsilon}}{\|f(Z_t,A_t)\|_{1,\mathcal P_\epsilon}} \\
    &= \sup_{g(W_t,A_t)} \frac{\|g(W_t,A_t)\|_{1,\mathcal P_\epsilon}}{\|T^*_{t,\epsilon} g(W_t,A_t)\|_{1,\mathcal P_\epsilon}} \\
    &= \left( \inf_{\|g(W_t,A_t)\|_{1,\mathcal P_\epsilon} \geq 1} \|T^*_{t,\epsilon} g(W_t,A_t)\|_{1,\mathcal P_\epsilon} \right)^{-1} \,.
\end{align*}
Now, we have from \cref{asm:semiparametric-reg} that
\begin{equation*}
    \liminf_{\epsilon \to 0} \inf_{\|g(W_t,A_t)\|_{1,\mathcal P_\epsilon} \geq 1} \|T^*_{t,\epsilon} g(W_t,A_t)\|_{1,\mathcal P_\epsilon} > 0 \,.
\end{equation*}
Therefore, the above must be bounded by some $C_1$, for sufficiently small $\epsilon$, which establishes the first required bound.

The second required bound follows from an almost-identical argument, except that here we instead apply the assumption that
\begin{equation*}
    \liminf_{\epsilon \to 0} \inf_{\|f(Z_t,A_t)\|_{1,\mathcal P_\epsilon} \geq 1} \|T_{t,\epsilon} f(Z_t,A_t)\|_{1,\mathcal P_\epsilon} > 0 \,.
\end{equation*}
\end{proof}

\begin{lemma}
\label{lem:t-inv-bound-2}
Under the conditions of \cref{thm:efficiency}, there exists some $u>0$ and constants $C^{(2)}_1,C^{(2)}_2<\infty$ such that for every $t \in [H]$, and functions $f_\epsilon^*(Z_t,A_t)$ and $g_\epsilon^*(W_t,A_t)$ indexed by $\epsilon$, we have
\begin{align*}
    \sup_{|\epsilon| \leq u} \|(T_{t,0})^{-1} g_\epsilon^*(W_t,A_t) \|_{2,\mathcal P_0} &\leq C^{(2)}_1 \sup_{|\epsilon| \leq u} \| g_\epsilon^*(W_t,A_t) \|_{2,\mathcal P_0} \\
    \sup_{|\epsilon| \leq u} \|(T^*_{t,0})^{-1} f_\epsilon^*(Z_t,A_t) \|_{2,\mathcal P_0} &\leq C^{(2)}_2 \sup_{|\epsilon| \leq u} \| f_\epsilon^*(Z_t,A_t) \|_{2,\mathcal P_0} \,.
\end{align*}
\end{lemma}

\begin{proof}[Proof of \cref{lem:t-inv-bound-2}]
The proof of this is very similar to that of \cref{lem:t-inv-bound}, except using $\mathcal P_0$ instead of $\mathcal P_\epsilon$. For the first required bound, we have
\begin{align*}
    &\sup_{|\epsilon| \leq u} \|(T_{t,0})^{-1} g_\epsilon^*(W_t,A_t) \|_{2,\mathcal P_0} \\
    &= \sup_{|\epsilon| \leq u, \|f(Z_t,A_t)\|_{2,\mathcal P_0} \leq 1} \e_0[f(Z_t,A_t) (T_{t,0})^{-1} g_\epsilon^*(W_t,A_t) ] \\
    &= \sup_{|\epsilon| \leq u, \|f(Z_t,A_t)\|_{2,\mathcal P_0} \leq 1} \e_0[g_\epsilon^*(W_t,A_t) (T^*_{t,0})^{-1}f(Z_t,A_t)] \\
    &\leq \sup_{|\epsilon| \leq u} \|g_\epsilon^*(W_t,A_t)\|_{2,\mathcal P_0} \sup_{\|f(Z_t,A_t)\|_{2,\mathcal P_0} \leq 1} \|(T^*_{t,0})^{-1}f(Z_t,A_t)\|_{2,\mathcal P_0} \,.
\end{align*}
where the first equality follows from the fact that the $2$-norm is self-dual, the second equality follows from the fact that the inverse of the adjoint is the adjoint of the inverse, and the inequality follows from Cauchy Schwartz. Next, we can further bound
\begin{align*}
    \sup_{\|f(Z_t,A_t)\|_{2,\mathcal P_0} \leq 1} \|(T^*_{t,0})^{-1}f(Z_t,A_t)\|_{2,\mathcal P_0} &= \sup_{f(Z_t,A_t)} \frac{\|(T^*_{t,0})^{-1}f(Z_t,A_t)\|_{2,\mathcal P_0}}{\|f(Z_t,A_t)\|_{2,\mathcal P_0}} \\
    &= \sup_{g(W_t,A_t)} \frac{\|g(W_t,A_t)\|_{2,\mathcal P_0}}{\|T^*_{t,0} g(W_t,A_t)\|_{2,\mathcal P_0}} \\
    &= \left( \inf_{\|g(W_t,A_t)\|_{2,\mathcal P_0} \geq 1} \|T^*_{t,0} g(W_t,A_t)\|_{2,\mathcal P_0} \right)^{-1} \,.
\end{align*}
But we know from \cref{asm:semiparametric-reg} that this term is finite, which gives us our first required bound.

The second required bound follows from an almost-identical argument, except that here we instead apply the fact that $\inf_{\|f(Z_t,A_t)\|_{2,\mathcal P_0} \geq 1} \|T_{t,0} f(Z_t,A_t)\|_{2,\mathcal P_0} > 0$, which also follows from \cref{asm:semiparametric-reg}.

\end{proof}

Now, given the previous lemma, as well as the definition of $\fbounded(u)$ and \cref{lem:boundedness-algebra}, the following lemmas establish that $q_\epsilon^{(t)}(Z_t,A_t)$ and $h_\epsilon^{(t)}(W_t,A_t)$ lie within this set for some $u>0$. This will along with \cref{lem:functional-deriv} will allow us to compute the derivative of $V(\mathcal P_\epsilon)$ in the main proof of \cref{thm:efficiency}.

\begin{lemma}
\label{lem:bounded-nuisances}
Under the conditions of \cref{thm:efficiency}, there exists some $u>0$ such that for every $t \in [H]$ we have $\sup_{|\epsilon| \leq u} \|q_\epsilon^{(t)}(Z_t,A_t)\|_\infty < \infty$ and $\sup_{|\epsilon| \leq u} \|h_\epsilon^{(t)}(Z_t,A_t)\|_\infty < \infty$.
\end{lemma}

\begin{proof}[Proof of \cref{lem:bounded-nuisances}]
First consider the case of $q^{(t)}_\epsilon$. First, by \cref{asm:semiparametric-reg} it must be the case that for some $u>0$ and some constant $C_3 < \infty$, we have
\begin{equation*}
    \sup_{|\epsilon| \leq u} \| P^*_{t,\epsilon}(A_t \mid W_t)^{-1} \|_\infty \leq C_3 \,,
\end{equation*}
simultaneously for every $t \in [H]$. Now, let $u>0$ be such that the above holds, as well as as \cref{lem:t-inv-bound}, and let $C_1$ and $C_2$ be defined as in \cref{lem:t-inv-bound}.
Note that clearly for $|\epsilon| \leq u$ we have $\mathcal P_\epsilon \in \mpci$, so let us restrict our attention to such $\epsilon$. 
Next, note that for any $|\epsilon| \leq u$ we have 
\begin{equation*}
    q^{(t)}_\epsilon(Z_t,A_t) = T^{-1}_{t,\epsilon} P^*_{t,\epsilon}(A_t \mid W_t)^{-1}\,,
\end{equation*}
and so given \cref{lem:t-inv-bound} we have
\begin{equation*}
    \|q^{(t)}_\epsilon(Z_t,A_t)\|_\infty \leq C_2 C_3\,.
\end{equation*}
We note that this bound does not depend on $\epsilon$, so therefore we have our desired result for $q^{(t)}_\epsilon$.

Next, for $h_\epsilon^{(t)}$, we note that
\begin{equation*}
    q^{(t)}_\epsilon(Z_t,A_t) = (T^*_{t,\epsilon})^{-1} \mu_{t,\epsilon}(Z_t,A_t)\,.
\end{equation*}
Therefore, we can proceed with a near-identical argument again using \cref{lem:t-inv-bound}, as long as we can uniformly bound $\| \mu_{t,\epsilon}(Z_t,A_t) \|_\infty$ over $|\epsilon| \leq u$ for each $t$.  
First, note that given the above bound for $q^{(t)}_\epsilon$ and the definition of $\mu_{t,\epsilon}$, for each $t \in [H]$ we have
\begin{equation*}
    \|\mu_{t,\epsilon}(Z_t,A_t)\|_\infty \leq (C_2C_3)^{t-1} \|Y_{t,\epsilon}\|_\infty \,.
\end{equation*}
Therefore it is sufficient to uniformly bound $Y_{t,\epsilon}$ for each $t$. We will do this by backward induction on $t$. For the base case where $t=H$, we trivially have $\|Y_{H,\epsilon}\|_\infty \leq  R_{\textup{max}}$, where $R_{\textup{max}}$ is a bound on the absolute value of all rewards. Now, suppose that $\|Y_{t+1,\epsilon}\|_\infty \leq C_{4,t+1}$ for some $C_{4,t+1}$ that doesn't depend on $\epsilon$. Then we have
\begin{equation*}
    \|Y_{t,\epsilon}\|_\infty \leq R_{\textup{max}} + \gamma (|\aset| + (C_3 C_2)) \|h_\epsilon^{(t+1)}(W_{t+1},A_{t+1})\|_\infty  + \gamma (C_3 C_2) \|Y_{t+1,\epsilon}\|_\infty \\
\end{equation*}
Now, by the inductive hypothesis $Y_{t+1,\epsilon}$ is uniformly bounded and therefore so is $\mu_{t+1,\epsilon}(Z_{t+1},A_{t+1})$, and therefore following an identical argument as above as for bounding $q_\epsilon^{(t)}$ we have that $\|h_\epsilon^{(t+1)}(W_{t+1},A_{t+1})\|_\infty$ is uniformly bounded. Specifically, by \cref{lem:t-inv-bound} we have
\begin{align*}
    \|\mu_{t+1,\epsilon}(Z_{t+1},A_{t+1})\|_\infty &\leq (C_3 C_2)^t C_{4,t+1} \\
    \|h_\epsilon^{(t+1)}(W_{t+1},A_{t+1})\|_\infty &\leq C_1 (C_3 C_2)^t C_{4,t+1} \,,
\end{align*}
and therefore $Y_{t,\epsilon}$ is uniformly bounded. This completes the induction, so therefore we conclude that we have a uniform bound on $\|h_\epsilon^{(t)}(W_t,A_t)\|_\infty$ over $|\epsilon| \leq u$.

\end{proof}

Now, we are finally ready to provide the proof of \cref{lem:convergent-nuisances}. 

\begin{proof}[Proof of \cref{lem:convergent-nuisances}]

Let $u$, $C_1$, $C_2$, and $C_3$ be defined as in the proof of \cref{lem:bounded-nuisances}, and define
\begin{align*}
    \delta^{(q,t)}_\epsilon(Z_t,A_t) &= \frac{1}{\epsilon} \left( q_\epsilon^{(t)}(Z_t,A_t) - q_0^{(t)}(Z_t,A_t) \right) \\
    \delta^{(h,t)}_\epsilon(W_t,A_t) &= \frac{1}{\epsilon} \left( h_\epsilon^{(t)}(W_t,A_t) - h_0^{(t)}(W_t,A_t) \right) \,.
\end{align*}
Now, given \cref{lem:bounded-nuisances} it is sufficient to prove that $\delta^{(q,t)}_\epsilon(Z_t,A_t)$ and $\delta^{(h,t)}_\epsilon(W_t,A_t)$ are uniformly bounded over $|\epsilon| \leq u$, in $2$-norm. We will prove this for each $\delta^{(q,t)}_\epsilon(Z_t,A_t)$ term via forward induction on $t$, and then for each $\delta^{(h,t)}_\epsilon(W_t,A_t)$ term via backward induction on $t$. Furthermore, noticing that $\delta^{(q,t)}_\epsilon(Z_t,A_t) = (T_{t,0})^{-1} T_{t,0} \delta^{(q,t)}_\epsilon(Z_t,A_t)$ and $\delta^{(h,t)}_\epsilon(W_t,A_t) = (T^*_{0,\epsilon})^{-1} T^*_{0,\epsilon} \delta^{(h,t)}_\epsilon(W_t,A_t)$, for each $t \in [H]$, it follows from \cref{lem:t-inv-bound-2} that it is sufficient to show
\begin{equation*}
    \sup_{|\epsilon| \leq u} \left\| T_{t,0} \delta^{(q,t)}_\epsilon(Z_t,A_t) \right\|_{2,\mathcal P_0} < \infty \quad  \text{and} \quad \sup_{|\epsilon| \leq u} \left\| T^*_{t,0} \delta^{(h,t)}_\epsilon(W_t,A_t) \right\|_{2,\mathcal P_0} < \infty \,.
\end{equation*}

First, we can show that show that
\begin{align}
    T_{t,0} \delta^{(q,t)}_\epsilon(Z_t,A_t) &= \frac{1}{\epsilon} \left( T_{t,\epsilon} q^{(t)}_\epsilon(Z_t,A_t) - T_{t,0} q^{(t)}_0(Z_t,A_t) \right) - \frac{1}{\epsilon} (T_{t,\epsilon} - T_{t,0}) q_\epsilon^{(t)}(Z_t,A_t) \nonumber \\
    &= \frac{1}{\epsilon} \left( P^*_{t,\epsilon}(A_t \mid W_t)^{-1} - P^*_{t,0}(A_t \mid W_t)^{-1} \right) \nonumber \\
    &\qquad - \frac{1}{\epsilon} \left( \e_\epsilon[\eta_{t,\epsilon} q_\epsilon^{(t)}(Z_t,A_t) \mid W_t,A_t] - \e_0[\eta_{t,0} q_\epsilon^{(t)}(Z_t,A_t) \mid W_t,A_t] \right) \nonumber \\
    &= \frac{1}{\epsilon} \left( P^*_{t,\epsilon}(A_t \mid W_t)^{-1} - P^*_{t,0}(A_t \mid W_t)^{-1} \right) \nonumber \\
    &\qquad - \frac{1}{\epsilon} \left( \e_\epsilon - \e_0 \right)[\eta_{t,0} q_\epsilon^{(t)}(Z_t,A_t) \mid W_t,A_t] \nonumber \\
    \label{eq:tq-terms}
    &\qquad - \e_\epsilon \left[ \frac{1}{\epsilon} \left( \eta_{t,\epsilon} - \eta_{t,0} \right) q_\epsilon^{(t)}(Z_t,A_t) \mid W_t,A_t \right]
\end{align}

We will proceed by bounding the three terms in \cref{eq:tq-terms} one by one. for the first, we have
\begin{align*}
    &\frac{1}{\epsilon} \left( P^*_{t,\epsilon}(A_t \mid W_t)^{-1} - P^*_{t,0}(A_t \mid W_t)^{-1} \right) \\
    &= \left( \frac{P^*_{t,\epsilon}(A_t \mid W_t)^{-1} - P^*_{t,0}(A_t \mid W_t)^{-1}}{P^*_{t,\epsilon}(A_t \mid W_t) - P^*_{t,0}(A_t \mid W_t)} \right) \left( \frac{P^*_{t,\epsilon}(A_t \mid W_t) - P^*_{t,0}(A_t \mid W_t)}{\epsilon} \right) \\
    &= -(aP^*_{t,\epsilon}(A_t \mid W_t) + (1-a)P^*_{t,0}(A_t \mid W_t) )^{-2} \left( \frac{P^*_{t,\epsilon}(A_t \mid W_t) - P^*_{t,0}(A_t \mid W_t)}{\epsilon} \right) \,,
\end{align*}
where the second equation follows from the mean value theorem, for some $a \in (0,1)$ that may depend on $t$, $A_t$, and $W_t$. Now, clearly $\| (aP^*_{t,\epsilon}(A_t \mid W_t) + (1-a)P^*_{t,0}(A_t \mid W_t) )^{-2}\|_\infty \leq C_3^2$. Furthermore, we have
\begin{equation*}
    P^*_{t,\epsilon}(A_t \mid W_t) = \e_\epsilon[\eta_{t,\epsilon} \mid W_t, A_t] P_\epsilon(A_t \mid W_t) \,.
\end{equation*}
In addition, applying the mean value theorem again we have
\begin{align*}
    \left\| \frac{1}{\epsilon} \left( P_\epsilon(A_t \mid W_t) - P_0(A_t \mid W_t) \right) \right\|_{2,\mathcal P_0} &= \| P_{\epsilon'}(A_t \mid W_t) s_{\epsilon'}(A_t \mid W_t) \|_{2,\mathcal P_0} \\
    &\leq \| s_{\epsilon'}(A_t, W_t) \|_{2,\mathcal P_0} + \| s_{\epsilon'}(A_t) \|_{2,\mathcal P_0} \,,
\end{align*}
for some $\epsilon' \in (0, 1)$. Therefore, since by our definition of score functions in \cref{apx:semiparametric} score functions have uniformly bounded euclidean norm, it follows that $P_\epsilon(A_t \mid W_t) \in \fbounded(u)$.

Now, in the base case where $t=1$ we have $P^*_{t,\epsilon}(A_t \mid W_t) = P_\epsilon(A_t \mid W_t)$, and otherwise in the inductive case where $t>1$, by the inductive assumption and \cref{lem:boundedness-algebra} it easily follows that $\e_\epsilon[\eta_{t,\epsilon} \mid W_t, A_t] \in \fbounded(u)$, and therefore by \cref{lem:boundedness-algebra} so is $P^*_{t,\epsilon}(A_t \mid W_t)$. In either case, this implies the required bound for the first term of \cref{eq:tq-terms}.

For the second term of \cref{eq:tq-terms}, by the mean value theorem we have
\begin{align*}
    &\frac{1}{\epsilon} \left( \e_\epsilon - \e_0 \right)[\eta_{t,0} q_0^{(t)}(Z_t,A_t) \mid W_t,A_t] \\
    &= \e_{\epsilon'}[s_{\epsilon'}(\tau_{t-1}, Z_t \mid W_t,A_t) \eta_{t,0} q_0^{(t)}(Z_t,A_t) \mid W_t,A_t] \\
    &= \e_{\epsilon'}[s_{\epsilon'}(\tau_{t-1},Z_t,W_t,A_t) \eta_{t,0} q_0^{(t)}(Z_t,A_t) \mid W_t,A_t] \\
    &\qquad - \e_{\epsilon'}[s_{\epsilon'}(W_t,A_t) \eta_{t,0} q_0^{(t)}(Z_t,A_t) \mid W_t,A_t]\,,
\end{align*}
where $\epsilon' \in (0,\epsilon)$ and may be measurable \wrt the other random variables inside the expectation. Now, by \cref{lem:bounded-nuisances} we have $\|\eta_{t,0} q_0^{(t)}(Z_t,A_t)\|_\infty < \infty$. Furthermore, by our definition of score functions in parametric submodels in \cref{apx:semiparametric} we have $\sup_{|\epsilon| \leq u} \| s_\epsilon(\tau_{t-1},Z_t,W_t,A_t) \|_{2,\mathcal P_0} < \infty$ and $\sup_{|\epsilon| \leq u} \| s_\epsilon(W_t,A_t) \|_{2,\mathcal P_0} < \infty$.

Finally, for the third term of \cref{eq:tq-terms}, in the base case that $t=1$ we have $\eta_{t,\epsilon} = \eta_{t,0} = 1$, so this term vanishes. Otherwise, in the inductive case where $t>1$, by the inductive assumption and \cref{lem:boundedness-algebra} we know that $\frac{1}{\epsilon}(\eta_{t,\epsilon}-\eta_{0,\epsilon}) \in \fbounded(u)$. Therefore, combined with the fact that $q_0^{(t)}(Z_t,A_t)$ is bounded, it follows that this third term is uniformly bounded over $|\epsilon| \leq u$.

Therefore, we have completed the process of forward induction and proved that $q_\epsilon^{(t)}(Z_t,A_t) \in \fbounded(u)$ for all $t \in [H]$. Now, we proceed to the backward induction to prove that $h_\epsilon^{(t)}(W_t,A_t) \in \fbounded(u)$ for all $t \in [H]$. Proceeding similarly as above, we have
\begin{align*}
    T^*_{t,\epsilon} \delta^{(h,t)}_\epsilon(W_t,A_t) &= \frac{1}{\epsilon} \left( \indicator{E_t=A_t}Y_{t,\epsilon} - \indicator{E_t=A_t}Y_{t,0}  \right) \\
    &\qquad - \frac{1}{\epsilon} \left( \e_\epsilon - \e_0 \right)[\eta_{t,0} h_\epsilon^{(t)}(W_t,A_t) \mid Z_t,A_t] \\
    &\qquad - \e_\epsilon \left[ \frac{1}{\epsilon} \left( \eta_{t,\epsilon} - \eta_{t,0} \right) h_\epsilon^{(t)}(W_t,A_t) \mid Z_t,A_t \right]
\end{align*}
The second two terms can be bounded in infinity norm uniformly over $|\epsilon| \leq u$ following an identical argument as for the second two terms in \cref{eq:tq-terms}, so we only need to bound the first term. In the base case, where $t=H$, we have $Y_{H,\epsilon} = Y_{H,0} = R_H$, so this first term simply disappears. Otherwise, in the inductive case where $t<H$, we note that $Y_{t,\epsilon} = R_t + \gamma \Omega_{t,\epsilon}$, where $\Omega_{t,\epsilon}$ is defined in terms of addition and multiplication of terms of the kind $R_{t'}$, $\indicator{E_{t'}=A_{t'}}$, $q_\epsilon^{(t')}(Z_{t'},A_{t'})$, and $h_\epsilon^{(t')}(W_{t'},A_{t'})$, for $t' > t$. Therefore, noting that rewards by assumption are bounded, it follows from the inductive assumption and \cref{lem:boundedness-algebra} that $Y_{t,\epsilon} \in \fbounded(u)$, and therefore this first term is bounded.

This completes the backward induction, and establishes that $h_\epsilon^{(t)}(W_t,A_t) \in \fbounded(u)$ for all $t \in [H]$, so therefore we can conclude.

\end{proof}

\subsection{Discussion of Issues with Tangent Spaces in Past Work}
\label{apx:tangent-space}

Here we will discuss the problems with tagnent spaces proposed in past work on proximal causal inference. Given that this past work has considered the simpler setting where $H=1$, we will omit all suffixes and prefixes involving $t$ in the discussion here. Let $T : L_2(Z,A) \mapsto L_2(W,A)$ be the conditional operator defined according to
\begin{equation*}
    T f(Z,A) = \e[f(Z,A) \mid W,A] \quad \forall f \,,
\end{equation*}
whose adjoint $T^* : L_2(W,A) \mapsto L_2(Z,A)$ satisfies
\begin{equation*}
    T^* g(W,A) = \e[g(W,A) \mid Z,A] \quad \forall g \,.
\end{equation*}

In \citet{cui2020semiparametric}, the authors propose to use the tangent space, which, in terms of our notation and definitions of $q$ and $h$, is defined by the restrictions
\begin{align*}
    \e[ q(Z,A)) (s(A \mid W) + s(Z \mid W,A)) \mid W,A] &\in \text{Range}(T) \\
    \e[(\indicator{E=A} R - h(W,A)) s(W,R \mid Z,A) \mid Z,A] &\in \text{Range}(T^*) \,.
\end{align*}
However, this choice of tangent space is never fully justified in terms of the model under consideration. In \citet{kallus2021causal}, the authors do justify the necessity of these restrictions by noting that if $q_\epsilon$ and $h_\epsilon$ are differentiable with respect to $\epsilon$ within a given submodel, then we must have
\begin{align*}
    &\e\left[ \left. \frac{\partial}{\partial\epsilon} \right|_{\epsilon=0} q_\epsilon(Z,A) \mid W,A \right] \\
    &= \left. \frac{\partial}{\partial\epsilon} \right|_{\epsilon=0} P_\epsilon(A \mid W)^{-1} - \e[s(Z \mid W,A) q(Z,A) \mid W,A] \\
    &= -P(A \mid W)^{-1} s(A \mid W) - \e[s(Z \mid W,A) q(Z,A) \mid W,A] \\
    &= - \e[(s(A \mid W) + s(Z \mid W,A)) q(Z,A) \mid W,A] \,,
\end{align*}
and
\begin{equation*}
    \e\left[ \left. \frac{\partial}{\partial\epsilon} \right|_{\epsilon=0} h_\epsilon(W,A) \mid Z,A \right] = \e[s(W,R \mid Z,A) (\indicator{E=A}R - h(W,A)) \mid Z,A] \,.
\end{equation*}

Unfortunately, there are still some problems in this choice of tangent space. Firstly, although they are clearly necessary conditions for differentiability of the nuisances, it is not clear that they are \emph{sufficient} conditions; that is, it is not clear that for a given score function satisfying these conditions we can actually construct a parametric submodel for which the nuisances are defined and differentiable. Note that this is contrast to many other areas of work involving semiparametric efficiency theory, where the tangent set restrictions simply correspond to some conditional independence assumptions, in which case it is trivial to see that the tangent set restrictions invoked are both necessary and sufficient, since the partitioning of the score function immediately implies the independence structure of corresponding parametric submodels.

Secondly, it is not clear that diferentiability of the nuisances is even necessary -- indeed we showed how to prove that $\psidr(\tau_H)$ is a gradient of the policy value without ever assuming or requiring that the nuisance functions were differentiable -- nor is it clear what impact if any this requirement of nuisance differentiability would have on the actual model of interest.

Thirdly, \citet{kallus2021causal} consider a more general model in which $h$ and $q$ are not necessarily uniquely determined, in which case the above restrictions would actually have to hold for \emph{all} valid $h$ and $q$ functions, and it is not immediately clear that requiring this restriction for a single chosen $h$ and $q$ is sufficient.

Finally, under a model in which the allowed distributions all actually correspond to observational distributions for latent variable models with hidden confounders satisfying the PCI, which the past work implies are the only kinds of distributions under consideration, there are additional necessary restrictions on the score functions. For example, let $L = (Z,A)$, and $Q = (W,R)$, then from the PCI independence assumptions is clear that the observed distribution must take the form
\begin{equation*}
    P(L,Q) = \int P(S) P(L \mid S) P(Q \mid S) d \mu(S)\,,
\end{equation*}
for some latent variable $S$. It is easy to show that this implies that for any differentiable submodel on the full data $(L,Q,S)$ we have
\begin{align*}
    s(L,Q) &= \frac{\int \partial (s(S) + s(L \mid S) + s(Q \mid S)) P(S) P(L \mid S) P(Q \mid S) d \mu(S)}{\int P(S) P(L \mid S) P(Q \mid S) d \mu(S)} \\
    &= \int \partial (s(S) + s(L \mid S) + s(Q \mid S)) P(S \mid L,Q) d \mu(S) \\
    &= \e[s(S) + s(L \mid S) + s(Q \mid S) \mid L,Q]\,.
\end{align*}
Therefore, there must exist functions $f_1$, $f_2$, and $f_3$ such that
\begin{equation*}
    s(Z,A,W,R) = \e[f_1(S) + f_2(Z,A;S) + f_3(W,R;S) \mid Z,A,W,R]\,,
\end{equation*}
which satisfy
\begin{equation*}
    \e[f_1(S)] = \e[f_2(Z,A;S) \mid S] = \e[f_3(W,R;S) \mid S] = 0\,.
\end{equation*}
It is not clear that the previously proposed tangent spaces ensure this condition, for example.

Given these above issues, we took care to define assumptions to avoid such issues, by ensuring that we consider a model that is locally saturated at $\pb$, which guarantees that the tangent set is all square integrable functions. Achieving this involves ensuring that the nuisances are uniquely determined locally near $\pb$, and defining the parameter of interest is not defined in terms of the actual policy value, and rather in terms of the nuisances and the identification quantity; that is, we ensure that the parameter of interest corresponds to the target policy value for distributions that actually come from an underlying valid PCI model satisfying our assumptions, and otherwise is still an unambiguous and well-defined quantity as long as the nuisances are uniquely defined.

\subsection{Proof of Semiparametric Efficiency Theorem}

Before we present the main proof we present the following lemma based on the dominated convergence theorem, which we will apply heavily.

\begin{lemma}
\label{lem:functional-deriv}
Let $\mathcal P_\epsilon$ be a parametric submodel, and suppose that function $f_\epsilon(\tau_H)$ indexed by $\epsilon$ converges point-wise to $f(\tau_H)$ as $\epsilon \to 0$. Suppose in addition that it is uniformly bounded for small $\epsilon$; that is, $\limsup_{\epsilon \to 0} \|f_\epsilon(\tau_H)\|_\infty < \infty$. Then, we have
\begin{equation*}
    \lim_{\epsilon \to 0} \frac{1}{\epsilon} \left( \e_\epsilon[f_\epsilon(\tau_H)] - \e_0[f_\epsilon(\tau_H)] \right) = \e_0[s(\tau_H) f(\tau_H)]
\end{equation*}
\end{lemma}

\begin{proof}[Proof of \cref{lem:functional-deriv}]
Let $\mu(\tau_H)$ be a dominating measure for all measures in the parametric submodel, and let $p_\epsilon(\tau_H) = (d \mathcal P_\epsilon / d \mu)(\tau_H)$. Then, we have
\begin{align*}
    \frac{1}{\epsilon} \left( \e_\epsilon[f_\epsilon(\tau_H)] - \e_0[f_\epsilon(\tau_H)] \right) &= \int \frac{1}{\epsilon} \left( p_\epsilon(\tau_H) - p_0(\tau_H) \right) f_\epsilon(\tau_H) d \mu(\tau_H) \\
    &= \int p_{\epsilon'}(\tau_H) s_{\epsilon'}(\tau_H) f_\epsilon(\tau_H) d \mu(\tau_H) \,,
\end{align*}
where in the second equality we apply the mean value theorem, and $\epsilon' \in (0, \epsilon)$. Then, given the boundedness condition on $p_{\epsilon'}(\tau_H) s_{\epsilon'}(\tau_H)$ for parametric submodels assumed in \cref{apx:semiparametric}, as well as the uniform boundedness and point-wise convergence assumed on $f_\epsilon$, applying the dominated convergence theorem gives us
\begin{align*}
    \lim_{\epsilon \to 0} \frac{1}{\epsilon} \left( \e_\epsilon[f_\epsilon(\tau_H)] - \e_0[f_\epsilon(\tau_H)] \right) &= \int p_0(\tau_H) s_0(\tau_H) f_0(\tau_H) d \mu(\tau_H) \\
    &= \e_0[s(\tau_H) f(\tau_H)]\,,
\end{align*}
as required.
\end{proof}

\begin{proof}[Proof of \cref{thm:efficiency}]

Following \cref{apx:semiparametric}, in order to prove this theorem we need to: (1) derive the tangent space of $\mpci$ at $\pb$; (2) justify that $\psidr$ lies within this tangent space; and (3) show that $\psidr$ is a gradient of $v_\gamma(\pi_e)$. For the first two parts, we note that by the conditions of \cref{thm:efficiency} it easily follows that any parametric submodel $\mathcal P_\epsilon$ passing through $\pb$ at $\epsilon=0$ must lie within $\mpci$ for sufficiently small $\epsilon$, so therefore the tangent space is simply the set of all square integrable functions. Given this, $\psidr$ clearly lies within this tangent space. Therefore, we only need to justify that it is a gradient. That is, we need to show that for every parametric submodel
\begin{equation*}
    \left. \frac{\partial V(\mathcal P_\epsilon)}{\partial \epsilon} \right|_{\epsilon=0} = \eb[s(\tau_H) (\psidr(\tau_H) - v_\gamma(\pi_e))]\,,
\end{equation*}
where $s(\tau_H)$ is the score function of the parametric submodel. Note that we have \emph{not} assumed that the nuiances $q^{(t)}_\epsilon$ and $h^{(t)}_\epsilon$ are differentiable, so we must proceed with caution.

First, note that
\begin{align*}
    &\frac{1}{\epsilon} (V(\mathcal P_\epsilon) - V(\pb)) \\
    &= \frac{1}{\epsilon} \e_\epsilon \left[ \sum_{a \in \aset} h^{(1)}_\epsilon(W_1,a) \right] - \frac{1}{\epsilon} \e_0 \left[ \sum_{a \in \aset} h^{(1)}_0(W_1,a) \right] \\
    &= \frac{1}{\epsilon} (\e_\epsilon - \e_0) \left[ \sum_{a \in \aset} h^{(1)}_\epsilon(W_1,a) \right] + \e_0 \left[ \sum_{a \in \aset} \frac{1}{\epsilon} (h^{(1)}_\epsilon(W_1,a) - h^{(1)}_0(W_1,a)) \right] \,.
\end{align*}
Now, since $h_\epsilon^{(1)}(W_1,A_1) \in \fbounded(u)$ for sufficiently small $u$, it trivially follows that $h_\epsilon^{(1)}(w,a)$ converges to $h_0^{(1)}(w,a)$ for every $w,a$. Combining this fact along with \cref{lem:bounded-nuisances} and \cref{lem:functional-deriv}, this implies that
the first term above converges to $\e_0[s(\tau_H) \sum_{a \in \aset} h^{(1)}_\epsilon(W_1, a)]$. For the second term, we have
\begin{align*}
    &\e_0 \left[ \sum_{a \in \aset} \frac{1}{\epsilon} (h^{(1)}_\epsilon(W_1,a) - h^{(1)}_0(W_1,a)) \right] \\
    &= \e_0 \left[ P_0(A_1 \mid W_1)^{-1} \frac{1}{\epsilon} (h^{(1)}_\epsilon(W_1,A_1) - h^{(1)}_0(W_1,A_1)) \right] \\
    &= \e_0 \left[ q_0^{(1)}(Z_1,A_1) \frac{1}{\epsilon} (h^{(1)}_\epsilon(W_1,A_1) - h^{(1)}_0(W_1,A_1)) \right] \\
    &= \frac{1}{\epsilon} (\e_0 - \e_\epsilon) \left[ q_0^{(1)}(Z_1,A_1) h^{(1)}_\epsilon(W_1,A_1) \right] \\
    &\qquad + \frac{1}{\epsilon} \left( \e_\epsilon \left[ q^{(1)}_0(Z_1,A_1) h^{(1)}_\epsilon(W_1,A_1) \right] - \e_0 \left[ q^{(1)}_0(Z_1,A_1) h^{(1)}_0(W_1,A_1) \right] \right) \\
    &= \frac{1}{\epsilon} (\e_0 - \e_\epsilon) \left[ q_0^{(1)}(Z_1,A_1) h^{(1)}_\epsilon(W_1,A_1) \right] \\
    &\qquad + \frac{1}{\epsilon} \left( \e_\epsilon \left[ q^{(1)}_0(Z_1,A_1) \indicator{E_1=A_1} Y_{1,\epsilon} \right] - \e_0 \left[ q^{(1)}_0(Z_1,A_1) \indicator{E_1=A_1} Y_{1,0} \right] \right) \\
    &= \frac{1}{\epsilon} (\e_\epsilon - \e_0) \left[ q_0^{(1)}(Z_1,A_1) \left( \indicator{E_1=A_1} Y_{1,\epsilon} - h^{(1)}_\epsilon(W_1,A_1) \right) \right] \\
    &\qquad + \e_0 \left[ q^{(1)}_0(Z_1,A_1) \indicator{E_1=A_1} \frac{1}{\epsilon} \left( Y_{1,\epsilon} - Y_{1,0} \right) \right] \,.
\end{align*}
Now, again applying \cref{lem:functional-deriv,lem:convergent-nuisances} the first term above converges to
\begin{equation*}
    \e_0[s(\tau_H) q_0^{(1)}(Z_1,A_1) ( \indicator{E_1=A_1} Y_{1,\epsilon} - h^{(1)}_\epsilon(W_1,A_1))]\,.
\end{equation*} 
Combining this with the previous result and the fact that $s(\tau_H)$ has mean zero, as well as \cref{lem:phi-is}, we get
\begin{equation*}
    \left. \frac{\partial V(\mathcal P_\epsilon)}{\partial \epsilon} \right|_{\epsilon=0} = \e_0[s(\tau_H)(\psidr(\tau_H) - v_\gamma(\tau_e))] + \lim_{\epsilon \to 0} \es{2} \left[ \frac{1}{\epsilon} \left( Y_{1,\epsilon} - Y_{1,0} \right) \right].
\end{equation*}
Therefore, all that remains to show is that the second term above vanishes. We will argue this by backward induction, by showing that for all $t \leq H$ we have
\begin{equation}
\label{eq:gradient-induction}
    \lim_{\epsilon \to 0} \es{t+1} \left[ \frac{1}{\epsilon} \left( Y_{t,\epsilon} - Y_{t,0} \right) \right] = 0\,.
\end{equation}

First, for the base case $t=H$, this is trivial, since $Y_{H,\epsilon} = R_H$ for all $\epsilon$. Next, suppose that \cref{eq:gradient-induction} holds for all $t \geq s$, for some $s \leq H$. We will argue that it also holds for $t=s-1$. Specifically, plugging in the definitions of $Y_{s,\epsilon}$ and $Y_{s-1,0}$, we have
\begin{align*}
    \es{s} \left[ \frac{1}{\epsilon} \left( Y_{s-1,\epsilon} - Y_{s-1,0} \right) \right] &=  \gamma \es{s} \left[ \frac{1}{\epsilon} \sum_{a \in \aset} \left( h^{(s)}_\epsilon(W_s, a) - h^{(s)}_0(W_{s}, a) \right) \right] \\
    &\qquad + \gamma \es{s} \left[ \frac{1}{\epsilon} \left( q^{(s)}_\epsilon(Z_s,A_s) \Delta_{s,\epsilon} - q^{(s)}_0(Z_s,A_s) \Delta_{s,0} \right) \right] \\
    &=  \gamma \es{s} \left[ q_0^{(s)}(Z_s,A_s) \frac{1}{\epsilon} \left( h^{(s)}_\epsilon(W_s, A_s) - h^{(s)}_0(W_{s}, A_s) \right) \right] \\
    &\qquad + \gamma \es{s} \left[ \frac{1}{\epsilon} \left( q^{(s)}_\epsilon(Z_s,A_s) \Delta_{s,\epsilon} - q^{(s)}_0(Z_s,A_s) \Delta_{s,0} \right) \right] \\
    &= \gamma \es{s} \left[ \frac{1}{\epsilon} \left( q^{(s)}_\epsilon(Z_s,A_s) - q^{(s)}_0(Z_s,A_s) \right) \Delta_{s,\epsilon} \right] \\
    &\qquad + \gamma \es{s} \left[ q^{(s)}_0(Z_s,A_s) \indicator{E_s=A_s} \frac{1}{\epsilon} (Y_{s,\epsilon} - Y_{s,0}) \right]\,,
\end{align*}
where $\Delta_{s,\epsilon} = \indicator{E_s=A_s} Y_{s,\epsilon} - h^{(s)}(W_s,A_s)$.
Now, by \cref{lem:phi-is} the second term is equal to $\es{s+1}[ \frac{1}{\epsilon} (Y_{s,\epsilon} - Y_{s,0}) ]$, which converges to zero by the inductive hypothesis. Furthermore, since $\es{s}[\Delta_{s,0} \mid Z_s,A_s] = 0$, we can derive
\begin{align*}
    &\es{s} \left[ \frac{1}{\epsilon} \left( q^{(s)}_\epsilon(Z_s,A_s) - q^{(s)}_0(Z_s,A_s) \right) \Delta_{s,\epsilon} \right] \\
    &= \es{s} \left[ \frac{1}{\epsilon} \left( q^{(s)}_\epsilon(Z_s,A_s) - q^{(s)}_0(Z_s,A_s) \right) (\Delta_{s,\epsilon} - \Delta_{s,0}) \right] \\
    &= \e_0 \left[ \eta_{t,0} \frac{1}{\epsilon} \left( q^{(s)}_\epsilon(Z_s,A_s) - q^{(s)}_0(Z_s,A_s) \right) (\Delta_{s,\epsilon} - \Delta_{s,0}) \right] \\
    &\leq \epsilon \|\eta_{t,0}\|_\infty \left\| \frac{1}{\epsilon} \left( q^{(s)}_\epsilon(Z_s,A_s) - q^{(s)}_0(Z_s,A_s) \right) \right\|_{2,\mathcal P_0} \left\| \frac{1}{\epsilon} \left( \Delta_{s,\epsilon} - \Delta_{s,0} \right) \right\|_{2,\mathcal P_0}\,,
\end{align*}
where the inequality follows from Cauchy Schwartz.
Also, by \cref{lem:convergent-nuisances} we have $q_\epsilon^{(s)}(Z_t,A_t) \in \fbounded(u)$ for some $u>0$, and further by \cref{lem:boundedness-algebra} we have $\Delta_{s,\epsilon} \in \fbounded(u)$. In addition, by \cref{lem:bounded-nuisances} we have $\|\eta_{t,0}\|_\infty < \infty$. 
Therefore,  the first term also converges to zero. Thus, putting the above together, we have $\lim_{\epsilon \to 0} \es{s} \left[ \frac{1}{\epsilon} \left( Y_{s-1,\epsilon} - Y_{s-1,0} \right) \right] = 0$, which proves the inductive case, and therefore our proof is complete.

\end{proof}

\section{Nuisance Estimation}
\label{apx:nuisance-estimation}

First in this appendix we provide a proof of \cref{lem:nuisance-estimation}. Then, we provide a meta-algorithm approach for actually estimating nuisances following this lemma, which we implement in our experiments.

\subsection{Proof of Nuisance Estimation Lemma}

The proof of this lemma follows in a very straightforward way by successively applying \cref{lem:phi-is}. The details are as follows.

\begin{proof}[Proof of \cref{lem:nuisance-estimation}]
We will first deal with the case of $q^{(t)}$. We first note that \cref{eq:q} is equivalent to
\begin{equation*}
    \es{t} \left[ g(W_t,A_t) \left( q^{(t)}(Z_t,A_t) - P^*_t(A_t \mid W_t)^{-1} \right) \right] = 0\,,
\end{equation*}
for all measurable $g$. Next, following the same argument as in \cref{lem:bridge-observed-cross}, the above is equivalent to
\begin{equation*}
    \es{t} \left[ g(W_t,A_t) q^{(t)}(Z_t,A_t) - \sum_{a \in \aset} g(W_t,a) \right] = 0 \qquad \forall g \,.
\end{equation*}
Next, we can argue by backward induction that for all $s \in [t]$ we have
\begin{align*}
    &\es{s} \left[ \prod_{s'=s}^{t-1} q^{(s')}(Z_{s'},A_{s'}) \indicator{E_{s'}=A_{s'}} \left( g(W_t,A_t) q^{(t)}(Z_t,A_t) - \sum_{a \in \aset} g(W_t,a) \right) \right] \\
    &= \es{t} \left[ g(W_t,A_t) q^{(t)}(Z_t,A_t) - \sum_{a \in \aset} g(W_t,a) \right] \,.
\end{align*}
The base case is trivially true, and the inductive case follows by applying \cref{lem:phi-is}. Therefore, noting that
\begin{equation*}
    \eta_t = \prod_{s'=0}^{t-1} q^{(s')}(Z_{s'},A_{s'}) \indicator{E_{s'}=A_{s'}} \,,
\end{equation*}
we can put the above together to conclude that \cref{eq:q} is equivalent to
\begin{equation*}
    \eb \left[ \eta_t \left( g(W_t,A_t) q^{(t)}(Z_t,A_t) - \sum_{a \in \aset} g(W_t,a) \right) \right] = 0 \qquad \forall g \,.
\end{equation*}

Finally, we can deal with the case of $h^{(t)}$ almost identically. We can first note that \cref{eq:h} is equivalent to
\begin{equation*}
    \es{t} \left[ f(Z_t,A_t) \left(  h^{(t)}(W_t,A_t) - \indicator{E_t=A_t}Y_t \right) \right] = 0 \qquad \forall f \,,
\end{equation*}
and applying an almost identical backwards induction argument this is equivalent to
\begin{equation*}
    \eb \left[ \eta_t f(Z_t,A_t) \left(  h^{(t)}(W_t,A_t) - \indicator{E_t=A_t}Y_t \right) \right] = 0 \qquad \forall f \,,
\end{equation*}
which in turn is equivalent to
\begin{equation*}
    \eb \left[ \eta_t \left(  h^{(t)}(W_t,A_t) - \indicator{E_t=A_t}Y_t \right) \bigmid Z_t, A_t \right] = 0 \,.
\end{equation*}

\end{proof}

\subsection{Meta-Algorithm for Nuisance Estimation with Kernel Critics}

We now consider a very general meta-algorithm for implementing the estimators in \cref{prop:nuisance-estimators} for kernel critics. In particular, let us assume that the kernel classes $\Gcal^{(t)}$ and $\Fcal^{(t)}$ are RKHSs, and the critic regularizers are $\mathcal R(\cdot) = (\alpha / 4) \|\cdot\|_{K}^2$, where $\|\cdot\|_{K}$ is the corresponding RKHS norm for that critic. Then, this gives the kernel VMM, for which \citet{bennett2023variational} established good theoretical properties in terms of consistency, asymptotic normality, and efficiency. This also admits a simple closed-form for the inner sup in the min-max problems above.

The resulting sequential estimation procedure is summarized in \cref{alg:sequential-vmm}. 
The algorithm estimates the functions $q^{(t)}$ in ascending order and then estimates $h^{(t)}$ in descending order, in each case using plug-in estimates of the previously estimated nuisances.
We let $K^{(q,t)}$ and $K^{(h,t)}$ denote the kernel functions for the critic function classes $\mathcal G^{(t)}$ and $\mathcal F^{(t)}$, respectively, where the former is defined on pairs of $(W_t,A_t)$ tuples, and the latter on pairs of $(Z_t,A_t)$, and let $\alpha^{(q,t)}$ and $\alpha^{(h,t)}$ denote the corresponding hyperparameters for the critic regularizers. We note that any or all of the above inputs may be data-driven, and we emphasize again that the prior estimates $\tilde q^{(t)}$ and $\tilde h^{(t)}$ may come from any methodology and need not necessarily be consistent. In particular, we can start by inputting the zero functions for these (or any other fixed functions), and then run the procedure again using the previous output as input for the prior estimates.

\begin{algorithm}
\setstretch{1.5}
\renewcommand{\algorithmicrequire}{\textbf{Input:}}
\renewcommand{\algorithmicensure}{\textbf{Output:}}
\caption{Sequential VMM for PCI-POMDP Nuisance Estimation}
\label{alg:sequential-vmm}
\begin{algorithmic}[1]
    \Require{Data $\mathcal D = (\tau_H^{(1)},\ldots,\tau_H^{(n)})$, nuisance function classes $\mathcal Q^{(t)}$ and $\mathcal H^{(t)}$, kernel functions $K^{(q,t)}$ and $K^{(h,t)}$, hyperparameters $\alpha^{(q,t)}$ and $\alpha^{(h,t)}$, prior estimates $\tilde q^{(t)}$ and $\tilde h^{(t)}$, and optional regularization functions $\mathcal R^{(q,t)}$ and $\mathcal R^{(h,t)}$, for all $t \in [H]$}
    \Ensure{Nuisance estimates $\hat q^{(t)}$ and $\hat h^{(t)}$ for all $t \in [H]$}
    \State $\eta^{(i)}_t \gets 1 \quad \forall i$
    \For {$t \in \{1, \ldots, H\}$}
        \State $V^{(q,t)} \gets \textsc{Vector}(\{(W^{(i)}_t,a) \mid i \in [n], a \in \aset\})$ 
        \State $\hat q^{(t)} \gets \textsc{ComputeQ}(W_t, Z_t, A_t, V^{(q,t)}, \mathcal Q^{(t)}, \mathcal R^{(q,t)}, \alpha^{(q,t)}, \tilde q^{(t)}, K^{(q,t)}, \eta_t)$
        \State $\eta^{(i)}_{t+1} \gets \eta^{(i)}_t \indicator{A_{t}^{(i)}=E_{t}^{(i)}} \hat q^{(t)}(Z_{t}^{(i)},A_{t}^{(i)}) \quad \forall i$
    \EndFor
    \State $\omega^{(i)}_H \gets 0 \quad \forall i$
    \For{$t \in \{H,H-1,\ldots,1\}$}
        \State $V^{(h,t)} \gets \textsc{Vector}(\{(Z^{(t)}_i,A^{(t)}_i) \mid i \in [n]\})$
        \State $\mu_t^{(i)} \gets \indicator{A_{t}^{(i)}=E_{t}^{(i)}} ( R_{t}^{(i)} + \gamma \omega_t^{(i)} ) \quad \forall i$
        \State $\hat h^{(t)} \gets \textsc{ComputeH}(W_t, Z_t, A_t, V^{(h,t)}, \mathcal H^{(t)}, \mathcal R^{(h,t)}, \alpha^{(h,t)}, \tilde h^{(t)}, K^{(h,t)}, \eta_t, \mu_t)$
        \State $\omega_{t-1}^{(i)} \gets \sum_{a \in \aset} \hat h^{(t)}(W^{(i)}_{t},a) + \hat q^{(t)}(Z^{(i)}_{t},A^{(i)}_{t}) \left( \mu_t^{(i)} - \hat h^{(t)}(W^{(i)}_{t},A^{(i)}_{t}) \right) \quad \forall i$
    \EndFor
    \State \Return $\hat q^{(1)}, \ldots, \hat q^{(H)}$, $\hat h^{(1)}, \ldots, \hat h^{(H)}$
    \Procedure{ComputeQ}{$W$, $Z$, $A$, $V$, $\mathcal Q$, $\mathcal R$, $\alpha$, $\tilde q$, $K$, $\eta$}
        \State $L_{i,j}(q) \gets \eta^{(i)}  \left( q(Z_i, A_i) K(((W_i,A_i), V_j) - \sum_{a \in \aset} K((W_i,a), V_j\right)) \quad \forall i,j,q$
        \State $\Omega_{i,j} \gets \frac{1}{n} \sum_{k=1}^n L_{k,i}(\tilde q) L_{k,j}(\tilde q) + \alpha K(V_i, V_j) \quad \forall i,j$
        \State $\rho_i(q) \gets \frac{1}{n}  \sum_{k=1}^n L_{k,i}(q) \quad \forall i,q$
        \State \Return $\argmin_{q \in \mathcal Q} \rho(q)^T \Omega^{-1} \rho(q) + \mathcal R(q)$
    \EndProcedure
    \Procedure{ComputeH}{$W$, $Z$, $A$, $V$, $\mathcal H$, $\mathcal R$, $\alpha$, $\tilde h$, $K$, $\eta$, $\mu$}
        \State $L_{i,j}(h) \gets \eta_i  K((Z_i,A_i), V_j) \left( h(W_i, A_i)  - \mu_i \right) \quad \forall i,j,h$
        \State $\Omega_{i,j} \gets \frac{1}{n} \sum_{k=1}^n L_{k,i}(\tilde h) L_{k,j}(\tilde h) + \alpha K(V_i, V_j) \quad \forall i,j$
        \State $\rho_i(h) \gets \frac{1}{n}  \sum_{k=1}^n L_{k,i}(h) \quad \forall i,h$
        \State \Return $\argmin_{h \in \mathcal H} \rho(h)^T \Omega^{-1} \rho(h) + \mathcal R(h)$
    \EndProcedure
\end{algorithmic}
\end{algorithm}

We provide a derivation of this algorithm below. We note that it is a meta-algorithm, since it requires some additional procedures to solve the minimization problems over $q \in \mathcal Q^{(t)}$ and $h \in \mathcal \mathcal H^{(t)}$ at the end of \textsc{ComputeQ} and \textsc{ComputeH} respectively. However, solving such problems is very standard and well studied, so we do not consider it explicitly here. 
In this algorithm we let $\textsc{Vector}$ denote a function which converts a set into a vector with the elements ordered arbitrarily.
Finally, in the case that the data is discrete, this algorithm is very efficient in terms of how it scales with $n$; the overall computational cost will be linear in $n$, since the maximum possible lengths of $V^{(q,t)}$ and $V^{(h,t)}$ are bounded for each $t \in [H]$.

\subsection{Derivation of Meta-Algorithm}

Following \citet{bennett2023variational}, the kernel VMM estimators work by solving
\begin{equation*}
    \hat q^{(t)} = \argmin_{q \in \mathcal Q^{(t)}} J_n(q;\alpha^{(q,t)}) + \mathcal R^{(q,t)}(q) \,,
\end{equation*}
where
\begin{align*}
    J_n(q;\alpha^{(q,t)}) &= \sup_g \e_n\left[ \eta_t \left( g(W_t,A_t) q(Z_t,A_t) - \sum_{a \in \aset}g(W_t,a) \right) \right] \\
    &\qquad - \frac{1}{4} \e_n\left[ \eta_t^2 \left( g(W_t,A_t) \tilde q^{(t)}(Z_t,A_t) - \sum_{a \in \aset}g(W_t,a) \right)^2 \right] \\
    &\qquad + \alpha^{(q,t)} \|g\|_{K^{(q,t)}}\,,
\end{align*}
where $\e_n$ is the empirical expectation using the $n$ observed trajectories, and $\|\cdot\|_{K^{(q,t)}}$ is the RKHS norm with kernel $K^{(q,t)}$. First, by Representer Theorem, we only need to optimize over $g$ of the form
\begin{equation*}
    g = \sum_{j=1}^{N(W_t;\aset)} \beta_j K^{(q,t)}(\mathbb S(W_t;\aset)_j, \cdot)\,.
\end{equation*}
Now, plugging this into the above equation for $J_n(q;\alpha_t)$, we obtain
\begin{equation*}
    J_n(q;\alpha_t) = \sup_{\beta} \beta^T \rho(q) - \frac{1}{4} \beta^T Q \beta\,,
\end{equation*}
where $Q$ and $\rho(q)$ are defined as in function \textsc{ComputeQ} in \cref{alg:sequential-vmm}. Given $Q$ is invertible, it is straightforward to verify that the above is maximized by $\beta = 2 Q^{-1} \rho(q)$, giving $J_n(q;\alpha_t) = \rho(q)^T Q^{-1} \rho(q)$. Therefore, the kernel VMM solution is given by
\begin{equation*}
    \hat q^{(t)} = \argmin_{q \in \mathcal Q^{(t)}} \rho(q)^T Q^{-1} \rho(q) + \mathcal R^{(q,t)}(q) \,,
\end{equation*}
which verifies the correctness of \textsc{ComputeQ}.

Next, for \textsc{ComputeH}, following an almost identical argument as above the kernel VMM estimator is given by
\begin{equation*}
    \hat h^{(t)} = \argmin_{h \in \mathcal H^{(t)}} J_n(h;\alpha^{(h,t)}) + \mathcal R^{(h,t)}(h) \,,
\end{equation*}
where
\begin{align*}
    J_n(h;\alpha^{(h,t)}) &= \sup_f \e_n \left[ \eta_t f(Z_t,A_t) \left( h(W_t,A_t) - \mu_t \right)\right] \\
    &\qquad - \frac{1}{4} \e_n \left[ \eta_t^2 f(Z_t,A_t)^2 \left( h(W_t,A_t) - \mu_t \right)^2 \right] + \alpha^{(h,t)} \|f\|_{K^{(h,t)}} \,,
\end{align*}
where $\mu_t = \indicator{E_t=A_t} Y_t$. Again, by the Representer Theorem, we only need to optimize over $f$ of the form
\begin{equation*}
    f = \sum_{j=1}^{N(Z_t,A_t)} \beta_j K^{(h,t)}(\mathbb S(Z_t,A_t)_j, \cdot)\,,
\end{equation*}
and plugging this in to the above gives us
\begin{equation*}
    J_n(h;\alpha_t) = \sup_{\beta} \beta^T \rho(h) - \frac{1}{4} \beta^T Q \beta\,,
\end{equation*}
where $\rho(h)$ and $Q$ are defined as in \textsc{ComputeH}. This is clearly minimized by $\beta = 2 Q^{-1} \rho(h)$, and plugging this into the original objective gives us
\begin{equation*}
    \hat h^{(t)} = \argmin_{h \in \mathcal H^{(t)}} \rho(h)^T Q^{-1} \rho(h) + \mathcal R^{(h,t)}(h) \,,
\end{equation*}
which verifies the correctness of \textsc{ComputeH}.

Finally, for the main part of the algorithm, we observe that it works by sequentially estimating each $q^{(t)}$ for increasing $t$ according to \textsc{ComputeQ}, using an estimate of $\eta_t$ given by plugging in our estimates of $q^{(t')}$ for $t' < t$, and then sequentially estimating each $h^{(t)}$ for decreasing $t$ according to \textsc{ComputeH}, using an estimate of $\eta_t$ given by plugging in our estimates of $q^{(t)}$ for $t' < t$ and an estimate of $\mu_t$ by plugging in our estimates of $q^{(t')}$ and $h^{(t')}$ for $t'>t$. We note that in the computation of the estimate of $\mu_t$, we use the fact that
\begin{equation*}
    \mu_t = \indicator{E_t=A_t} \left( R_t + \gamma \omega_t \right)\,,
\end{equation*}
where $\omega_H=0$, and for $t<H$ we have
\begin{equation*}
    \omega_t = \sum_{a \in \aset} h^{(t+1)}(W_{t+1},a) + q^{(t+1)}(Z_{t+1},A_{t+1}) \left( \mu_{t+1} - h^{(t+1)}(W_{t+1},A_{t+1}) \right)\,.
\end{equation*}

\section{Additional Details for Experiment 1}
\label{apx:experiment-details}

\subsection{Environment Details}

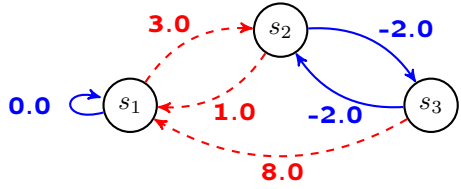
\begin{figure}
\centering
\begin{tikzpicture}[->,>=stealth',auto,node distance=3cm,
  thick,mainnode/.style={circle,draw},dottednode/.style={circle,draw,dotted},noborder/.style={{circle,inner sep=0,outer sep=0}}]
  \node[mainnode] (s1) at (-2,0) {$s_1$};
  \node[mainnode] (s2) at (0,1) {$s_2$};
  \node[mainnode] (s3) at (2,0) {$s_3$};
  \path[every node/.style={font=\sffamily\tiny}]
    (s1) edge[red, dashed, bend left] node [right, scale=2.0, xshift=-12, yshift=3] {3.0} (s2)
    (s2) edge[red, dashed, bend left] node [right, scale=2.0, xshift=-5, yshift=-3] {1.0} (s1)
    (s3) edge[red, dashed, bend left] node [right, scale=2.0, yshift=-3, xshift=-7.5] {8.0} (s1)
    (s1) edge[blue, loop left] node [scale=2.0] {0.0} ()
    (s2) edge[blue, bend left] node [right, scale=2.0, xshift=-2, yshift=2] {-2.0} (s3)
    (s3) edge[blue, bend left] node [right, scale=2.0, xshift=-10, yshift=-3.5] {-2.0} (s2);
\end{tikzpicture}
\caption{Graphical representation of the \noisyobs POMDP scenario used in our Proximal RL experiments. Red dashed edges / blue solid edges represent the transitions under actions $a_1$ / $a_2$ respectively, and the numeric label for each edge indicates the corresponding reward. Note that all transitions and rewards in \noisyobs are deterministic, and do not depend on the time index. In each state $s_i$ we receive observation $o_i$ with probability $1 - \epsnoise$, or observatoin $o_j$ with probability $\epsnoise / 2$, for each $j \neq i$.}
\label{fig:noisyobs}
\end{figure}

We describe here the details of the \noisyobs environment. First, in \cref{fig:noisyobs} we describe the state transition and reward structure of this POMDP. Second, the initial state is drawn according to the following process: for each logged trajectory we first sample a prior state $S_0$ equal to $s_1$, $s_2$, or $s_3$ with probabilities $0.5$, $0.3$, and $0.2$ respectively, a prior observation $O_0 \sim P_O(\cdot \mid S_0)$, and a prior action $A_0 \sim \pinoisyobs(\cdot \mid S_0)$. Then, the initial state $S_1$ is given by transitioning from $S_0$ with $A_0$.

\subsection{Policy Details}

\begin{table}
\begin{center}
    \begin{tabular}{c|cc}
         & $a_1$ & $a_2$  \\
        \hline
        $s_1$ & 0.8 & 0.2 \\
        $s_2$ & 0.8 & 0.2 \\
        $s_3$ & 0.2 & 0.8 \\
    \end{tabular}
    \hspace{1cm}
    \begin{tabular}{c|cc}
         & $a_1$ & $a_2$  \\
        \hline
        $o_1$ & 1 & 0 \\
        $o_2$ & 1 & 0 \\
        $o_3$ & 0 & 1 \\
    \end{tabular}
    \hspace{1cm}
    \begin{tabular}{c|cc}
         & $a_1$ & $a_2$  \\
        \hline
        $o_1$ & 0 & 1 \\
        $o_2$ & 0 & 1 \\
        $o_3$ & 1 & 0 \\
    \end{tabular}
    \hspace{1cm}
    \begin{tabular}{c|cc}
         & $a_1$ & $a_2$  \\
        \hline
        $o_1$ & 1 & 0 \\
        $o_2$ & 0 & 1 \\
        $o_3$ & 1 & 0 \\
    \end{tabular}
\end{center}
\caption{Details of the policies under consideration for our Proximal RL experiments in the \noisyobs POMDP scenario.
The first table summarizes the probability distribution of the logging policy $\pinoisyobs$, where each row gives the probability distribution over actions for the corresponding state. The next three tables similarly summarize the evaluation policies $\pieasy$, $\pihard$, and $\pioptim$ respectively, which are all deterministic policies that depend on the current observation only. Note that none of these policies depend on the time index.}
\label{tab:experiment-policies}
\end{table}

In \cref{tab:experiment-policies} we fully describe both the behavior policy $\pinoisyobs$ as well as the evaluation policies $\pieasy$, $\pihard$, and $\pioptim$ used in our experiments.

\subsection{Method Details}

Here we provide more detail about each of the methods used in our experiments.

\subsubsection{Ours}

Our method is an implementation of the estimator described in \cref{sec:estimation}, using 5-fold cross-fitting, and with nuisance estimation following \cref{alg:sequential-vmm}. As described in \cref{sec:experiments}, we use the PCI reduction given by setting $Z_t = O_{t-1}$ and $W_t = O_t$, and we did not include an explicit $X_t$.

For every $t \in [H]$ we set the inputs to the algorithm as follows: $\mathcal H^{(t)}$ and $\mathcal Q^{(t)}$ were the set of all tabular functions; all regularization functions were set as $\mathcal R(f) = \lambda \|f\|_{2,n}$, for some fixed hyperparameter $\lambda$; all values of $\alpha^{(q,t)}$ and $\alpha^{(h,t)}$ were set a to a common hyperparameter $\alpha$; and the kernels $K^{(q,t)}$ and $K^{(h,t)}$ were set as in \citet{bennett2023variational}, using the same process of combining three Gaussian kernels with automatically calibrated bandwidths based on the variance of the data. Furthermore, the inputs to the kernel functions were given by concatenating one-hot embeddings of $Z_t$ and $A_t$ or $W_t$ and $A_t$.

\begin{table}[t]
\centering
\begin{tabular}{c|ccc}
    & $\pieasy$ & $\pihard$ & $\pioptim$ \\
    \hline
    $\epsnoise=0$ & $(10^{-4}, 10^{-4})$ & $(10^{-2}, 10^{-2})$ & $(10^{-4}, 10^{-2})$ \\
    $\epsnoise=0.2$ & $(10^{-4}, 10^{-4})$ & $(10^{-2}, 10^{-4})$ & $(10^{-4}, 10^{-4})$ \\
\end{tabular}
\caption{Summary of hand-chosen hyperparameter combinations for each setting. Each tuple gives the chosen value for $\alpha$ and $\lambda$ respectively.}
\label{tab:hyperparameters}
\end{table}

For each setting (given by combination of $\pi_e$ and $\epsnoise$), we experimented with hyperparameter values given by $\alpha \in \{10^{-2}, 10^{-4}, 10^{-6}, 10^{-8}\}$, and $\lambda \in \{1, 10^{-2}, 10^{-4}, 10^{-6}\}$. In all cases, we set all values of $\alpha^{(q,t)}$ and $\alpha^{(h,t)}$ to the same $\alpha$, and similarly we set all values of $\lambda^{(q,t)}$ and $\lambda^{(h,t)}$ to the same $\lambda$. In each setting, we performed grid search by experimenting with all of the above combinations of $\alpha$ and $\lambda$, and hand-selected and presented results for a combination of $(\alpha,\lambda)$ for which the algorithm performed well. Our hand-selected $(\alpha,\lambda)$ for each setting is summarized in \cref{tab:hyperparameters}.

We emphasize that this process is \emph{not} meant to simulate what would be done in real applications, where the true target policy value is unknown. Rather, our intention is demonstrate the proof of concept of our theory, and show that an algorithm such as \cref{alg:sequential-vmm} can produce accurate policy value estimates when the hyperparameters are well-calibrated. We leave the problem of automatically selecting such hyperparameters in a data-driven way to future work.

\subsubsection{MeanR}

This baseline is extremely simple, given by $\frac{1}{n} \sum_{i=1}^n \sum_{t=1}^H \gamma^t R_t^{(i)}$.

\subsubsection{MDP}

For this baseline we first fit an MDP model to the observed observation, action, next observation counts in $\mathcal D$, treating observations as states. Then, we compute the value of $\pi_e$ on this count-based tabular MDP model, using dynamic programming.

\subsubsection{TIS}

This is given by estimating the time-independent sampling identification quantity defined by \cref{thm:identification-ind,lem:ind-tabular}, by estimating the required probability matrices directly from the observed counts, and replacing the expectation over $\pind$ with its empirical analogue, based on summing over all $n^H$ combinations of separately sampling an observed trajectory at each time step, and then normalizing by multiplying by $n^{-H}$.

\subsection{Results for the MDP Setting}

\Cref{fig:mdp-results} presents the results of our experiment in an MDP setting without any confounding. Namely, where we set $\epsnoise=0$. We can observe that this makes \textsc{MDP} baseline consistent and highly accurate. Our estimate remains consistent, has lower variance than in the POMDP setting shown in \cref{fig:pomdp-results} but still more than the \textsc{MDP} baseline. The \textsc{MeanR} and \textsc{TIS} baselines still perform badly even in the MDP setting, as expected.

\begin{figure}
    \centering
    \begin{tabular}{cc}
         \includegraphics[width=0.47\textwidth]{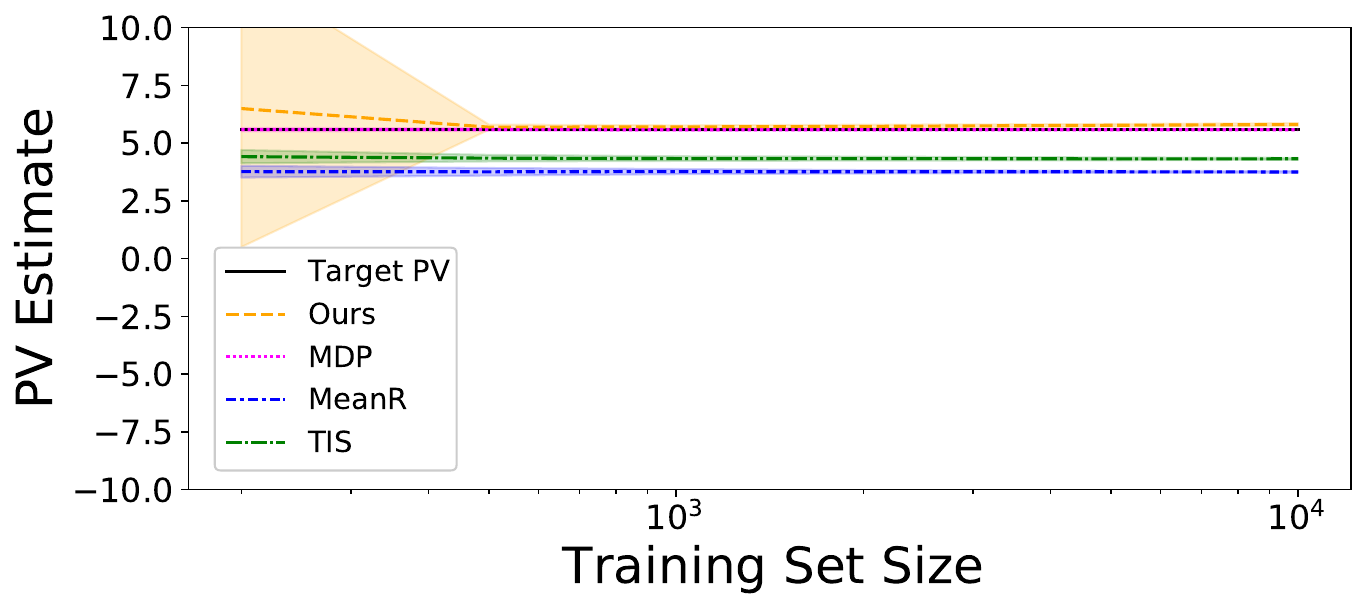} & 
         \includegraphics[width=0.47\textwidth]{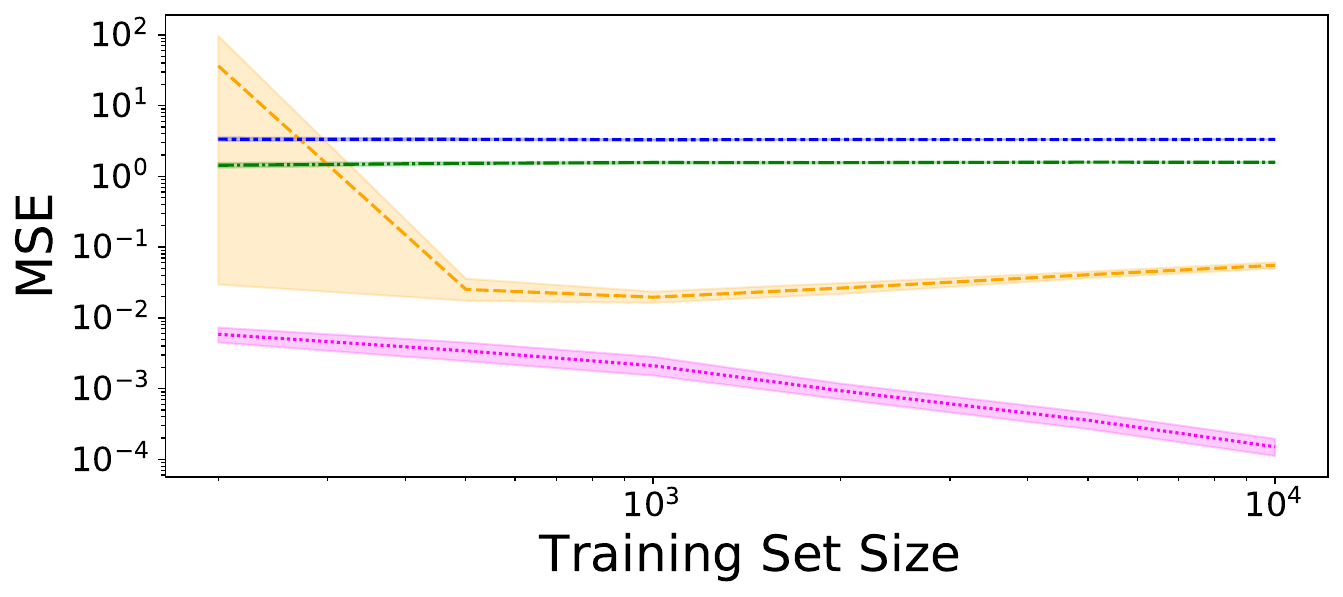} \\
         \includegraphics[width=0.47\textwidth]{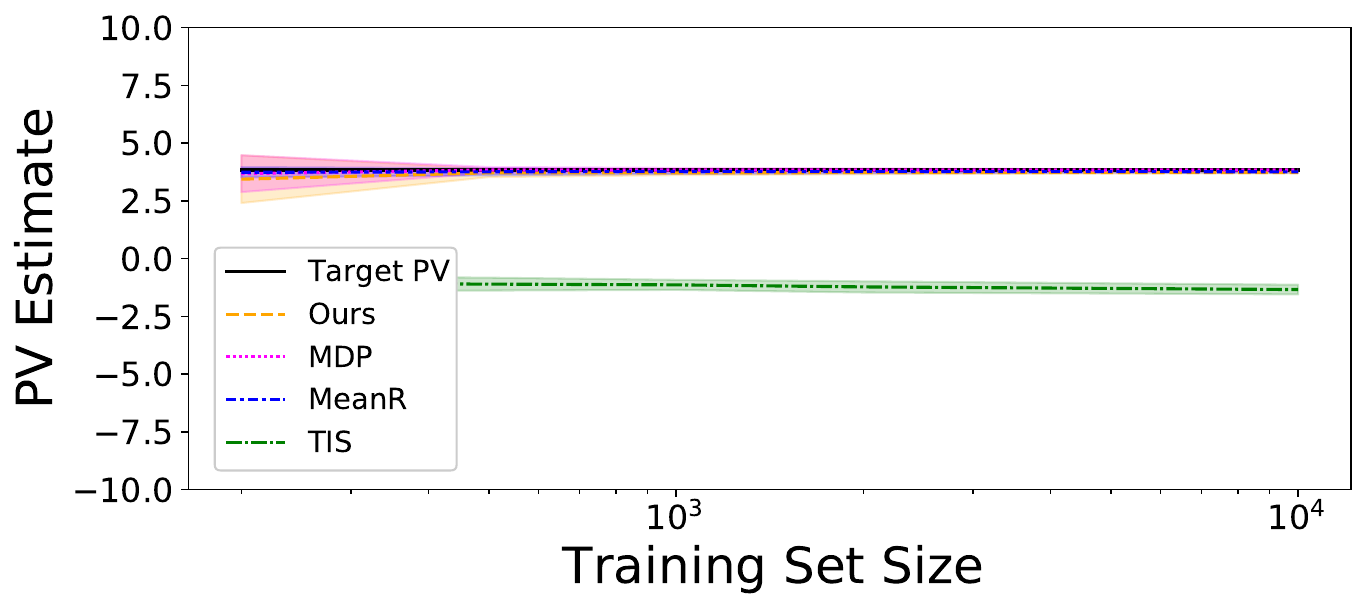} & 
         \includegraphics[width=0.47\textwidth]{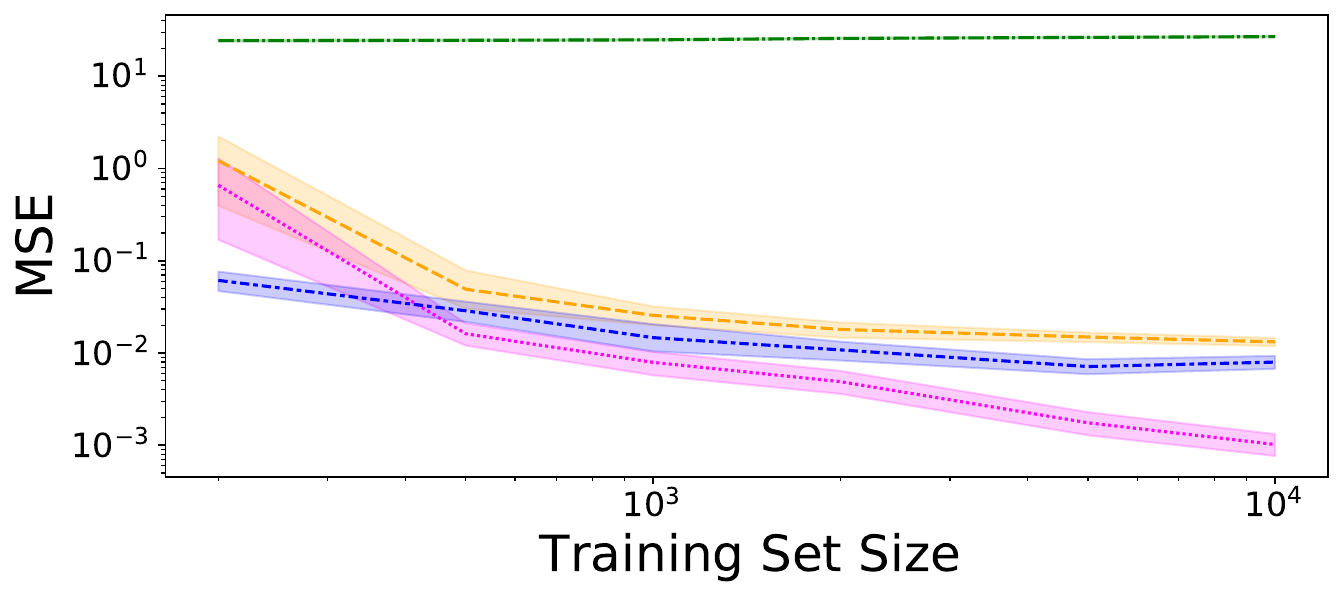} \\
         \includegraphics[width=0.47\textwidth]{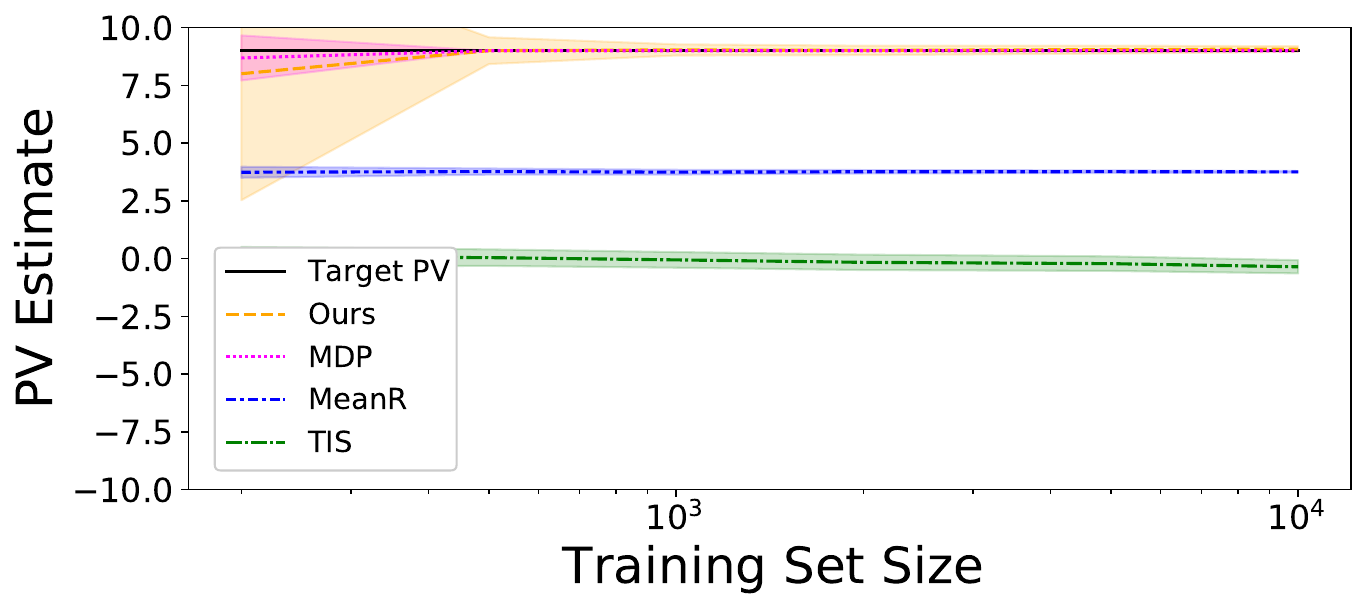}  &
         \includegraphics[width=0.47\textwidth]{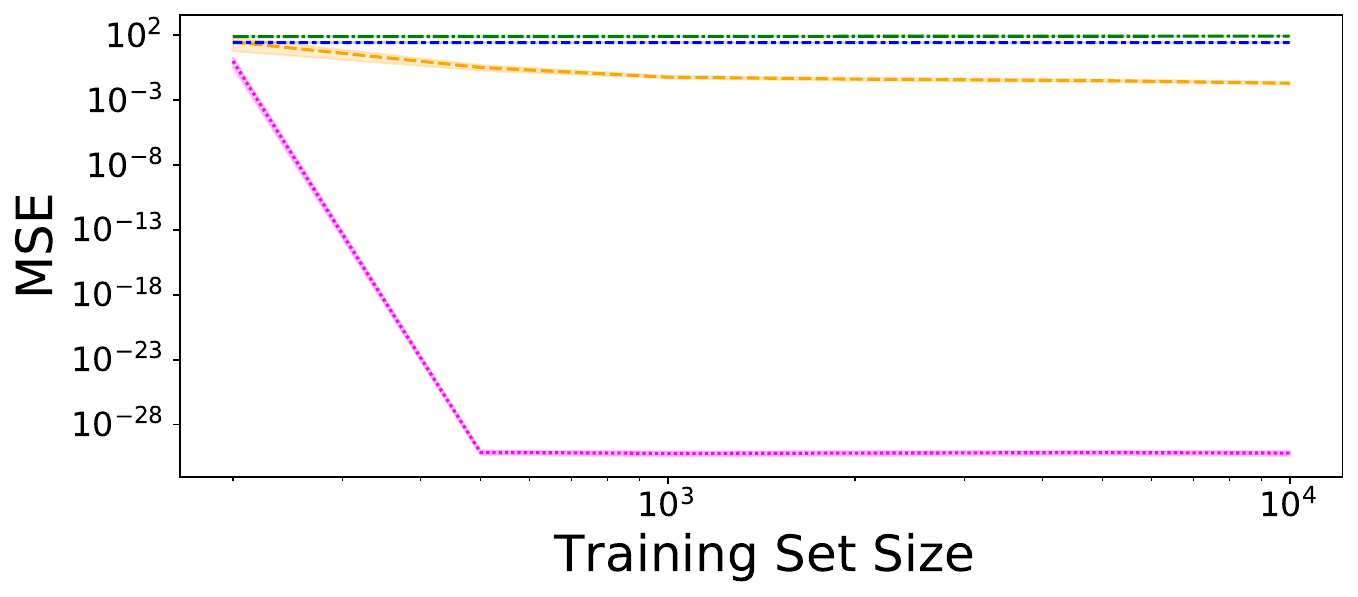}  \\
    \end{tabular}
    \caption{Results of our Proximal RL experiments on the \noisyobs environment with $\epsnoise=0$. Plots are displayed in the same order as in \cref{fig:pomdp-results}.}
    \label{fig:mdp-results}
\end{figure}

\section{Additional Details for Experiment 2}

\subsection{Environment and Policy Details}

We used the exact sepsis simulator environment as in \citet{oberst2019counterfactual}, except with the following modifications:
\begin{enumerate}
    \item We replaced their initial state distribution with a uniform distribution over all non-terminal states. 
    \item As described in the main text, each patient initially (at $t=1$) has a 20\% chance of having their diabetes status censored, in which case the observed value of diabetes is always False for that patient (regardless of whether they actually have diabetes or not)
    \item We used a slightly different reward function, in order to make rewards less sparse. As in the original simulator, this is based on a count of the number of core indicators (heart rate, blood pressure, oxygen concentration, and glucose level) taking values within safe bounds. Specifically, rewards are defined as follows: (1) if all four indicators are safe and the patient is not on any of the three treatments, they receive a reward of $1$; (2) if all four indicators are safe but they are on some treatment they receive a reward of $0$; (3) if exactly one indicator is unsafe they receive a reward of $-1$; (4) if exactly two indicators are unsafe they receive a reward of $-2$; and (5) if three or more indicators are unsafe they receive a reward of $-10$.
    In addition, as in the original simulator, in the case of (1) or (5) they enter a cured or dead terminal (absorbing) state respectively. However, in our version of the simulator, in these cases they continue to receive the reward of $1$ or $-10$ respectively for all future time steps. 
\end{enumerate}

For full exact details of the environment, as well as the exact details of how we generated $\pi_b$ and $\pi_e$ (which was sketched out in the main text), we refer readers to our code release.

\subsection{Method Details}

\subsubsection{Benchmarks}
The \textsc{MeanR} and \textsc{MDP} benchmarks were implemented identically as in the prior experiment. See the prior appendix section for details.

\subsubsection{Our Method}

Our method is again an implementation of the estimator described in \cref{sec:estimation}, this time using 2-fold cross-fitting, and again with nuisance estimation following \cref{alg:sequential-vmm}. As described in the main text in \cref{sec:experiments-sepsis}, if we partition $O_t = (G_t,X_t)$, where $G_t$ is the (possibly censored) indicator of whether the patient is diabetic, and $X_t$ is all other aspects of $O_t$, then we use the PCI reduction given by setting $Z_t=(G_{t-1},X_t)$, and $W_t=(G_t,X_t)$. Note that since the (censored) diabetes observation doesn't change over time, this means that $G_{t-1}=G_t$ so both proxies are identical, but this is consistent with our assumptions since $G_t$ is deterministically determined by the full state $S_t$, so $G_t$ and $G_{t-1}$ are conditionally independent given $S_t$.

For every $t \in [H]$ we set the inputs to the algorithm as follows: $\mathcal H^{(t)}$ and $\mathcal Q^{(t)}$ were some particular neural net classes described below; the regularization functions were set as $\Rcal^{(q,t)}(q) = \lambda_q \|q\|_{2,n}$ and $\Rcal^{(h,t)}(h) = \lambda_h \|h\|_{2,n}$ for all $t$, for some fixed hyperparameters $\lambda_q$ and $\lambda_h$ that don't depend on $t$; the values of $\alpha^{(q,t)}$ and $\alpha^{(h,t)}$ were set to common hyperparameters $\alpha_q$ and $\alpha_h$ respectively that don't depend on $t$; and the kernels $K^{(q,t)}$ and $K^{(h,t)}$ were set as in \citet{bennett2023variational}, using the same process of combining three Gaussian kernels with automatically calibrated bandwidths based on the variance of the data. Furthermore, the inputs to the kernel functions were given by concatenating embeddings of $(G_t,X_t,A_t)$ based on the following featurizations that were concatenated together:
\begin{enumerate}
    \item The $G_t$ was encoded as a single $0$ or $1$ valued variable
    \item We extracted 13 features from $X_t$ as follows:
    \begin{enumerate}
        \item $0$ or $1$-valued variables indicating whether heart rate was abnormally high, heart rate was abnormally low, blood pressure was abormally high, blood pressure was abormally low, oxygen level was abornally low, and glucose level was abnormal (6 features)
        \item $0$ or $1$-valued variables indicating whether there were exactly one, or exactly two, main indicators with unsafe values (2 features)
        \item The glucose level, normalized in range of $0$ to $1$ (1 feature)
        \item $0$ or $1$-valued binary variables indicating whether the patient is currently on each of the three treatments (3 features)
        \item $0$ or $1$-valued variable indicating whether the patient is in an absorbing state (1 feature)
    \end{enumerate}
    \item $A_t$ was encoded as a vector of three $0$ or $1$-valued binary variables indicating whether each treatment type was used or not
\end{enumerate}

We used the following neural net class for each of the $\Hcal^{(t)}$ and $\Qcal^{(t)}$ classes: (1) we first separately passed the $A_t$ and $(X_t,G_t)$ embedings described above through linear layers with outputs of sizes 10 and 50 respectively; (2) we then passed these two vectors through a bi-linear layer with 50 outputs; (3) we passed the output of the bi-linear layer through a linear layer with 50 outputs again; and (4) we passed the previous outputs through a final linear layer with 1 output. In addition, we used GeLU \citep{hendrycks2016gaussian} nonlinearities in between all linear and bi-linear layers. 

For more thorough details of our method implementation, we refer readers to our code release.

\subsection{Hyperparameter Details}

We experimented with separately setting the hyperparameters $\lambda_h$, $\lambda_q$, $\alpha_h$, and $\alpha_q$ for our method, each taking values in the set $\{10^{-2}, 10^{-3}, 10^{-4}\}$. This gave a total of 81 different total hyperparameter configurations that we experimented with. Out of those different configurations, the single best configuration (in terms of overall MSE), which we presented results for, was given by $\lambda_h=10^{-4}$, $\lambda_q=10^{-2}$, $\alpha_h=10^{-2}$, and $\alpha_q=10^{-4}$. On the other hand, our automatic hyperparameter heuristic separately selected from these 81 different configurations, by taking a median of the predictions after discarding out-of-bound values, on each of the 50 replications.

\end{document}